\documentclass{article}
\usepackage{algorithm}
\usepackage{algorithmicx}
\usepackage{algpseudocode}

\usepackage{hyperref}

\usepackage[accepted]{icml2023}

\usepackage{amsmath}
\usepackage{amssymb}
\usepackage{mathtools}
\usepackage{amsthm}

\usepackage[capitalize,noabbrev]{cleveref}

\theoremstyle{plain}

\newtheorem{theorem}{Theorem}[section]
\newtheorem{lemma}[theorem]{Lemma}
\newtheorem{corollary}[theorem]{Corollary}

\newtheorem{definition}[theorem]{Definition}

\newtheorem{remark}[theorem]{Remark}

\newtheorem{assumption}[theorem]{Assumption}

\usepackage[textsize=tiny]{todonotes}

\usepackage[utf8]{inputenc} 
\usepackage[T1]{fontenc}    
\usepackage{url}            
\usepackage{booktabs}       
\usepackage{times}
\usepackage{xcolor} 
\usepackage{graphicx}
\usepackage{blindtext}
\usepackage[labelformat=empty, position=top]{subcaption}
\usepackage[export]{adjustbox}
\usepackage{wrapfig}
\usepackage{pifont}
\newcommand{\cmark}{\ding{51}}%
\newcommand{\xmark}{\ding{55}}%
\definecolor{mydarkblue}{rgb}{0,0.08,0.45}
\hypersetup{ %
    colorlinks=true,
    linkcolor=mydarkblue,
    citecolor=mydarkblue,
    filecolor=mydarkblue,
    urlcolor=mydarkblue,
}

\usepackage{tikz}
\usetikzlibrary{decorations.pathreplacing,calc}

\definecolor{darkblue}{rgb}{0,0.08,0.45}
\definecolor{cred}{HTML}{D62728}
\definecolor{cblue}{HTML}{1F77B4}
\definecolor{cgreen}{HTML}{79AB76}
\definecolor{cgrey}{rgb}{0.6,0.6,0.6}

\newcommand{\omegaspace}{\mathcal{W}}

\newcommand{\ours}[0]{\textsc{Rollin}}

\renewcommand{\mathbf}{\boldsymbol}

\newcommand{\mb}{\mathbf}
\newcommand{\mc}{\mathcal}

\newcommand{\bb}{\mathbb}

\newcommand{\eps}{\varepsilon}

\makeatletter

\newcommand{\wt}{\widetilde}

\newcommand{\norm}[2]{\left\| #1 \right\|_{#2}}
\newcommand{\abs}[1]{\left| #1 \right|}

\newcommand{\paren}[1]{\left( #1 \right)}
\newcommand{\brac}[1]{\left[ #1 \right]}
\newcommand{\Brac}[1]{\left\{ #1 \right\}}

\newcommand{\mr}{\mathrm}

\icmltitlerunning{Understanding the Complexity Gains of Single-Task RL with a Curriculum}

\begin{document}

\twocolumn[
\icmltitle{Understanding the Complexity Gains of Single-Task RL with a Curriculum}



\icmlsetsymbol{equal}{*}

\begin{icmlauthorlist}
\icmlauthor{Qiyang Li}{equal,yyy}
\icmlauthor{Yuexiang Zhai}{equal,yyy}
\icmlauthor{Yi Ma}{yyy}
\icmlauthor{Sergey Levine}{yyy}
\end{icmlauthorlist}

\icmlaffiliation{yyy}{UC Berkeley}

\icmlcorrespondingauthor{Qiyang Li}{qcli@berkeley.edu}
\icmlcorrespondingauthor{Yuxiang Zhai}{simonzhai@berkeley.edu}

\icmlkeywords{Machine Learning, ICML}

\vskip 0.3in
]



\printAffiliationsAndNotice{}  

\begin{abstract}
Reinforcement learning (RL) problems can be challenging without well-shaped rewards. Prior work on provably efficient RL methods generally proposes to address this issue with dedicated exploration strategies. However, another way to tackle this challenge is to reformulate it as a multi-task RL problem, where the task space contains not only the challenging task of interest but also easier tasks that implicitly function as a curriculum. Such a reformulation opens up the possibility of running existing multi-task RL methods as a more efficient alternative to solving a single challenging task from scratch. In this work, we provide a theoretical framework that reformulates a single-task RL problem as a multi-task RL problem defined by a curriculum. Under mild regularity conditions on the curriculum, we show that sequentially solving each task in the multi-task RL problem is more computationally efficient than solving the original single-task problem, without any explicit exploration bonuses or other exploration strategies. We also show that our theoretical insights can be translated into an effective practical learning algorithm that can accelerate curriculum learning on simulated robotic tasks.

\end{abstract}

\section{Introduction}
\label{sec:intro}
Reinforcement learning (RL) provides an appealing and simple way to formulate control and decision-making problems in terms of reward functions that specify what an agent should do, and then automatically train policies to learn how to do it. However, in practice the specification of the reward function requires great care: if the reward function is well-shaped, then learning can be fast and effective, but if rewards are delayed, sparse, or can only be achieved after extensive explorations, RL problems can be exceptionally difficult~\citep{kakade2002approximately,andrychowicz2017hindsight,agarwal2019reinforcement}. This challenge is often overcome with either reward shaping~\citep{ng1999policy,andrychowicz2017hindsight,andrychowicz2020learning,gupta2022unpacking} or dedicated exploration methods~\citep{tang2017exploration,stadie2015incentivizing,bellemare2016unifying,burda2018exploration}, but reward shaping can bias the solution away from optimal behavior, while even the best exploration methods, in general, may require covering the entire state space before discovering high-reward regions.

\begin{figure}
    \centering
    \begin{subfigure}[t]{0.49\textwidth}
        \includegraphics[width=\textwidth]{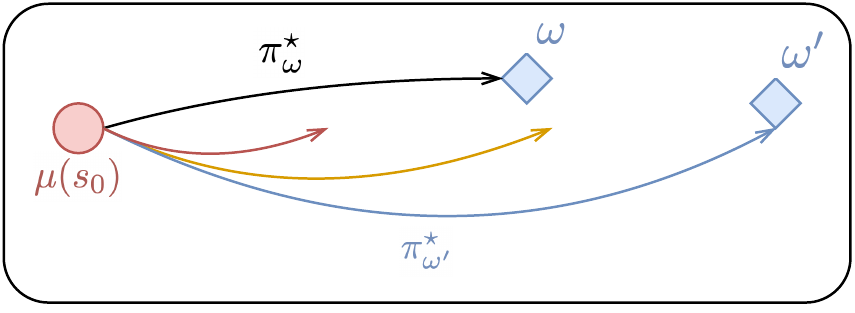}
        \caption{Learning $\pi^\star_{\omega^\prime}$ from scratch}
        \label{fig:rollin-illu1}
    \end{subfigure}
    \begin{subfigure}[t]{0.49\textwidth}
        \includegraphics[width=\textwidth]{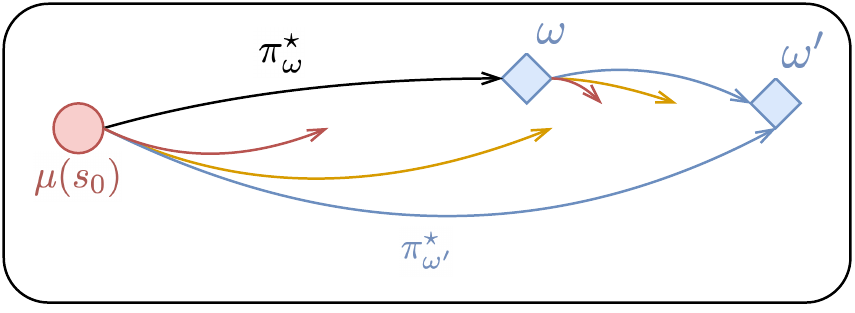}\caption{Learning $\pi^\star_{\omega^\prime}$ with \ours{}} 
        \label{fig:rollin-illu2}
    \end{subfigure}
    \caption{\textbf{Illustration of \ours{}}. {\color[HTML]{B85450} \textbf{The red circle}} represents the initial state distribution. \textbf{The dark curve} represents the optimal policy of the preceding task $\omega$. {\color[HTML]{6C8EBF} \textbf{The blue diamonds}} represent the optimal state distributions $d_\mu^{\pi_\omega^\star}, d_\mu^{\pi_{\omega^\prime}^\star}$ of the preceding task $\omega$ and the current $\omega^\prime$ respectively. \ours{} runs the optimal policy of the preceding task $\pi^\star_{\omega}$ to obtain a better initial state distribution for faster learning of the optimal policy of the current task $\pi^\star_{\omega^\prime}$.}
    \label{fig:illustration}
\end{figure}

On the other hand, a number of recent works have proposed multi-task learning methods in RL that involve learning contextual policies that simultaneously represent solutions to an entire space of tasks, such as policies that reach any potential goal~\citep{fu2018variational,eysenbach2020c,fujita2020distributed,zhai2022computational}, policies conditioned on language commands~\citep{nair2022learning}, or even policies conditioned on the parameters of parametric reward functions~\citep{kulkarni2016deep,siriwardhana2019vusfa,eysenbach2020rewriting,yu2020meta}. While such methods are often not motivated directly from the standpoint of handling challenging exploration scenarios, but rather directly aim to acquire policies that can perform all tasks in the task space, these multi-task formulations often present a more tractable learning problem than acquiring a solution to a single challenging task in the task space (e.g., the hardest goal, or the most complex language command). We pose the following question: 
\begin{center}
\em
When do we expect solving the reformulated multi-task RL problem with task-conditioned policies to be more efficient than solving the original single-task problem directly?    
\end{center}

In this work, we study this question by analyzing the complexity of learning an optimal policy in the stochastic policy gradient (SPG) setting~\citep{agarwal2021theory,mei2020global,ding2021beyond} with a curriculum (learning a list of tasks in sequence). As pointed out by~\citet{ding2021beyond}, for learning an optimal policy, SPG requires a polynomial sample complexity if the initialization is near-optimal.\footnote{See~\cref{def:near-optimal-init} of \cref{sec:results-main}.} In general, there is no guarantee that the initial policy is near-optimal, which could potentially lead to an unbounded density ratio and thus poor sample complexity bound. While there have been a lot of prior works that utilize exploration bonuses to address the sample complexity~\citep{azar2017minimax,jin2018q,agarwal2020pc,zhang2020almost}, we take a different approach without the need for exploration bonuses by making use of a curriculum of tasks where adjacent tasks in the curriculum are close in terms of their optimal state visitation distributions. Our algorithm, \ours{}, works by (1) using the optimal policy parameters of the previous task as an initialization for the current task, and (2) constructing the initial state distribution as a mixture of the optimal state visitation distribution of the previous task and the original initial state distribution of interest. In a nutshell, \ours{} mixes in the distribution of the optimal policy of the preceding task to the initial distribution to make sure that such distribution is close to the optimal state visitation distribution of the current task, reducing the density mismatch ratio and yielding better sample complexity.

We illustrate the intuition of \ours{} in \cref{fig:illustration}. We adopt the contextual MDP formulation, where we assume each MDP, $\mc M_\omega$, is uniquely defined by a context $\omega$ in the context space $\mc W$, and we are given a curriculum $\{\omega_k\}_{k=0}^K$, with the last MDP, $\mc M_{\omega_K}$, being the MDP of interest. Our main results require a Lipschitz continuity assumption on the context-dependent reward function $r_\omega$ and a fixed transition dynamics model across all contexts. We show that learning $\pi^\star_{K}$ by recursively rolling in with a near-optimal policy for $\omega_k$ to construct the initial distribution $\mu_{k+1}$ for the next context $\omega_{k+1}$, can have a smaller minimum required sample complexity compared with learning $\pi^\star_{\omega_K}$ from scratch directly. In particular, we show that when an appropriate sequence of contexts is selected, we can reduce the minimum required iteration and sample complexity bounds of entropy-regularized softmax policy gradient (with an inexact stochastic estimation of the gradient) from an original exponential dependency on the state space size, as suggested by~\citet{ding2021beyond}, to a polynomial dependency. We also prescribe a practical implementation of \ours{}.

Our contributions are as follows. We introduce \ours{}, a simple algorithm that facilitates single-task learning by recasting it as a multi-task problem. Theoretically, we show that under the entropy-regularized softmax policy gradient (PG) setting, our algorithm reduces the exponential complexity bound to a polynomial dependency on $S$.  Empirically, we verify our theory on a tabular MDP and provide a practical implementation of \ours{} that can accelerate curriculum learning in the tabular environment and a range of simulated robotic tasks.

\section{Related Work}
\label{sec:related}
\paragraph{Convergence of policy gradient methods.}
Theoretical analysis of policy gradient methods has a long history~\citep{williams1992simple,sutton1999policy,konda1999actor,kakade2002approximately,peters2008natural}. Motivated by the recent empirical success~\citep{schulman2015trust, schulman2017proximal} in policy gradient (PG) methods, the theory community has extensively studied the convergence of PG in various settings~\citep{fazel2018global,agarwal2021theory, agarwal2020pc, bhandari2019global, mei2020global, zhang2020global, agarwal2020pc, zhang2020sample, li2021softmax, cen2021fast, ding2021beyond, yuan2022general, moskovitz2022towards}. \citet{agarwal2021theory} established the asymptotic global convergence of policy gradient under different policy parameterizations. We extend the result of entropy regularized PG with stochastic gradient~\citep{ding2021beyond} to the contextual MDP setting. In particular, our contextual MDP setting reduces the exponential state space dependency w.r.t. the iteration number and per iteration sample complexity suggested by~\citet{ding2021beyond} to a polynomial dependency. While there exists convergence analyses on other variants of PG that produce an iteration number that does not suffer from an exponential state space dependency~\citep{agarwal2021theory,mei2020global}, they assume access to the {\em exact} gradient during each update of PG. In contrast, we assume a stochastic estimation of the gradient.

\paragraph{Exploration bonuses.} A number of prior works have shown that one can achieve a polynomial complexity of learning an optimal policy with effective exploration methods~\citep{azar2017minimax,jin2018q,du2019provably,misra2020kinematic,agarwal2020pc,zhang2020almost}. The computational efficiency suggested by our work is different from some of the aforementioned prior methods that rely on adding exploration bonuses~\citep{azar2017minimax,jin2018q,agarwal2020pc,zhang2020almost}, as we assume access to a ``good'' curriculum which ensures the optimal policy of the next context is not too different from the optimal policy of the current context while eschewing exploration bonuses entirely.

\paragraph{Contextual MDPs.}
Contextual MDPs (or MDPs with side information) have been studied extensively in the theoretical RL literature~\citep{abbasi2014online,hallak2015contextual,dann2019policy,jiang2017contextual,modi2018markov,sun2019model,dann2019policy,modi2020sample}. We analyze the iteration complexity and sample complexity of (stochastic) policy gradient methods, which is distinct from these prior works that mainly focus on regret bounds~\citep{abbasi2014online,hallak2015contextual,dann2019policy} and PAC bounds~\citep{jiang2017contextual,modi2018markov,sun2019model,dann2019policy,modi2020sample}. Several works assumed linear transition kernel and reward model (or generalized linear model~\citep{abbasi2014online}) with respect to the context~\citep{abbasi2014online,modi2018markov,dann2019policy,modi2020sample,belogolovsky2021inverse}. These assumptions share similarity to our assumptions --- we have a weaker Lipschitz continuity assumption with respect to the context space (since linear implies Lipschitz) on the reward function and a stronger shared transition kernel assumption.

\paragraph{Curriculum learning in reinforcement learning.}
Curriculum learning is a powerful idea that has been widely used in RL~\citep{florensa2017reverse, kim2018screenernet, omidshafiei2019learning, ivanovic2019barc, akkaya2019solving, portelas2020teacher, bassich2020curriculum, fang2020adaptive, klink2020self, dennis2020emergent, parker2022evolving, liu2022revolver} (also see \citep{narvekar2020curriculum} for a detailed survey). Although curricula formed by well-designed reward functions~\citep{vinyals2019grandmaster, OpenAI_dota, berner2019dota, ye2020towards, zhai2022computational} are usually sufficient given enough domain knowledge, tackling problems with limited domain knowledge requires a more general approach where a suitable curriculum is automatically formed from a task space. In the goal-conditioned reinforcement learning literature, this corresponds to automatic goal proposal mechanisms~\citep{florensa2018automatic, warde2018unsupervised, sukhbaatar2018learning, ren2019exploration, ecoffet2019go, hartikainen2019dynamical, pitis2020maximum, zhang2020automatic, openai2021asymmetric, zhang2021c}. The practical instantiation of this work is also similar to~\cite{bassich2020curriculum, liu2022revolver}, where a curriculum is adopted for learning a progression of a set of tasks. 
\citet{klink2022boosted} also analyzed the theoretical benefits of curriculum learning in RL, but is primarily concerned with the problem of representations for value functions when utilizing approximate value iteration methods for curriculum learning. This is accomplished by using boosting to increase the effective capacity of the value function estimator. In contrast, our method does not make any prescription in regard to the representation, but is aimed at studying sample complexity and exploration, showing that "rolling in" with the previous policy and then collecting data with the new policy leads to good sample complexity. In principle, we could even imagine in future work combining the representation analysis in \citet{klink2022boosted} with the discussion of state coverage in our analysis.

\paragraph{Learning conditional policies in multi-task RL.}
Multi-task RL~\citep{tanaka2003multitask} approaches usually learn a task-conditioned policy that is shared across different tasks~\citep{rusu2015policy, rajeswaran2016epopt, andreas2017modular, finn2017model, D'Eramo2020Sharing, yu2020gradient, ghosh2021learning, kalashnikov2021mt, agarwal2022provable}. Compared to learning each task independently, joint training enjoys the sample efficiency benefits from sharing the learned experience across different tasks as long as the policies generalize well across tasks. To encourage generalization, it is often desirable to condition policies on low dimensional feature representations that are shared across different tasks instead (e.g., using variational auto-encoders~\citep{nair2018visual, pong2019skew, nair2020contextual} or variational information bottleneck~\citep{goyal2018transfer, Goyal2020The, mendonca2021discovering}). The idea of learning contextual policies has also been discussed in classical adaptive control literature~\citep{sastry1990adaptive, tao2003adaptive, landau2011adaptive, aastrom2013adaptive, goodwin2014adaptive}.  Different from these prior works which have been mostly focusing on learning policies that can generalize across different tasks, our work focuses on how the near-optimal policy from a learned task could be used to help the learning of a similar task.

\section{Preliminaries}
\label{sec:preliminary}
We consider the contextual MDP setting, where a contextual MDP, $\mc{M}_{\omegaspace} = (\omegaspace, \mc S, \mc A, \mb P, r_\omega, \gamma, \rho)$, consists of a context space $\omegaspace$, a state space $\mc S$, an action space $\mc A$, a transition dynamic function $\mb P: \mc S \times \mc A  \rightarrow \mc P(\mc S)$ (where $\mc P(X)$ denotes the set of all probability distributions over set $X$), a context-conditioned reward function $r: \omegaspace \times \mc S \times \mc A \rightarrow [0,1]$, a discount factor $\gamma \in (0, 1]$, and an initial state distribution of interest $\rho$. For convenience, we use $S = |\mc S|, A = |\mc A|$ to denote the number of states and actions. While some contextual MDP formulations~\citep{hallak2015contextual} have context-conditioned transition dynamics and reward functions, we consider the setting where only the reward function can change across contexts. We denote $r_\omega$ as the reward function conditioned on a fixed $\omega \in \omegaspace$ and $\mc M_{\omega} = (\mc S, \mc A, \mb P, r_\omega, \gamma, \rho)$ as the MDP induced by such fixed reward function. We use $\pi(a|s): \mc S \rightarrow \mc P(\mc A)$ to denote a policy and we adopt the softmax parameterization:
${\pi_\theta(a|s) = \frac{\exp\brac{\theta(s,a)}}{\sum_{a^\prime}\exp\brac{\theta(s,a^\prime)}}}$, where ${\theta:\mc S\times \mc A\mapsto \bb R}$. We use $d^\pi_\rho(s) := (1-\gamma) \sum_{t=0}^{\infty} \gamma^t \bb P^\pi(s_t = s|s_0\sim\rho)$ to denote the discounted state visitation distribution and ${V^\pi_\omega :=  \bb E\left[\sum_{t=0}^{\infty} \gamma^t r_\omega(s_t, a_t) \right] + \alpha \bb H(\rho,\pi)}$ to denote the entropy regularized discounted return on $\mc M_{\omega}$, where $\bb H(\rho,\pi) := {\bb E}_{s_0\sim \rho,a_h\sim\pi(\cdot|s_h)}\brac{\sum_{h=0}^\infty-\gamma^h\log\pi(a_h|s_h)}$ is the discounted entropy term. We use $\pi^\star_\omega := \arg\max_\pi V^\pi_\omega$ to denote an optimal policy that maximizes the discounted return under $\mc M_\omega$.  We assume all the contextual reward functions are bounded within $[0,1]$: $r_\omega(s,a) \in [0,1],\;\forall \omega \in \mc W, \forall (s,a)\in \mc S\times \mc A$. Similarly to previous analysis~\cite{agarwal2021theory,mei2020global,ding2021beyond}, we assume the initial distribution $\rho$ for PG or stochastic PG satisfies $\rho(s)>0,\forall s\in \mc S$. Supposing we are given a curriculum $\{\omega_k\}_{k=0}^K$, where the last context $\omega_K$ defines $\mc M_{\omega_K}$ the MDP of interest, our goal is to show that sequentially solving $\pi_{\omega_k}^\star$ for $k=0,1,\dots,K$, enjoys better computational complexity and sample complexity than learning $\pi^\star_{\omega_K}$ problem $\mc M_{\omega_K}$ from scratch.

\subsection{Assumptions}
As we will show in \cref{sec:results}, if there is a curriculum $\{\omega_k\}_{k=0}^K$ where the optimal policies $\pi^\star_{\omega_{k}},\pi^\star_{\omega_{k+1}}$ with respect to two consecutive contexts $\omega_k,\omega_{k+1}$ are close enough to each other in terms of their state visitation distributions, using an $\eps$-optimal policy of $\omega_k$ as an initialization allows us to directly start from the near-optimal regime of $\omega_{k+1}$, hence only requiring polynomial complexity to learn $\pi^\star_{\omega_{k+1}}$. We describe our curriculum assumptions as follows.

\begin{assumption}[Lipschitz reward in the context space]
\label{assump:LipReward}
The reward function is Lipschitz continuous with respect to the context: $\max_{s,a}\abs{ r_\omega(s,a)-r_{\omega^\prime}(s,a)}\leq L_r\norm{\omega-\omega^\prime}{2}$, $\forall \omega,\omega^\prime \in \omegaspace$.
\end{assumption}
Intuitively, \cref{assump:LipReward} defines the similarity between two tasks via a Lipschitz continuity in the context space. Similar Lipschitz assumptions also appears in~\cite{abbasi2014online,modi2018markov,dann2019policy,modi2020sample,belogolovsky2021inverse}.

\begin{assumption}[Similarity of Two Contexts]
\label{assump:CloseEnoughContext}
The curriculum $\{\omega_k\}_{k=0}^K$ satisfies $\max_{0\leq k\leq K-1}\norm{\omega_{k+1}-\omega_k}{2}\leq O\paren{S^{-2}}$, and we have access to a near-optimal initialization $\theta_0^{(0)}$ for learning $\pi^\star_{\omega_0}$ (formally defined in \cref{sec:results-main}).
\end{assumption}
At first glance, the near-optimal initialization $\theta_0^{(0)}$ for the first task $\omega_0$ in the curriculum (suggested by~\cref{assump:CloseEnoughContext}) may seem like a strong assumption, but in many practical settings, it could be quite easy to obtain. For example, if the tasks correspond to reaching different goals, the curriculum might start with a goal right on top of the starting state, and therefore trivially easy to learn. As another example, if the task is a locomotion task and $\omega$ contexts correspond to target velocities, $\omega_0$ might correspond to a velocity of zero, corresponding to standing still.

\cref{assump:LipReward} and \cref{assump:CloseEnoughContext} together quantify the maximum difference between two consecutive tasks $\mc M_{\omega_{k-1}},\mc M_{\omega_k}$, in terms of the maximum difference between their reward function, which plays a crucial role in reducing the exponential complexity to a polynomial one. We will briefly discuss intuition in the next section.

\subsection{Prior Analysis on PG with stochastic gradient}
\label{subsec:spg-prior-results}
\citet{ding2021beyond} proposed a two-phased PG convergence analysis framework with a stochastic gradient. In particular, the author demonstrates that with high probability, stochastic PG with arbitrary initialization achieves an $\eps$-optimal policy can be achieved with iteration numbers of $T_1, T_2$ and per iteration sample complexities of $B_1, B_2$ in two separate phases where $T_1 = \wt{\Omega}\paren{S^{2S^3}},T_2 = \wt{\Omega}\paren{S^{3/2}}$ ($\wt{\Omega}(\cdot)$ suppresses the $\log S$ and terms that do not contain $S$) and $B_1 = \wt{\Omega}\paren{S^{2S^3}},B_2 = \wt{\Omega}\paren{S^5}$, respectively, and PG enters phase 2 only when the updating policy becomes $\eps_0$-optimal, where $\eps_0$ is a term depending on $S$ (formally defined by \eqref{eq:complexity_constants} in \cref{subsec:appendix-main-result-spg}). For completeness, we restate
the main theorem of \citet{ding2021beyond} in \cref{thm:spg-complexity}, provide the details of such dependencies on $S$ in \cref{coro:spg_complexity}, and describe the two-phase procedure in \cref{algo:SPG}. The main implication of the two-phase results is that, when applying SPG to learn an optimal policy from an arbitrary initialization, we suffer from exponential complexities, unless the initialization is $\eps_0$-optimal. We will now discuss how \cref{assump:LipReward} and \cref{assump:CloseEnoughContext} enable an $\eps_0$-optimal initialization for every $\omega_k$, reducing the exponential complexities to  polynomial complexities.
\section{Theoretical Analysis} 
\label{sec:results}
In this section, we first introduce \ours{}, a simple algorithm that accelerates policy learning under the contextual MDP setup by bootstrapping new context learning with a better initial distribution (\cref{algo:PGSPG_context}). Then, we provide the total complexity analysis of applying \ours{} to stochastic PG for achieving an $\eps$-optimal policy. Finally, we validate our theoretical results on a tabular MDP.
\subsection{\ours{}}
\label{sec:results-rollin}
The theoretical version of \ours{} is provided in \cref{algo:PGSPG_context}. The intuition behind \ours{} is that when two consecutive contexts in the curriculum $\{\omega_k\}_{k=1}^K$ are close, their optimal parameters  $\theta^\star_{\omega_{k-1}},\theta^\star_{\omega_k}$ should be close to each other. Let $\theta_t^{(k)}$ denote the parameters at the $t^\text{th}$ iteration of stochastic PG for learning $\theta^\star_{\omega_k}$. If we initialize  $\theta_0^{(k)}$ as the optimal parameter of the previous context $\theta_{\omega_{k-1}}^\star$ (line 5 in \cref{algo:PGSPG_context}), and set the initial distribution $\mu_{k}$ as a mixture of the optimal state visitation distribution of the previous context $d^{\pi_{\omega_{k-1}}^{\star}}_{\mu_{k-1}}$ and the original distribution of interest $\rho$ with $\beta \in (0, 1)$ (line 6 in \cref{algo:PGSPG_context}),
\begin{equation}
    \label{eq:mu_update}
    \mu_{k} =
    \beta d^{\pi_{\omega_{k-1}}^{\star}}_{\mu_{k-1}} + (1-\beta)\rho,
\end{equation}
then we can show that stochastic PG enjoys a faster convergence rate. This is because setting $\theta_0^{(k)} = \theta^\star_{k-1}$ ensures a near-optimal initialization for learning $\omega_k$, and setting $\mu_{k}$ as the mixture further improves the rate of convergence by decreasing the density mismatch ratio $\Big\|d^{\pi^\star_{\omega_k}}_{\mu_k}/\mu_k\Big\|_{\infty}$
(a term with that influences the convergence rate).
\subsection{Main Results}
\label{sec:results-main}
We now discuss how to use a sequence of contexts to learn the target context $\omega_K$ with provable efficiency given a near-optimal  policy $\pi_{\theta^{(0)}_0}$ of the initial context $\omega_0$, without incurring an exponential dependency on $S$ (as mentioned in \cref{subsec:spg-prior-results}). Our polynomial complexity comes as a result of enforcing an $\eps_0$-optimal initialization ($\eps_0$ is the same as \cref{subsec:spg-prior-results} and \eqref{eq:complexity_constants}) for running stochastic PG (line 6 of \cref{algo:PGSPG_context}). Hence, stochastic PG directly enters phase 2, with a polynomial dependency on $S$.

Our main results consist of two parts. We first show that when two consecutive contexts $\omega_{k-1},\omega_k$ are close enough to each other, using \ours{} for learning $\theta^\star_{k}$ with initialization $\theta_0^{(k)} = \theta^\star_{\omega_{k-1}}$ and applying an initial distribution $\mu_k=\beta d^{\pi_{\omega_{k-1}}^{\star}}_{\mu_{k-1}} + (1-\beta)\rho$ improves the convergence rate. Specifically, the iteration number and complexity for learning $\theta^\star_{\omega_k}$ from $\theta^\star_{\omega_{k-1}}$ are stated as follows: 
\begin{algorithm}[t]
  \caption{Provably Efficient Learning via \ours{}}\label{algo:PGSPG_context}
  \begin{algorithmic}[1]
    \State \textbf{Input:} $\rho$, $\{\omega_k\}_{k=0}^K$, $\mc M_{\omegaspace}$, $\beta\in(0,1)$, $\theta^{(0)}_{0}$.
    \State Initialize $\mu_0 = \rho$.
    \State Run stochastic PG (\cref{algo:SPG}) with initialization $\theta_0^{(0)}$, $\mu_{0},\mc M_{\omega_{0}}$ and obtain $\theta^\star_{\omega_0}$.
    \For{$k=1,\dots,K$}
    \State {\color{purple} Set $\theta^{(k)}_1 = \theta^\star_{\omega_{k-1}}$.\Comment{$\pi_{\theta^\star_{\omega_{k-1}}}$ is optimal for $\omega_{k-1}$.}}
    \State {\color{purple} Set $\mu_{k} = \beta d^{\pi_{\omega_{k-1}}^{\star}}_{\mu_{k-1}} + (1-\beta)\rho$. }
    \State Run stochastic PG (\cref{algo:SPG}) with initialization $\theta_0^{(k)}$, $\mu_{k},\mc M_{\omega_{k}}$ and obtain $\theta^\star_{\omega_k}$.
    \EndFor
    \State \textbf{Output:} $\theta_{\omega_K}^\star$
  \end{algorithmic}
\end{algorithm}

\begin{table*}[t]
\centering
\begin{minipage}[b]{2\columnwidth}
\centering
\small
\begin{tabular}{cc|ccccccc}
\toprule
 & \multicolumn{1}{c}{Entropy Coeff.} & Baseline & \ours{} \\ \midrule
Hard & $\alpha = 0.01$ & $0.500\pm 0.000$ & {\color{purple}$0.562 \pm 0.000$} \\
& $\alpha=0.001$& $0.856\pm0.006$ & {\color{purple} $1.000\pm0.000$}  \\ \midrule
Easy & $\alpha=0.01$ & $0.944 \pm 0.003$ & {\color{purple}$1.000\pm0.000$} \\
& $\alpha = 0.001$ &  {\color{purple}$1.000\pm0.000$}& {\color{purple}$1.000\pm0.000$} \\
\bottomrule
\end{tabular}
\hfill
\begin{tabular}{cc|ccccccc}
\toprule
& \multicolumn{1}{c}{Entropy Coeff.} &  Baseline & \ours{}\\ \midrule
Hard & $\alpha = 0.01$ & $0.000\pm 0.000$ & $0.000\pm 0.000$\\
& $\alpha=0.001$& $0.424\pm0.023$ & {\color{purple} $1.067\pm0.000$}  \\ \midrule
Easy & $\alpha=0.01$ & $4.093 \pm 0.224$ & {\color{purple}$7.374\pm0.216$} \\
& $\alpha = 0.001$ &  $10.536\pm0.002$ & {\color{purple}$10.620\pm0.002$} \\
\bottomrule
\end{tabular}
\caption{\footnotesize \textbf{Curriculum progress $\kappa$ (left) and final return $V^\pi$ (right) on the four-room navigation with stochastic PG.} Both metrics are reported at the $50,000^{\text{th}}$ gradient step. We use a mixing ratio of $\beta=0.75$. Across two entropy coefficients and two reward settings (easy and hard), stochastic PG with \ours{} consistently achieves better curriculum progress and final return. The standard error is computed over 10 random seeds.}
\label{table:numerical}
\end{minipage}
\hfill
\end{table*}
\begin{theorem}[Complexity of Learning the Next Context]
\label{thm:softmax_spg_convergence_context-Main}
Consider the context-based stochastic softmax policy gradient (line 7 of \cref{algo:PGSPG_context}), and suppose \cref{assump:LipReward} and \cref{assump:CloseEnoughContext} hold, then the iteration number of obtaining an $\eps$-optimal policy for $\omega_{k}$ from $\theta^\star_{\omega_{k-1}}$ is $\wt{\Omega}\paren{S}$ and the per iteration sample complexity is $\wt{\Omega}\paren{\frac{L_r}{\alpha(1-\beta)}S^3}$.
\end{theorem}
\cref{thm:softmax_spg_convergence_context-Main} shows that when $\omega_{k-1},\omega_k$ are close enough, \ours{} reduces the minimum required iteration and sample complexity from an exponential dependency of $\wt{\Omega}(S^{2S^3})$ to an iteration number of $\wt{\Omega}(S)$ and per iteration sample complexity of $\wt{\Omega}(S^3)$. It is worth noting that the theorem above only addresses the iteration number and sample complexity for learning $\theta^\star_{\omega_k}$ from $\theta^\star_{\omega_{k-1}}$. \cref{thm:total_complexity-Main} provides the total complexity for learning $\theta_{\omega_K}^{\star}$ from $\theta_{0}^{(0)}$ via recursively applying the results in \cref{thm:softmax_spg_convergence_context-Main}. Before introducing \cref{thm:total_complexity-Main}, we first provide a criterion for the desired initialization of $\theta_0^{(0)}$.
\begin{definition}[Near-optimal Initialization]
    \label{def:near-optimal-init}
    We say $\theta_0$ is a near-optimal initialization for learning $\theta_{\omega}^\star$ if $\theta_0$ satisfies $V^{\pi^\star_{\omega}}_{\omega}(\rho) - V^{\pi_{\theta_0}}_{\omega}(\rho)<\eps_0$ and $\norm{\rho-d_\rho^{\pi^\star_{\omega_0}}}{1}\leq\norm{\omega_1-\omega_0}{2}$.
\end{definition}
Note that in the above definition, $\pi^\star_{\omega_{k}}$ represents the optimal policy of $\omega_k$, and $V^\pi_{\omega_k}$ represents value function of context $\omega_k$ under policy $\pi$. We now introduce the results for the overall complexity:
\begin{theorem}[Main Results: Total Complexity of \ours{}]
\label{thm:total_complexity-Main}
Suppose \cref{assump:LipReward} and \cref{assump:CloseEnoughContext} hold, and $\theta^{(0)}_0$ is a near-optimal initialization, then the total number of iterations of learning $\pi^\star_{\omega_K}$ using \cref{algo:PGSPG_context} is $\Omega(KS)$ and the per iteration sample complexity is $\wt{\Omega}\paren{S^{3}}$, with high probability. 
\end{theorem}
A direct implication of \cref{thm:total_complexity-Main} is that, with a curriculum $\{\omega_k\}_{k=0}^K$ satisfying \cref{assump:LipReward} and \cref{assump:CloseEnoughContext}, one can reduce the daunting exponential dependency on $S$ caused by poor initialization to a polynomial dependency. Admittedly the state space $S$ itself is still large in practice, but reducing the state space $S$ itself requires extra assumptions on $\mc S$, which is beyond the scope of this work. We now provide a sketch proof of \cref{thm:softmax_spg_convergence_context-Main} and \cref{thm:total_complexity-Main} in the next subsection and leave all the details 
to \cref{subsec:appendix-learning-next-context} and  \cref{subsec:appendix-learning-final-context} respectively.
\subsection{Proof Sketch}
\paragraph{Sketch proof of \cref{thm:softmax_spg_convergence_context-Main}.}
The key insight for proving \cref{thm:softmax_spg_convergence_context-Main} is to show that in MDP $\mc M_{\omega_k}$, the value function with respect to $\pi^\star_{\omega_k},\pi^\star_{\omega_{k-1}}$ can be bounded by the $\ell^2$-norm between $\omega_k$ and $\omega_{k-1}$. In particular, we prove such a relation in \cref{lemma:value_diff_adj_ctx}:
\begin{equation}
    \label{eq:V_diff_k_k-1}
    V^{\pi^\star_{\omega_{k}}}_{\omega_{k}}(\rho) - V^{\pi^\star_{\omega_{k-1}}}_{\omega_{k}}(\rho)\leq \frac{2L_r\norm{\omega_k-\omega_{k-1}}{2}}{(1-\gamma)^2}.   \vspace{-0.15cm}
\end{equation}
By setting $\theta_0^{(k)}=\theta^\star_{\omega_{k-1}}$, Equation~\eqref{eq:V_diff_k_k-1} directly implies $V^{\pi^\star_{\omega_{k}}}_{\omega_{k}}(\rho) - V^{\theta_0^{(k)}}(\rho)\leq \frac{2L_r\norm{\omega_k-\omega_{k-1}}{2}}{(1-\gamma)^2}$. As suggested by~\citet{ding2021beyond} stochastic PG can directly start from stage 2 with polynomial complexity of $T_2 = \wt{\Omega}(S), B_2 = \wt{\Omega}(S^5)$, if $V^{\pi^\star_{\omega_{k}}}_{\omega_{k}}(\rho) - V^{\theta_0^{(k)}}(\rho)\leq\eps_0$, where $\eps_0$ (formally defined in Equation~\eqref{eq:complexity_constants} in \cref{subsec:appendix-main-result-spg}) is a constant satisfying $\eps_0=O(S^{-2})$. Hence, by enforcing two consecutive contexts to be close enough $\norm{\omega_k-\omega_{k-1}}{2}\leq O(S^{-2})$, we can directly start from a near-optimal initialization with polynomial complexity with respect to $S$. It is worth highlighting that the per iteration sample complexity $B_2$ shown by~\citet{ding2021beyond} scales as $\wt{\Omega}(S^5)$, while our result in \cref{thm:softmax_spg_convergence_context-Main} only requires a smaller sample complexity of $\wt{\Omega}(S^3)$. Such an improvement in the sample complexity comes from line 6 of \ours{}: $\mu_{k}=\beta d^{\pi_{\omega_{k-1}}^{\star}}_{\mu_{k-1}} + (1-\beta)\rho$. Intuitively, setting $\mu_k$ as $\beta d^{\pi_{\omega_{k-1}}^{\star}}_{\mu_{k-1}} + (1-\beta)\rho$ allows us to provide an upper bound on the density mismatch ratio:
\begin{equation}
    \label{eq:density_mismatch_ratio}
    \norm{d^{\pi^\star_{\mu_k}}_{\mu_k}/\mu_k}{\infty}\leq \wt{ O}\paren{\frac{L_r}{\alpha(1-\beta)}\Delta^k_{\omega}S},
\end{equation}
where $\Delta_\omega^k=\max_{1\leq i\leq k}\norm{\omega_i-\omega_{i-1}}{2}$. Since the sample complexity $B_2$ (provided in \cref{coro:spg_complexity}) contains one multiplier of $\norm{d^{\pi^\star_{\mu_k}}_{\mu_k}/\mu_k}{\infty}$, setting $\Delta_\omega^k=O(S^{-2})$ immediately reduces the complexity by an order of $S^2$. The proof of the upper bound of the density mismatch ratio (Equation~\eqref{eq:density_mismatch_ratio}) is provided in \cref{lemma:mismatch_ratio_bound}.

\paragraph{Sketch proof of \cref{thm:total_complexity-Main}.} We obtain \cref{thm:total_complexity-Main} by recursively applying \cref{thm:softmax_spg_convergence_context-Main}. More precisely, we use induction to show that, if we initialize the parameters of the policy as $\theta_0^{(k)}=\theta^\star_{\omega_{k-1}}$, when $t=\wt{\Omega}(S)$, $\forall k\in[K]$, we have $V^{\pi^\star_{\omega_{k}}}_{\omega_{k}}(\rho) - V^{\pi_{\theta_t^{(k-1)}}}_{\omega_{k}}(\rho)<\eps_0$. Hence, for any context $\omega_{k},k\in [K]$, initializing  $\theta_0^{(k)}=\theta_t^{(k-1)}$ from learning $\pi^\star_{\omega_{k-1}}$ via stochastic PG after $t=\Omega(S)$ iteration, $\theta_0^{(k)}$ will directly start from the efficient phase 2 with polynomial complexity. Hence, the total iteration number for learning the $\theta_K^\star$ is $\Omega(KS)$, and the per iteration sample complexity remains the same as \cref{thm:softmax_spg_convergence_context-Main} ($\wt{\Omega}\paren{S^{3}}$).

\begin{figure}[t]
    \centering
    \includegraphics[width=0.23\textwidth]{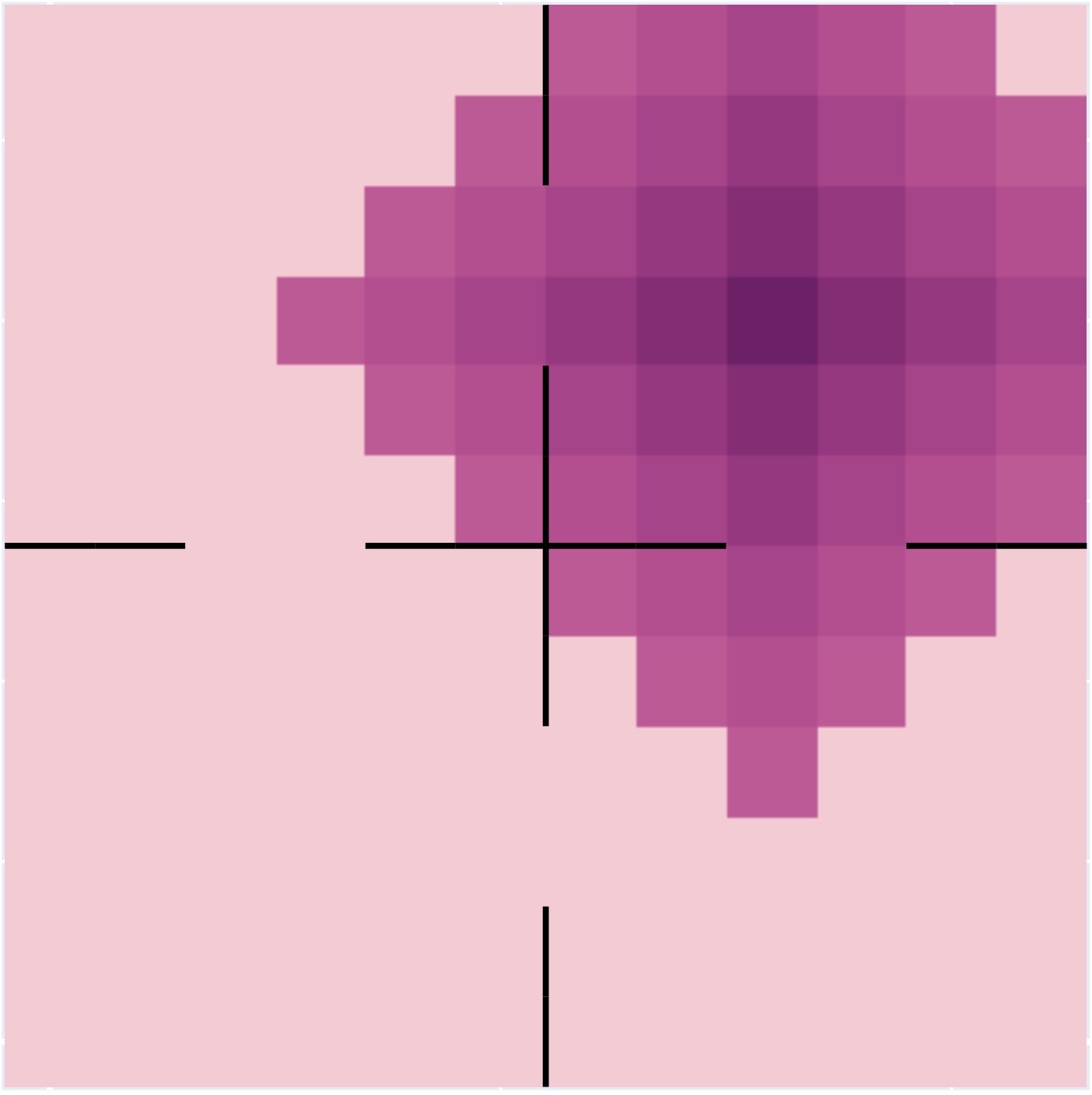}\hfill
    \includegraphics[width=0.23\textwidth]{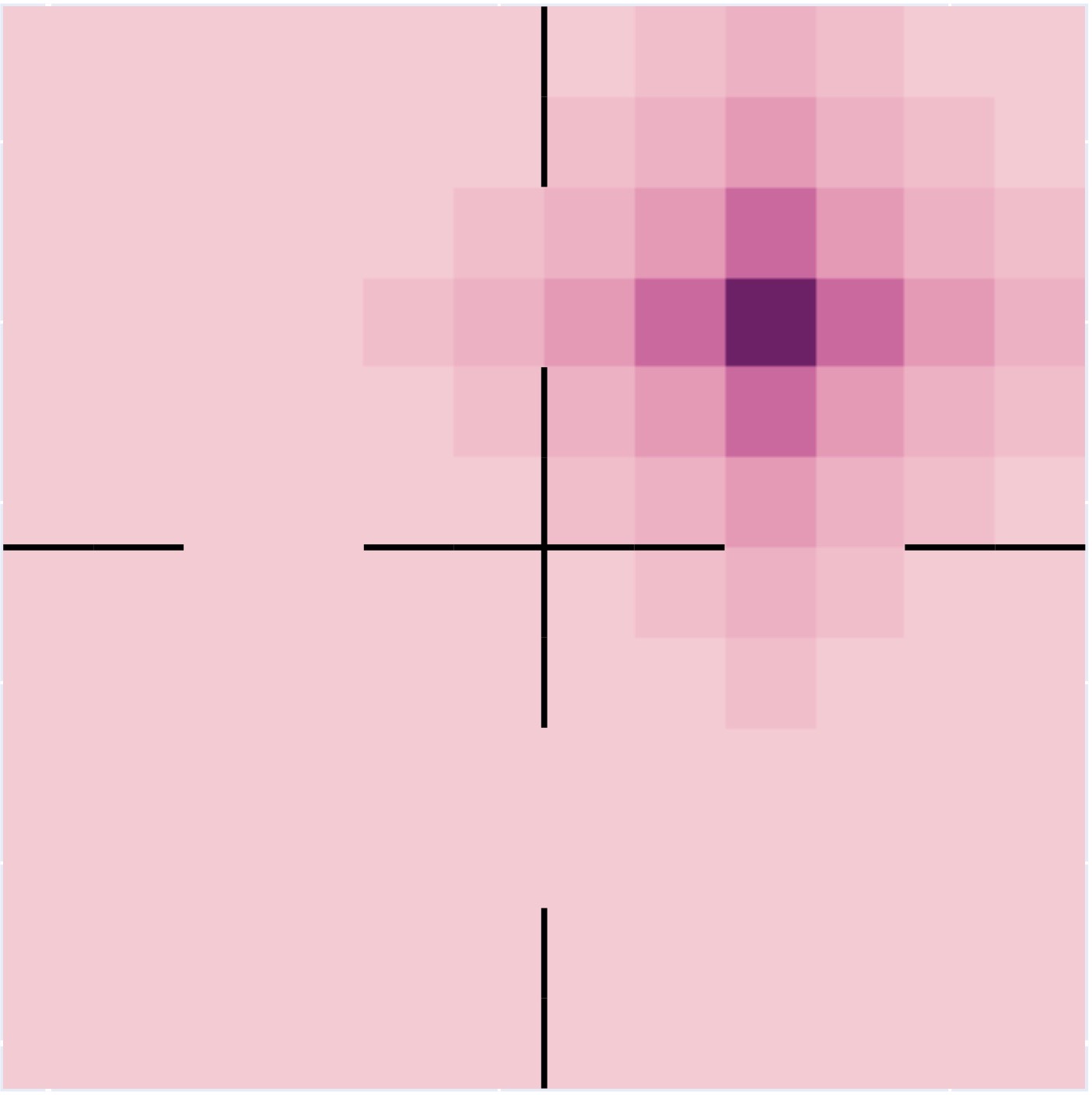}
    \caption{\footnotesize \textbf{Visualization of the two reward functions for the four-room navigation environment.} Left: easy; Right: hard. Darker color indicates a higher reward when the agent reaches the state (with the goal state the darkest). The agent receives $0$ reward when it is too far from the goal location (5 steps for the easy reward function and 4 steps for the hard reward function). The black line indicates walls in the environment where the agent cannot pass through. The reward function visualization in this figure is for the last context. The reward function for other contexts (other goals) is the same but with the reward function shifted according to the goal state.}
    \label{fig:numerical-mdp} \vspace{-6mm}
\end{figure}

\subsection{Numerical Experiments: Four-room Navigation}
To provide empirical support for the validity of our theoretical results, we follow the exact setting that is used in our theoretical analysis and implement \ours{} with stochastic PG and the softmax policy parameterization on a tabular contextual MDP. It is worth highlighting that this is distinct from the practical implementation of \ours{} in \cref{sec:practical} -- the focus there is to apply the idea of \ours{} to design a sample-efficient deep RL algorithm for a more practical setting (e.g., continuous state and action space), whereas the focus here is solely to validate our theory in the theoretical setting. The contextual MDP is a grid world consisting of $12 \times 12$ grid cells where each cell corresponds to a state in the MDP. The agent always starts from the bottom-left corner of the grid world and navigates around to collect rewards. In particular, the agent receives a positive reward when it is close to a pre-defined goal cell/state which is defined by the context. We use a curriculum consisting of $17$ contexts/goals in sequence, $\{\omega_k\}_{k=0}^{16}$, which form a path from the initial state of the agent to a far away goal location, and we switch from the current context to the next one in the curriculum whenever the current goal location is reached with more than $50\%$ probability. We experiment with two different reward functions (visualized in Figure \ref{fig:numerical-mdp}). Table \ref{table:numerical} summarizes the results of our experiments. \ours{} is able to consistently improve upon the vanilla stochastic PG baseline (across two different entropy coefficients and two reward functions with varying difficulties) in terms of the curriculum progress and the final return. This verifies that \ours{} can indeed improve the sample complexity of stochastic PG in a curriculum learning setting, validating our theory. See more implementation details of the numerical experiments in Appendix \ref{appendix:numerical}.

\section{Practical Implementation of \ours{}}
\label{sec:practical}
We have shown empirical evidence through numerical experiments that \ours{} can lead to sample complexity reduction under our theoretical setting (tabular MDP with discrete action space and state space, softmax parameterization and entropy regularized objective). Now, we introduce a practical implementation of \ours{} using Soft-Actor-Critic~\citep{haarnoja2018soft} such that \ours{} can be applied to more practical problems with continuous action space and state space. SAC can be seen as a variant of entropy-regularized stochastic PG with the addition of the critics to reduce gradient variance. Recall that in the theoretical analysis, we learn a separate policy for each context that can start from the near-optimal state distribution of the previous context to achieve a good return under the current context. However, in practice, we usually would want to have a policy that can directly start from the initial distribution $\rho$ to obtain a good return for the final context $\omega_K$. In order to learn such a policy, we propose to have two context-conditioned RL agents training in parallel, where the first agent $\pi_{\mathrm{main}}$ is the main agent that eventually will learn to achieve a good return from $\rho$, and the second agent $\pi_{\mathrm{exp}}$ is an exploration agent that learns to achieve a good return under the current context from the near-optimal state density of the previous context. Another purpose of the exploration agent (as the name suggests) is to provide a better exploration experience for the main agent to learn the current context better. This is made convenient by using an off-policy RL agent where the main agent can learn from the data collected by the exploration agent. 

Specifically, for each episode, there is a probability of $\beta$ where we run the main agent conditioned on the previous context for a random number of steps until we switch to the exploration agent to collect experience for the current context until the episode ends. Otherwise, we directly run the main agent for the entire episode. Both agents are trained to maximize the return under the current context. Whenever the average return of the last 10 episodes exceeds a performance threshold $R$, we immediately switch to the next context and re-initialize the exploration agent and its replay buffer. A high-level description is available in \cref{algo:impl} (a more detailed version in \cref{algo:detailed-impl}).

\begin{algorithm}[ht]
  \caption{Practical Implementation of \ours{}}\label{algo:impl}
  \begin{algorithmic}[1]
    \State \textbf{Input:} $\{\omega_k\}_{k=0}^K$: input curriculum, $R$: near-optimal threshold, $\beta$: roll-in ratio, $H$: horizon, $\gamma$: discount factor. 
    \State Initialize $\mathcal{D} \leftarrow \emptyset,\mathcal{D}_{\mathrm{exp}} \leftarrow \emptyset,k \leftarrow 0$, and two SAC agents $\pi_\mathrm{main}$ and $\pi_\mathrm{exp}$.
    \For {each episode}
        \If {average return of the last 10 episodes under  context $\omega_k$ is greater than $R$ }
            \State $k \leftarrow k + 1$, $\mathcal{D}_{\mathrm{exp}} \leftarrow \emptyset$, and re-initialize the exploration agent $\pi_{\mathrm{exp}}$
        \EndIf
        \If {$k > 0$ and with probability of $\beta$}
            \State $h \sim \mathrm{Geom}(1 - \gamma)$ (truncated at $H$)
            \State run $\pi_{\mathrm{main}}(a|s, \omega_{k-1})$ from the initial state for $h$ steps and switch to $\pi_{\mathrm{exp}}(a|s, \omega_k)$ until the episode ends to obtain trajectory $\tau_{0:H}= \{s_0, a_0, r_0, s_1, a_1, \cdots, s_H\}$.
            \State record $\tau_{0:H}$ in $\mc{D}$, and $\tau_{h:H}$ in $\mc{D}_{\mathrm{exp}}$.
        \Else
            \State run $\pi_{\mathrm{main}}(a|s, \omega_{k})$ to obtain trajectory $\tau_{0:H}$ and record $\tau_{0:H}$ in $\mc{D}$.
        \EndIf
        \State at each environment step in the episode, update $\pi_{\mathrm{main}}(\cdot | \cdot, \omega_k)$ using $\mathcal{D}$ and $\pi_{\mathrm{exp}}(\cdot | \cdot, \omega_k)$ using $\mathcal{D}_{\mathrm{exp}}$.
    \EndFor
    \State \textbf{Output:} $\pi_{\mathrm{main}}$
  \end{algorithmic}
\end{algorithm}
\section{Experimental Results}

While the focus of our work is on developing a provably efficient approach to curriculum learning, we also conduct an experimental evaluation of our practical implementation of \ours{} with soft actor-critic (SAC)~\citep{haarnoja2018soft} as the RL algorithm on several continuous control tasks including a goal reaching task and four non-goal reaching tasks with oracle curricula.

\begin{table*}[t]
\centering
\begin{minipage}[b]{2\columnwidth}
\centering
\small
\begin{tabular}{cc|cccc}
\toprule
\multicolumn{1}{c}{}&\multicolumn{1}{c}{}&
\multicolumn{2}{c}{w/o Geometric Sampling}&
\multicolumn{2}{c}{w/ Geometric Sampling}\\
\cmidrule(lr){3-4}\cmidrule(lr){5-6} \multicolumn{1}{c}{Setting} & \multicolumn{1}{c}{Method} & $\Delta=1/24$& $\Delta=1/12$& $\Delta=1/24$& $\Delta=1/12$\\ \midrule
Vanilla & Baseline & $0.40\pm 0.02$ & $0.36\pm0.00$ & $0.82\pm0.08$ & $0.38\pm0.03$\\
& \ours{} & {\color{purple} $0.49\pm0.04$} & {\color{purple} $0.44\pm0.01$} & {\color{purple} $0.92\pm0.02$}& {\color{purple} $0.55\pm0.04$} \\ \midrule
Relabeling & Baseline & $0.89 \pm 0.03$ & {${0.66\pm0.04}$} & {$0.76\pm0.02$} & $0.72\pm0.03$\\
& \ours{} & {\color{purple}$0.91\pm0.03$}& {\color{purple}$0.74\pm0.01$} &{\color{purple}${0.78\pm0.01}$}& {\color{purple}$0.73\pm0.00$} \\
\midrule
Go-Explore & Baseline & $0.37 \pm 0.02$ & {${0.38\pm0.01}$} & {$0.82\pm0.07$}  & {$0.42\pm0.03$}\\
Noise = 0.1& \ours{} &{\color{purple}$0.52\pm0.07$}& {${0.38\pm0.01}$} & {\color{purple}$0.95\pm0.02$} & {\color{purple}$0.43\pm0.02$}\\
\bottomrule
\end{tabular}
\caption{Learning progress $\kappa$ at 3 million environment steps with varying curriculum step size $\Delta$ of different settings of goal reaching in \texttt{antmaze-umaze}. We pick $\beta = 0.1$ for all experiments using \ours{}, the results of using other $\beta$s, $\Delta$s, and exploration noises can be found in \cref{table:vanilla-antumaze}, \cref{table:relabeling-antumaze}, and \cref{table:go-explore-antumaze} in \cref{appendix:learning-curves-goal-reaching-oracle}. The standard error is computed over 8 random seeds.}

\label{table:summary-antumaze}
\end{minipage}
\hfill
\end{table*}

\subsection{Goal Reaching with an Oracle Curriculum}
\label{subsec:goal-reaching-oracle}
We adopt the \texttt{antmaze-umaze} environment~\citep{fu2020d4rl} for evaluating the performance of \ours{} in goal-reaching tasks. We use a hand-crafted path of contexts, where each specifies a goal location (as shown in \cref{appendix:oracle-details}, \cref{fig:antmaze-path-cshape}). We consider a path of contexts $\omega(\kappa)$ parameterized by $\kappa \in [0, 1]$ where $\omega(0) = \omega_0$ and $\omega(1) = \omega_K$, and step through the contexts along the path with a fixed step size $\Delta$. See \cref{appendix:oracle-details} for more implementation details.

We combine \ours{} with a variety of prior methods, and we evaluate the following conditions: (1) standard goal reaching; (2) goal reaching with goal relabeling~\citep{andrychowicz2017hindsight}; (3) goal reaching with an exploration phase that is similar to Go-Explore~\citep{ecoffet2019go}. For goal relabeling, we adopt a similar relabeling technique as~\citet{pitis2020maximum}, where each mini-batch contains $1/3$ original transitions, $1/3$ transitions with future state relabeling, and $1/3$ transitions with next state relabeling. We implemented the Go-Explore method by adding an additional standard Gaussian exploration noise in the action to the agent for learning the next goal $\omega(k+1)$, once it reaches the current goal $\omega(k)$. We empirically observed that sampling the replay buffer from a geometric distribution with $p=10^{-5}$ (more recent transitions are sampled more frequently) improves the overall performance. Hence, in all future experiments, we compare the performance of \ours{} with classic uniform sampling and the new geometric sampling. We compare the learning speed of \ours{} with parameter $\beta = 0.1$ on three different step sizes $\Delta = \frac{1}{24},\frac{1}{18},\frac{1}{12}$ in \cref{table:summary-antumaze}. 

\textbf{Main comparisons.} We first provide an overview experiments that compares \ours{} with a fixed $\beta=0.1$ on different step sizes $\Delta$ in different settings. In each case, we compare the prior method (vanilla, relabeled, or Go-Explore) with and without the addition of \ours{}. As shown in \cref{table:summary-antumaze}, \ours{} improves the largest value of $\kappa$ reached by the agent
in most presented settings (except Go-Explore with $\Delta=1/12$). This result suggests that \ours{} facilitates goal-conditioned RL with a curriculum, as we only update the learning progress $\kappa$ to $\kappa+\Delta$ when the return of the current policy reaches a certain threshold $R$ (See detailed update of $\kappa$ in \cref{algo:impl}). Note that $\beta=0.1$ does not always produce the best result, we will provide more results comparing different $\beta$s in different settings later in this section, and we leave all the learning curves and detailed tables to \cref{appendix:learning-curves-goal-reaching-oracle}. Note that we do not include the results of directly learning the last context in the \texttt{antmaze-umaze} environment because the agent cannot reach the goal without the aid of a curriculum, which is corroborated by~\cite{pitis2020maximum}. 

\begin{table*}[t]
\centering
\begin{minipage}[b]{\linewidth}
\centering
\small
\begin{tabular}{cc|cccccc}
\toprule
\multicolumn{1}{c}{}&\multicolumn{1}{c}{}&
\multicolumn{3}{c}{Step = $0.5 \times 10^6$}&
\multicolumn{3}{c}{Step = $1.0 \times 10^6$}\\
\cmidrule(lr){3-5}\cmidrule(lr){6-8} \multicolumn{1}{c}{\texttt{Env.}} & \multicolumn{1}{c}{Method} & $\kappa$& $x$-velocity& return & $\kappa$& $x$-velocity& return \\ \midrule
\texttt{walker} & Scratch & n/a & $3.07\pm 0.26$ & $3373.1\pm 170.5$ & n/a & $3.30\pm0.36$ & $4212.3\pm151.4$\\
& Baseline & $0.83\pm 0.03$ & $3.09\pm 0.31$ & $3450.1\pm 307.4$ & $0.92\pm0.03$ & $3.69\pm0.27$ & $4032.3\pm224.3$\\
& \ours{} & {$0.79\pm0.04$} & {${2.83\pm0.31}$} & {${3350.4\pm184.6}$}& {\color{purple} $0.94\pm0.03$} & {$3.62\pm0.26$} & {${4128.8\pm159.6}$}\\ \midrule
\texttt{hopper} & Scratch & n/a & $2.50\pm0.13$ & $2943.6\pm 80.3$ & n/a & $2.55\pm0.12$ & $3073.2\pm137.7$\\
& Baseline & $0.85 \pm 0.02$ & $2.42\pm0.18$ & $3192.5\pm 80.4$ & $0.88\pm0.01$ & $2.58\pm0.16$ & $3386.2\pm124.7$\\
& \ours{} & {$0.82\pm0.03$}& {$2.26\pm0.22$} &{\color{purple}${3148.6\pm160.7}$}& {\color{purple}$0.89\pm0.00$} & {\color{purple} $2.65\pm0.15$} & {\color{purple}${3421.9\pm109.8}$}\\
\midrule
\texttt{humanoid} & Scratch & n/a & $0.24\pm 0.05$  & $2417.1\pm188.2$ & n/a & $0.37\pm 0.05$ & $2763.8
\pm96.5$\\
& Baseline & $0.32 \pm 0.05$ & $0.26\pm 0.05$  & $2910.1
\pm262.9$ & {$0.67\pm0.03$} & $0.39\pm 0.05$ & $3017.2
\pm169.0$\\
& \ours{} &{\color{purple}$0.36\pm0.04$}& {\color{purple}$0.32\pm 0.07$} & {\color{purple}$2939.7\pm392.0$} & {\color{purple}$0.69\pm0.06$} & {\color{purple}$0.46\pm 0.09$} &{\color{purple}$3173.6\pm238.3$}\\
\midrule
\texttt{ant} & Scratch & n/a & $3.60\pm0.49$ & $2910.7\pm 354.3$ & n/a & $4.55\pm0.36$ & $4277.9\pm 120.0$ \\
& Baseline & $0.72 \pm 0.02$ & $3.38\pm0.43$ & $2976.2\pm 252.4$ & {$1.00\pm0.00$} & $4.29\pm0.51$ & $4248.5\pm 88.6$ \\
& \ours{} &{\color{purple}$0.82\pm0.06$}& {\color{purple} $3.85\pm0.41$} & {\color{purple} $3593.1\pm237.8$} & {$1.00\pm0.00$} & {\color{purple} $4.66\pm0.30$} & {\color{purple} $4473.0\pm102.2$}\\
\bottomrule
\end{tabular}
\caption{Learning progress $\kappa$, average $x$-velocity, and average return at 0.5 and 1.0 million environment steps in \texttt{walker}, \texttt{hopper}, \texttt{humanoid}, and \texttt{ant}. The average $x$-velocity and return are estimated using the last 50k environment steps. ``Scratch'' shows the results of directly training the agent with the last context $\omega(1)$. ``Baseline'' indicates $\beta=0$, where we provide the curriculum $\omega(\kappa)$ to the agent without using \ours{}. We pick $\beta = 0.1$ for all experiments using \ours{}, the results of using other $\beta$s can be found in \cref{table:meta-learning-progress}, \cref{table:meta-learning-speed}, and \cref{table:meta-learning-return} in \cref{appendix:learning-curves-meta-learning}. The standard error is computed over 8 random seeds.}

\label{table:summary-meta-learning}
\end{minipage}
\hfill
\end{table*}
\subsection{Non-Goal Reaching Tasks}
For the non-goal tasks, we consider the tasks of gradually increasing the $x$-velocity of a locomotion agent in the following environments: \texttt{walker2d}, \texttt{hopper}, \texttt{humanoid}, and \texttt{ant} in OpenAI gym~\citep{Brockman2016OpenAI}. More specifically,  we set the desired speed range to be $[\lambda\kappa,\lambda(\kappa+0.1))$, where $\lambda$ is a parameter depending on the physics of the agent in different environments and we choose a fixed contextual space with ten discrete contexts: $\kappa\in\{0.1,0.2,\dots,1\}$. The agent receives a higher reward when the $x$-velocity is within the desired speed range and a lower reward otherwise. In each environment, we increase the task difficulty with later curriculum steps (larger $\kappa$), by increasing the near-optimal threshold $R(\kappa)$. Detailed parameters of the desired speed range $\lambda$, near optimal-threshold $R(\kappa)$, and the reward functions are in \cref{appendix:meta-learning-details}.

\textbf{Main comparisons.} We first compare \ours{} with a fixed $\beta=0.1$ at different environment steps: $0.5 \times 10^6$, $1 \times 10^6$. In each case, we compare the learning progress $\kappa$, averaged $x$-velocity, and averaged return, with and without the addition of \ours{}. Note that for the case without \ours{}, we still provide the curriculum to the agent for training. The results in \cref{table:summary-meta-learning} show that \ours{} improves most benchmarks (See detailed update of $\kappa$ in \cref{algo:impl}). Note that $\beta=0.1$ does not always produce the best result, and we provide more results comparing different $\beta$s in different settings later in this section, with learning curves and more detailed tables in \cref{appendix:learning-curves-meta-learning}.

\subsection{Experimental Summary}
We empirically showed that \ours{} improves the performance of one goal-reaching task and four non-goal tasks in different settings. Although \ours{} introduces an extra parameter $\beta$, our experiments show reasonable improvement by simply choosing $\beta=0.1$ or $0.2$. More careful selection of $\beta$ might lead to further improvements.

\section{Discussion and Future Work}
\label{sec:discussion}
We presented \ours{}, a simple algorithm that accelerates curriculum learning under the contextual MDP setup by rolling in a near-optimal policy to bootstrap the learning of new nearby contexts with provable learning efficiency benefits. Theoretically, we show that \ours{} attains polynomial sample complexity by utilizing adjacent contexts to initialize each policy. Since the key theoretical insight of \ours{} suggests that one can reduce the density mismatch ratio by constructing a new initial distribution, it would be interesting to see how \ours{} can affect other variants of convergence analysis of PG (e.g., NPG~\citep{kakade2001natural,cen2021fast} or PG in a feature space~\citep{agarwal2021theory,agarwal2020pc}). On the empirical side, our experiments demonstrate that \ours{} improves the empirical performance of various tasks beyond our theoretical assumptions, which reveals the potential of \ours{} in other practical RL tasks with a curriculum. Our initial practical instantiation of the \ours{} algorithm has a lot of room for future research. First of all, our implementation requires domain-specific knowledge of a ``good'' return value as it currently rely on a fixed return threshold $R$ to determine when we are going to switch from the current context to the next context. Another promising direction is to combine our algorithm with context-based meta-RL methods such as learning to generate sub-goal/context to accelerate the learning of the current sub-goal/context. Finally, our method is not specific to the goal-conditioned settings, which opens up the possibility of applying our algorithm to more challenging domains.

\section{Acknowledgements}
We are thankful to Laura Smith, Dibya Ghosh, Chuer Pan, and other members of the RAIL lab for feedback and suggestions on earlier drafts. QL would like acknowledge the support of the Berkeley Fellowship. YZ would like to thank Jincheng Mei from Google and Yuhao Ding from UC Berkeley for insightful discussions on the related proof. YM would like to acknowledge the support of ONR grants N00014-20-1-2002, N00014-22-1-2102, the joint Simons Foundation-NSF DMS grant \# 2031899, and Tsinghua-Berkeley Shenzhen Institute (TBSI) Research Fund. SL would like to acknowledge Air Force Office of Scientific Research AFOSR FA9550-22-1-0273. The research is supported by Savio computational cluster provided by the Berkeley Research Compute program.

\bibliography{main}
\bibliographystyle{plainnat}

\newpage
\onecolumn
\appendix
\newpage
\section{Generalization Between Different Tasks in the Context Space}
\label{subsec:TentativeMain}
\subsection{Summaries of Notations and Assumptions}
\begin{enumerate}
    \item The maximum entropy RL (MaxEnt RL) objective with initial state distribution $\rho$ in reinforcement aims at maximizing (Equation 15 \& 16 of \cite{mei2020global})
\begin{align}
    \label{eq:MaxEntRLV}
    V^\pi(\rho):= \sum_{h=0}^{\infty} \gamma^{h} \mathbb{E}_{s_0\sim\rho, a_h \sim \pi(a_h | s_h)}\left[r(s_h, a_h)\right]     + \alpha \bb H(\rho,\pi)
\end{align}
and $\bb H(\pi(a_h|s_h))$ is the discounted entropy term 
\begin{equation}
    \label{eq:DiscountedEntropy}
    \bb H(\rho,\pi) := \underset{s_0\sim \rho,a_h\sim\pi(\cdot|s_h)}{\bb E}\brac{\sum_{h=0}^\infty-\gamma^h\log\pi(a_h|s_h)},
\end{equation}
and $\alpha$ is the penalty term. For simplicity, we denote the optimization objective function in \eqref{eq:MaxEntRLV} as $\alpha$-MaxEnt RL. Similar to Equation 18 \& 19 of \cite{mei2020global}, we also define the advantage and $Q$-functions and for MaxEnt RL as
\begin{equation}
    \label{eq:MaxEntRLAQ}
    \begin{split}
        A^\pi(s,a)&:=Q^\pi(s,a)-\alpha\log\pi(s,a)-V^\pi(s),\\
        Q^\pi(s,a)&:=r(s,a)+\gamma \sum_{s^\prime} P(s^\prime|s,a)V^\pi(s).
    \end{split}
\end{equation}
\item We let 
\begin{equation}
    \label{eq:discount_state_visitation}
    d_{s_0}^\pi(s) = (1-\gamma)\sum_{t=0}^\infty \gamma^t \bb P^\pi(s_t=s|s_0),
\end{equation}
to denote the discounted state visitation of policy $\pi$ starting at state $s_0$, and let 
\begin{equation}
    \label{eq:discount_state_visitation_distribution}
    d_\rho^\pi(s) = \bb E_{s\sim\rho}[d_s^\pi(s)]
\end{equation}
denote the initial state visitation distribution under {\bf initial state distribution} $\rho$.
\item We assume the reward functions under all context are bounded within $[0,1]$:
\begin{equation}
    r_\omega(s,a) \in [0,1],\;\forall \omega \in \Omega, \forall (s,a)\in \mc S\times \mc A.
\end{equation}
\item Similar to previous analysis in \cite{agarwal2021theory,mei2020global,ding2021beyond}, we assume the initial distribution $\mu$ for PG/stochastic PG satisfies $\rho(s)>0,\forall s\in \mc S$.
\end{enumerate}

\subsection{Main Results: Mismatch Coefficient Upper Bound}
\label{subsec:appendix-main-result-mismatch}

\begin{lemma}[Density Mismatch Ratio via \ours{}]
\label{lemma:mismatch_ratio_bound}
Assuming $\rho = \mathrm{Unif}(\mc S)$, 
 and $\mu_{k} =\beta d^{\pi_{\omega_{k-1}}^{\star}}_{\mu_{k-1}} + (1-\beta)\rho$ (using \eqref{eq:mu_update} from \ours{}), the density mismatch ratio $\norm{d^{\pi^\star_{\omega_k}}_{\mu_k}/\mu_{k}}{\infty}$ satisfies
\begin{equation}
    \norm{\frac{d^{\pi^\star_{\omega_k}}_{\mu_k}}{\mu_k}}{\infty}\leq\;\wt{ O}\paren{\frac{L_r}{\alpha(1-\beta)}\Delta^k_{\omega} S},
\end{equation}
where $\Delta_\omega^k = \max_{1\leq i\leq k}\norm{\omega_{i}-\omega_{i-1}}{2}$.
\end{lemma}

\begin{proof}
\label{proof:mismatch_ratio_bound}
By \eqref{eq:mu_update} from \ours, we have
\begin{equation}
    \label{eq:connecting_pi_to_mu_derivation_0}
    \begin{split}
        &\norm{\frac{d_{\mu_{k}}^{\pi^\star_{\omega_{k}}}}{\mu_{k}}}{\infty} = \norm{\frac{d_{\mu_{k}}^{\pi^\star_{\omega_{k}}}-d_{\mu_{k-1}}^{\pi^\star_{\omega_{k-1}}}+d_{\mu_{k-1}}^{\pi^\star_{\omega_{k-1}}}}{\mu_{k}}}{\infty}\\
        \overset{(i)}{\leq}&\frac{\norm{d^{\pi^\star_{\omega_{k}}}_{\mu_{k}}-d^{\pi^\star_{\omega_{k-1}}}_{\mu_{k-1}}}{1}}{\min\mu_k} + \norm{\frac{d_{\mu_{k-1}}^{\pi^\star_{\omega_{k-1}}}}{\beta d^{\pi^\star_{\omega_{k-1}}}_{\mu_{k-1}} + (1-\beta)\rho}}{\infty}\\
        \overset{(ii)}{\leq}&\frac{\norm{d^{\pi^\star_{\omega_{k}}}_{\mu_{k}}-d^{\pi^\star_{\omega_{k-1}}}_{\mu_{k-1}}}{1}}{\min\mu_{k}}+\frac{1}{\beta}
    \end{split}
\end{equation}
where inequality $(i)$ holds because of \eqref{eq:mu_update}, and inequality $(ii)$ holds because $\rho(s)\geq0,\forall s\in \mc S$. Now it remains to bound $\norm{d^{\pi^\star_{\omega_{k+1}}}_{\mu_{k+1}}-d^{\pi^\star_{\omega_{k}}}_{\mu_{k}}}{1}$ using the difference $\norm{\omega_{k+1}-\omega_k}{2}$. Let $\bb P_h^{k} = \bb P_h^{\pi^\star_{\omega_{k}}}(s^\prime|s_0\sim \mu _{k})$ denote the state visitation distribution resulting from $\pi_{\omega_k}^\star$ probability starting at $\mu_k$, then we have
\begin{equation}
    \begin{split}
        &\bb P^{k}_h(s^\prime) - \bb P^{k-1}_h(s^\prime) = \sum_{s,a} \paren{\bb P^{k}_{h-1}(s)\pi_{\omega_{k}}^\star(a|s)-\bb P^{k-1}_{h-1}(s)\pi^\star_{\omega_{k-1}}(a|s)}P(s^\prime|s,a)\\
        = & \sum_{s,a}\paren{\bb P^{k}_{h-1}(s)\pi^\star_{\omega_{k}}(a|s) - \bb P^{k}_{h-1}(s)\pi^\star_{\omega_{k-1}}(a|s) + \bb P^{k-1}_{h-1}(s)\pi^\star_{\omega_{k-1}}(a|s) - \bb P^{k-1}_{h-1}(s)\pi^\star_{\omega_{k-1}}(a|s)}P(s^\prime|s,a)\\
        = & \sum_{s}\bb P_{h-1}^{k}(s)\brac{\sum_a\paren{\pi^\star_{\omega_{k}}(a|s)-\pi^\star_{\omega_{k-1}}(a|s)}P(s^\prime|s,a)}\\
        &+ \sum_s\paren{\bb P_{h-1}^{k}(s)-\bb P_{h-1}^{k-1}(s)}\brac{\sum_a\pi^\star_{\omega_{k-1}}(a|s)P(s^\prime|s,a)}.
    \end{split}
\end{equation}
Taking absolute value on both side, yields
\begin{equation}
    \label{eq:connecting_pi_to_mu_derivation_1}
    \begin{split}
        &\norm{\bb P_{h}^{k}-\bb P_{h}^{k-1}}{1}=\sum_{s^\prime}\abs{\bb P^{k}_h(s^\prime) - \bb P^{k-1}_h(s^\prime)}\\
        \overset{}{\leq} &\sum_{s}\bb P_{h-1}^{k}(s)\sum_a\underbrace{\abs{\pi^\star_{\omega_{k}}(a|s)-\pi^\star_{\omega_{k-1}}(a|s)}}_{\leq c_1\norm{\omega_{k} - \omega_{k-1}}{2}}\sum_{s^\prime}P(s^\prime|s,a)\\
        &+ \sum_s\abs{\bb P_{h-1}^{k}(s)-\bb P_{h-1}^{k-1}(s)}\brac{\sum_{s^\prime}\sum_a\pi^\star_{\omega_{k-1}}(a|s)P(s^\prime|s,a)} \\
        \overset{(i)}{\leq} & c_1\norm{\omega_{k}-\omega_{k-1}}{2}+\norm{\bb P_{h-1}^{k}-\bb P_{h-1}^{k-1}}{1} \leq \dots \leq c_1h\norm{\omega_{k}-\omega_{k-1}}{2} + \norm{\bb P_0^{k} - \bb P_0^{k-1}}{1}\\
        \overset{(ii)}{=} &c_1h\norm{\omega_{k}-\omega_{k-1}}{2} + \norm{\mu_{k} - \mu_{k-1}}{1},
    \end{split}
\end{equation}
where inequality $(i)$ holds by applying \cref{lemma:RewardToOptPolicy} with $c_1 = L_r/\alpha(1-\gamma)$ and equality $(ii)$ holds because the initial distribution of $\bb P_h^{k}$ is $\mu_k$. By the definition of $d_\mu^\pi$, we have
\begin{equation}
    \label{eq:connecting_pi_to_mu_derivation_2}
    d_{\mu_{k}}^{\pi^\star_{\omega_{k}}}(s) -  d^{\pi^\star_{\omega_{k-1}}}_{\mu_{k-1}}(s) \overset{(i)}{=} d_{k}(s) - d_{k-1}(s) = (1-\gamma)\sum_{h=0}^\infty\gamma^h\paren{\bb P_h^{k}(s) - \bb P_h^{k-1}(s)},\;\forall s\in \mc S.
\end{equation}
where in equality $(i)$, we use $d_{k}$ to denote $d_{\mu_k}^{\pi^\star_{\omega_k}}$.
Adding $\ell^1$ norm on both sides of \eqref{eq:connecting_pi_to_mu_derivation_2} and applying \eqref{eq:connecting_pi_to_mu_derivation_1}, yields
\begin{equation}
    \begin{split}
        &\norm{d_{k}-d_{k-1}}{1} \leq(1-\gamma)\sum_{h=0}^\infty\gamma^h\paren{c_1h\norm{\omega_{k}-\omega_{k-1}}{2}+\norm{\mu_{k}-\mu_{k-1}}{1}}\\
        \overset{(i)}{=} &\frac{\gamma c_1 }{1-\gamma}\norm{\omega_{k}-\omega_{k-1}}{2}+\norm{\mu_{k}-\mu_{k-1}}{1} \overset{(ii)}{=}\frac{\gamma c_1 }{1-\gamma}\norm{\omega_{k}-\omega_{k-1}}{2}+ \beta\norm{d_{k-1}-d_{k-2}}{1} ,
    \end{split}
\end{equation}
where equality $(i)$ holds because $\sum_{h=0}^\infty\gamma^h h=\gamma/(1-\gamma)^2$ and equality $(ii)$ holds because of \eqref{eq:mu_update}. Hence, we know that 
\begin{equation}
    \label{eq:connecting_pi_to_mu_derivation_3}
    \begin{split}
        \norm{d_{k} - d_{k-1}}{1} \leq &\frac{\gamma c_1}{1-\gamma}\norm{\omega_{k}-\omega_{k-1}}{2}+\beta\norm{d_{k-1}-d_{k-2}}{1}\\
        \leq & \frac{\gamma c_1}{1-\gamma} \sum_{i=0}^{k-1}\brac{\norm{\omega_{i+1}-\omega_i}{2}\beta^{k-i}}+\beta^{k-1}\norm{d_1-d_0}{1}\\
        \leq & \frac{\gamma c_1}{1-\gamma}\cdot\frac{1}{1-\beta}\Delta_\omega^k + \beta^{k-1}\norm{d_1-d_0}{1}\leq\paren{\frac{\gamma c_1}{(1-\gamma)(1-\beta)}+1}\Delta_\omega^k
    \end{split}
\end{equation}
where $\Delta_\omega^k = \max_{1\leq i\leq k}\norm{\omega_{i}-\omega_{i-1}}{2}$ and the last inequality holds due to the near optimality definition (\cref{def:near-optimal-init}). Therefore, applying \eqref{eq:connecting_pi_to_mu_derivation_3} back to \eqref{eq:connecting_pi_to_mu_derivation_0}, we know that
\begin{equation}
    \begin{split}
        &\;\norm{\frac{d_{\mu_{k}}^{\pi^\star_{\omega_{k}}}}{\mu_{k}}}{\infty} \leq \frac{\norm{d^{\pi^\star_{\omega_{k}}}_{\mu_{k}}-d^{\pi^\star_{\omega_{k-1}}}_{\mu_{k-1}}}{1}}{\min\mu_{k}}+\frac{1}{\beta}\\
        \overset{(i)}{\leq} &\; \frac{1}{\min\mu_k}\paren{\frac{\gamma c_1}{(1-\gamma)(1-\beta)}+1}\Delta_\omega^k+\frac{1}{\beta}=\wt{O}\paren{\frac{L_r}{\alpha(1-\beta)}\Delta^k_{\omega} S},
    \end{split}
\end{equation}
where inequality $(i)$ holds since (1) \cref{lemma:RewardToOptPolicy} implies $c_1 = L_r/\alpha(1-\gamma)$, and we omit the $1/(1-\gamma)^6$ and $\log$ in the $\wt{O}$; (2) $1/\min \mu_k\leq S/(1-\beta)$ according to 
$\mu_{k} =\beta d^{\pi_{\omega_{k-1}}^{\star}}_{\mu_{k-1}} + (1-\beta)\rho$. Note that we can only achieve the final bound $\wt{O}\paren{\frac{L_r}{\alpha(1-\beta)}\Delta^k_{\omega} S}$ by setting $\beta$ as a constant. If we pick an arbitrarily small $\beta$, then the $1/\beta$ term will dominate the complexity and we will not have the final bound of $\wt{ O}\paren{\frac{L_r}{\alpha(1-\beta)}\Delta^k_{\omega} S}$. 
\end{proof}

\subsection{Complexity of Vanilla Stochastic PG}
\label{subsec:appendix-main-result-spg}

\begin{theorem}[Complexity of Stochastic PG (Theorem 5.1 of \cite{ding2021beyond})]
\label{thm:spg-complexity}
Consider an arbitrary tolerance level $\delta>0$ and a small enough tolerance level $\eps >0$. For every initial point $\theta_0$, if $\theta_{T+1}$ is generated by SPG (\cref{algo:SPG}) with 
\begin{equation}
    \begin{split}
        &T_1 \geq \paren{\frac{6 D(\theta_0)}{\delta \eps_0}}^{\frac{8L}{C_\delta^0\ln2}}, \; T_2 \geq \paren{\frac{\eps_0}{6\delta\eps}-1}t_0,\; T=T_1+T_2,\\
        & B_1\geq \max\Brac{\frac{30\sigma^2}{C_\delta^0\eps_0\delta},\frac{6\sigma T_1\log T_1}{\bar{\Delta}L}},\;B_2\geq \frac{\sigma^2\ln(T_2+t_0)}{6 C_\zeta \delta \eps},\\
        &\eta_t=\eta\leq \min\Brac{\frac{\log T_1}{T_1 L},\frac{8}{C_\delta^0},\frac{1}{2L}}\;\forall 1\leq t\leq T_1,\;\eta_t=\frac{1}{t-T_1+t_0} \;\forall t>T_1,
    \end{split}
\end{equation}
where 
\begin{equation}
    \label{eq:complexity_constants}
    \begin{split}
        & D(\theta_t) = V^{\pi^\star}(\rho)-V^{\pi_{\theta_t}}(\rho),\;\eps_0=\min\Brac{\paren{\frac{\alpha \min_{s\in \mc S}\rho(s)}{6\ln 2}}^2\brac{\zeta\exp\paren{-\frac{1}{(1-\gamma)\alpha}}}^4,1},\\ &t_0\geq \sqrt{\frac{3\sigma^2}{2\delta \eps_0}},\;C_\delta^0 = \frac{2\alpha}{S}\norm{\frac{d_\rho^{\pi^\star}}{\rho}}{\infty}^{-1}\min_{s\in \mc S}\rho(s)\min_{\theta\in \mc G_\delta^0}\min_{s,a}\pi_\theta(a|s)^2,\\
        &C_\zeta = \frac{2\alpha}{S}\norm{\frac{d_\rho^{\pi^\star}}{\rho}}{\infty}^{-1}\min_{s\in \mc S}\rho(s)(1-\zeta)^2\min_{s,a}\pi^\star(a|s)^2,\\
        &\mc G_{\delta}^0:=\Brac{\theta\in \bb R^{S\times A}:\min_{\theta^\star\in \Theta^\star}\norm{\theta-\theta^\star }{2}\leq(1+1/\delta)\bar{\Delta}},\;\bar{\Delta} = \norm{\log c_{\bar{\theta}_1,\eta}-\log \pi^\star}{2},\\
        &c_{\bar{\theta}_1,\eta}=\inf_{t\geq1}\min_{s,a}\pi_{\theta_t}(a|s),\;\sigma^2 = \frac{8}{(1-\gamma)^2}\paren{\frac{1+(\alpha\log A)^2}{(1-\gamma^{1/2})^2}},\;L=\frac{8+\alpha(4+8\log A)}{(1-\gamma)^3},
    \end{split}
\end{equation}
then we have $\bb P(D(\theta_{T+1})\leq \eps)\geq 1-\delta$.\footnote{Note that the $\zeta$ here is an optimization constant that appears in $\eps_0$ and $C_\zeta$.}
\end{theorem}
\begin{corollary}[Iteration Complexity and Sample Complexity for $\eps$-Optimal Policies]
\label{coro:spg_complexity}
Suppose we set the tolerance level $\eps,\delta = O(S^{-1})$, the iteration complexity and sample complexity of obtaining an $\eps$-optimal policy using stochastic softmax policy gradient (\cref{algo:SPG}) in phase 1 and phase 2 satisfies:
\begin{itemize}
    \item Phase 1: $T_1 = \wt{\Omega}\paren{S^{2S^3}}$, $B_1 = \wt{\Omega}\paren{S^{2S^3}}$,
    \item Phase 2: $T_2 = \wt{\Omega}\paren{S}$, $B_2 = \wt{\Omega}\paren{S^5}$,
\end{itemize}
with probability at least $1-\delta$.
\end{corollary}
\begin{proof}
\label{proof:spg_complexity}
We first check the dependency of \eqref{eq:complexity_constants} on $S$. Notice that
\begin{itemize}
    \item $\eps_0$:
    \begin{equation}
        \label{eq:eps-0}
        \frac{1}{\eps_0}=\max\Brac{\paren{\frac{6\ln 2}{\alpha \min_{s\in \mc S}\rho(s)}}^2\brac{\zeta\exp\paren{-\frac{1}{(1-\gamma)\alpha}}}^{-4},1}=\wt{\Omega}(S^2);
    \end{equation} 
    \item $t_0$:
    \begin{equation}
        t_0\geq \sqrt{\frac{3\sigma^2}{2\delta \eps_0}}=\wt{\Omega}(S);
    \end{equation}
    \item $C_\delta^0$:
    \begin{equation}
        \frac{1}{C_\delta^0} = \frac{S}{2\alpha}\norm{\frac{d_\rho^{\pi^\star}}{\rho}}{\infty}\max_{s\in \mc S}\rho(s)^{-1}\frac{1}{\min_{\theta\in \mc G_\delta^0}\min_{s,a}\pi_\theta(a|s)^2}=\wt{\Omega}(S^3);
    \end{equation}
    \item $C_\zeta$: 
    \begin{equation}
        \label{eq:C-zeta}
        \frac{1}{C_\zeta}=\frac{S}{2\alpha}\norm{\frac{d_\rho^{\pi^\star}}{\rho}}{\infty}\max_{s\in \mc S}\rho(s)^{-1}(1-\zeta)^{-2}\max_{s,a}\pi^\star(a|s)^{-2}=\wt{\Omega}(S^3).
    \end{equation}
\end{itemize}
Hence, the complexities in phase 1 scales at
\begin{equation}
        T_1 \geq \paren{\frac{6 D(\theta_0)}{\delta \eps_0}}^{\frac{8L}{C_\delta^0\ln2}} =\wt{\Omega}\paren{S^{2S^3}},\;B_1 \geq \max\Brac{\frac{30\sigma^2}{C_\delta^0\eps_0\delta},\frac{6\sigma T_1\log T_1}{\bar{\Delta}L}}=\wt{\Omega}\paren{S^{2S^3}}.
\end{equation}
To enforce a positive $T_2$, the tolerance level $\eps,\delta$ should satisfy $\frac{\eps_0}{6\delta\eps}\geq1$, which implies $\frac{1}{\delta\eps}=\Omega(S^2)$. Hence, assuming $\frac{\eps_0}{\delta\eps}=o(S)$, $\eps,\delta = O(S^{-1})$, then the complexities in phase 2 scales at
\begin{equation}
        T_2 \geq \paren{\frac{\eps_0}{6\delta\eps}-1}t_0 =\wt{\Omega}\paren{S},\;B_2 \geq \frac{\sigma^2\ln(T_2+t_0)}{6 C_\zeta \delta \eps}=\wt{\Omega}\paren{S^{5}}.
\end{equation}
\end{proof}

\subsection{Complexity of Learning the Next Context}
\label{subsec:appendix-learning-next-context}
\begin{theorem}[\cref{thm:softmax_spg_convergence_context-Main}: Complexity of Learning the Next Context]
\label{thm:softmax_spg_convergence_context}
Consider the context-based stochastic softmax policy gradient (line 7 of \cref{algo:PGSPG_context}),
suppose \cref{assump:LipReward} and \cref{assump:CloseEnoughContext} hold, then the iteration number of obtaining an $\eps$-optimal policy for $\omega_{k}$ from $\theta^\star_{\omega_k-1}$ is $\wt{\Omega}\paren{S}$
and the per iteration sample complexity is $\wt{\Omega}\paren{\frac{L_r}{\alpha(1-\beta)}S^3}$.
\end{theorem}
We first introduce the following lemma to aid the proof of \cref{thm:softmax_spg_convergence_context}.
\begin{lemma}[Bounded Optimal Values Between two Adjacent Contexts]
\label{lemma:value_diff_adj_ctx}
Under the same conditions as \cref{thm:softmax_spg_convergence_context}, we have
\begin{equation}
    V^{\pi^\star_{\omega_{k}}}_{\omega_{k}}(\rho) - V^{\pi^\star_{\omega_{k-1}}}_{\omega_{k}}(\rho)\leq \frac{2L_r\norm{\omega_k-\omega_{k-1}}{2}}{(1-\gamma)^2}.
\end{equation}
\end{lemma}
\begin{proof}
\label{proof:value_diff_adj_ctx}
Let $V^{\pi}_\omega$ denote the value function of policy $\pi$ with reward function $r_\omega$. From \eqref{eq:sub_opt_lemma} of \cref{lemma:Soft_subopt}, we know that for any initial distribution $\rho$, we have
\begin{equation}
    \label{eq:softmax_spg_convergence_context_derivation0}
    V^{\pi^\star_{\omega_{k}}}_{\omega_{k}}(\rho) - V^{\pi^\star_{\omega_{k-1}}}_{\omega_{k}}(\rho) = \frac{1}{1-\gamma} \sum_s\brac{d_\rho^{\pi^\star_{\omega_{k-1}}}(s)\cdot\alpha\cdot D_{\mr{KL}}\paren{\pi^\star_{\omega_{k-1}}(\cdot|s)||\pi^\star_{\omega_{k}}(\cdot|s)}}.
\end{equation}
From \eqref{eq:MaxEntRLOptPolicy} of \cref{lemma:OptPolicyMaxEnt}, we know that
\begin{equation}
    \begin{split}
        \pi^\star_{\omega_{k-1}}(a|s) &= \brac{\mr{softmax}(Q^{\pi^\star_{\omega_{k-1}}}(\cdot,s)/\alpha)}_a := \frac{\exp\brac{Q^{\pi_{\omega_{k-1}}^\star}(s,a)/\alpha}}{\sum_{a^\prime}\exp\brac{Q^{\pi_{\omega_{k-1}}^\star}(s,a^\prime)/\alpha}}\\
        \pi^\star_{\omega_k}(a|s) &= \brac{\mr{softmax}(Q^{\pi^\star_{\omega}}(\cdot,s)/\alpha)}_a :=  \frac{\exp\brac{Q^{\pi_{\omega_k}^\star}(s,a)/\alpha}}{\sum_{a^\prime}\exp\brac{Q^{\pi_{\omega_k}^\star}(s,a^\prime)/\alpha}},
    \end{split}
\end{equation}
hence, we have
\begin{equation}
    \label{eq:softmax_spg_convergence_context_derivation1}
    \begin{split}
        &D_\mr{KL}\paren{\pi^\star_{\omega_{k-1}}(\cdot|s)||\pi^\star_{\omega_{k}}(\cdot|s)}\\
        = &\sum_a \pi_{\omega_{k-1}}^\star(a|s)\Brac{\log \paren{\brac{\mr{softmax}(Q^{\pi^\star_{\omega_{k-1}}}(a,s)/\alpha)}_a} - \log \paren{\brac{\mr{softmax}(Q^{\pi^\star_{\omega_{k}}}(a,s)/\alpha)}_a}}.
    \end{split}
\end{equation}
Let $f(\mb x)$ denote the log soft max function for an input vector $\mb x=[x_1,x_2,\dots,x_{A}]^\top$ such that $x_i\geq 0$, then for a small perturbation $\mb \Delta\in \bb R^{A}$, the intermediate value theorem implies
\begin{equation}
    \abs{\brac{f(\mb x+\mb \Delta)}_i- \brac{f(\mb x)}_i} = \abs{\mb \Delta^\top \nabla_{\mb z} \brac{f(\mb z)}_i}, 
\end{equation}
for some vector $\mb z$ on the segment $[\mb x, \mb x+\mb \Delta]$. Now consider the Jacobian of the log softmax function $\partial[\nabla_{\mb z}f(\mb z)]_i/\partial z_j$:
\begin{equation}
    \frac{\partial[\nabla_{\mb z}f(\mb z)]_i}{\partial z_j} = \begin{cases}
            1-p_i(\mb z)\in(0,1)  &  \text{ if }i=j,\\
            -p_j(\mb z)\in(-1,0) & \text{ otherwise},
    \end{cases}
\end{equation}
where $p_i(\mb z) = \exp(z_i)/\sum_{k=1}^A\exp(z_k)$. hence, we know that
\begin{equation}
    \label{eq:softmax_spg_convergence_context_derivation2}
    \begin{split}
        &\abs{\brac{f(\mb x+\mb \Delta)}_i- \brac{f(\mb x)}_i} = \abs{\mb \Delta^\top \nabla_{\mb z} \brac{f(\mb z)}_i}\leq \norm{\mb \Delta}{\infty}\sum_{k=1}^A\abs{\frac{\partial [f(\mb z)]_i}{\partial z_k}}\\
        =&\norm{\mb\Delta}{\infty}\paren{1-p_i(\mb z)+\sum_{j\neq i}p_j(\mb z)}\leq 2\norm{\mb \Delta}{\infty}.
    \end{split}
\end{equation}
Now let
\begin{equation}
    \begin{split}
        \mb x &= \frac{1}{\alpha}[Q^{\pi_{{\omega_{k-1}}}^\star}(s,a_1),Q^{\pi_{{\omega_{k-1}}}^\star}(s,a_2),\dots,Q^{\pi_{{\omega_{k-1}}}^\star}(s,a_A)],\\
        \mb x +\mb \Delta &= \frac{1}{\alpha}[Q^{\pi_{\omega_k}^\star}(s,a_1),Q^{\pi_{\omega_k}^\star}(s,a_2),\dots,Q^{\pi_{\omega_k}^\star}(s,a_{A})],
    \end{split}
\end{equation}
\eqref{eq:Q_diff_infty} from \cref{lemma:RewardToOptPolicy} implies that 
\begin{equation}
    \label{eq:softmax_spg_convergence_context_derivation3}
    \frac{1}{\alpha}\norm{Q^{\pi^\star_{\omega_k}}-Q^{\pi^\star_{\omega_{k-1}}}}{\infty}\leq\frac{L_r\norm{\omega_k-\omega_{k-1}}{2}}{\alpha(1-\gamma)},
\end{equation}
substituting \eqref{eq:softmax_spg_convergence_context_derivation3} and \eqref{eq:softmax_spg_convergence_context_derivation2} into \eqref{eq:softmax_spg_convergence_context_derivation1}, yields 
\begin{equation}
    \label{eq:softmax_spg_convergence_context_derivation4}
    D_\mr{KL}\paren{\pi^\star_{\omega_{k-1}}(\cdot|s)||\pi^\star_{\omega_{k}}(\cdot|s)} \leq \sum_a 2\pi_{\omega_{k-1}}^\star(a|s)\norm{\mb \Delta}{\infty}\leq 2\norm{\mb \Delta}{\infty}\leq \frac{2L_r\norm{\omega_k-\omega_{k-1}}{2}}{\alpha(1-\gamma)}.
\end{equation}
Combine \eqref{eq:softmax_spg_convergence_context_derivation4} with \eqref{eq:softmax_spg_convergence_context_derivation0}, we have
\begin{equation}
    \begin{split}
        V^{\pi^\star_{\omega_{k}}}_{\omega_{k}}(\rho) - V^{\pi^\star_{\omega_{k-1}}}_{\omega_{k}}(\rho) =& \frac{1}{1-\gamma} \sum_s\brac{d_\rho^{\pi^\star_{\omega_{k-1}}}(s)\cdot\alpha\cdot D_{\mr{KL}}\paren{\pi^\star_{\omega_{k-1}}(\cdot|s)||\pi^\star_{\omega_{k}}(\cdot|s)}}\\
        \leq& \frac{2L_r\norm{\omega_{k}-\omega_{k-1}}{2}}{(1-\gamma)^2},
    \end{split}
\end{equation}
which completes the proof.
\end{proof}
Now we are ready to proceed to the proof of \cref{thm:softmax_spg_convergence_context}.

\begin{proof}
\label{proof:softmax_spg_convergence_context}
From \eqref{eq:complexity_constants} we know that
\begin{equation}
    \eps_0 = \min\Brac{\paren{\frac{\alpha \min_{s\in \mc S}\rho(s)}{6\ln 2}}^2\brac{\zeta\exp\paren{-\frac{1}{(1-\gamma)\alpha}}}^4,1}=O\paren{\frac{1}{S^2}}.
\end{equation}
And from Section 6.2 of \cite{ding2021beyond}, we can directly enter phase 2 of the stochastic PG when 
\begin{equation}
    V^{\pi^\star_{\omega_{k}}}_{\omega_{k}}(\rho) - V^{\pi^\star_{\omega_{k-1}}}_{\omega_{k}}(\rho)\leq \eps_0.
\end{equation}
Hence, when $\Delta_\omega^k = \max_{1\leq i\leq k}\norm{\omega_{i}-\omega_{i-1}}{2} = O(1/S^2)$, we have
\begin{equation}
    \label{eq:final_bound/2}
    V^{\pi^\star_{\omega_{k}}}_{\omega_{k}}(\rho) - V^{\pi^\star_{\omega_{k-1}}}_{\omega_{k}}(\rho)\leq \frac{2L_r\Delta_\omega}{(1-\gamma)^2}\leq\frac{\eps_0}{2},
\end{equation}
which implies we can directly enter phase 2 and enjoys the faster iteration complexity of $T_2 = \Omega\paren{S}$ (by choosing $\delta = O(S^{-1})$) and the smaller batch size of
\begin{equation}
    \label{eq:beta-discussion}
    \begin{split}
        B_2 \geq \frac{\sigma^2\ln(T_2+t_0)}{6 C_\zeta \delta \eps} \overset{(i)}{=} \wt{\Omega}\paren{\frac{L_r}{\alpha(1-\beta)}\Delta^k_{\omega} S^5} \overset{(ii)}{=} \wt{\Omega}\paren{\frac{L_r}{\alpha(1-\beta)}S^3},
    \end{split}
\end{equation}
where equation $(i)$ holds by applying \cref{lemma:mismatch_ratio_bound} to \eqref{eq:C-zeta}:
\begin{equation*}
    \frac{\sigma^2\ln(T_2+t_0)}{6 C_\zeta \delta \eps} = \wt{\Omega}\paren{S^4\cdot\norm{d_{\mu_{k}}^{\pi^\star_{\omega_{k}}}/\mu_{k}}{\infty}}=\wt{\Omega}\paren{\frac{L_r}{\alpha(1-\beta)}\Delta^k_{\omega} S^5},
\end{equation*}
and equality $(ii)$ holds by the assumption that $\Delta^k_{\omega}= O(S^{-2})$ and we omit the $\log$ term and components not related to $S$ in $\wt{\Omega}$.
\end{proof}
\subsection{Total Complexity of \ours{}}
\label{subsec:appendix-learning-final-context}

\begin{theorem}[\cref{thm:total_complexity-Main}: Total Complexity of Learning the Target Context]
\label{thm:total_complexity}
Suppose \cref{assump:LipReward} and \cref{assump:CloseEnoughContext} hold, and $\theta^{(0)}_0$ is an near-optimal initialization, then the total number of iteration of learning $\pi^\star_{\omega_K}$ using \cref{algo:PGSPG_context} is $\Omega(KS)$ and the per iteration is $\wt{\Omega}\paren{S^{3}}$, with high probability. 
\end{theorem}
\begin{proof}
\label{proof:total_complexity}
From \cref{lemma:value_diff_adj_ctx}, we know that 
\begin{equation}
    V^{\pi^\star_{\omega_{k}}}_{\omega_{k}}(\rho) - V^{\pi^\star_{\omega_{k-1}}}_{\omega_{k}}(\rho)\leq \frac{2L_r\norm{\omega_k-\omega_{k-1}}{2}}{(1-\gamma)^2}.
\end{equation}
Suppose for each context $\omega_k$, we initialize the parameters of the policy as $\theta_0^{(k)}=\theta^\star_{\omega_{k-1}}$, and let $\theta_t^{(k)}$ denote the parameters at the $t^\text{th}$ iteration of SPG. We will use induction to show that when $t=\wt{\Omega}(S)$, $\forall k\in[K]$, we have 
\begin{equation}
    \label{eq:induction_result}
    V^{\pi^\star_{\omega_{k}}}_{\omega_{k}}(\rho) - V^{\pi_{\theta_t^{(k-1)}}}_{\omega_{k}}(\rho)<\eps_0,
\end{equation}
this implies that for any context $\omega_{k},k\in [K]$, we can always find a good initialization by setting $\theta_0^{(k)}=\theta_t^{(k-1)}$ from learning $\pi^\star_{\omega_{k-1}}$ using SPG after $t=\Omega(S)$ iteration. This result guarantees that every initialization $\theta_0^{(k)}$ for learning the optimal contextual policy $\pi_{\omega_k}^\star$ will directly start from the efficient phase 2.\\

{\bf Induction: $k=0$. } When $k=0$, \cref{assump:CloseEnoughContext} and the near-optimal initialization (\cref{def:near-optimal-init}) of $\theta_0^{(0)}$ implies that \begin{equation}
    V^{\pi^\star_{\omega_{0}}}_{\omega_{0}}(\rho) - V^{\pi_{\theta_0^{(0)}}}_{\omega_{0}}(\rho)<\eps_0.
\end{equation}
This result implies that a near-optimal initialization allows the initialization to directly start from phase 2 of SPG.

{\bf Induction: from $k-1$ to $k$. }
Suppose the result in \eqref{eq:induction_result} holds for $k-1$, then we know that 
\begin{equation}
    V^{\pi^\star_{\omega_{k-1}}}_{\omega_{k-1}}(\rho) - V^{\pi_{\theta_0^{(k-1)}}}_{\omega_{k-1}}(\rho) = V^{\pi^\star_{\omega_{k-1}}}_{\omega_{k-1}}(\rho) - V^{\pi_{\theta_t^{(k-2)}}}_{\omega_{k-1}}(\rho)<\eps_0.    
\end{equation}
Select $\eps$ such that $\eps\leq\eps_0/2$. \cref{thm:softmax_spg_convergence_context} suggests that when $t^\prime=\tilde{\Omega}(S)$, with high probability, we have
\begin{equation}
    \label{eq:induction_derivation_1}
    V^{\pi^\star_{\omega_{k}}}_{\omega_{k}}(\rho) - V^{\pi_{\theta_{t^\prime}^{(k-1)}}}_{\omega_{k}}(\rho)<\eps\leq \frac{\eps_0}{2}.
\end{equation}
Hence, if we initialize $\theta_0^{(k)} = \theta_{t}^{(k-1)}$, with high probability when $t^\prime=\wt{\Omega}(S)$, we have
\begin{equation}
    \begin{split}
        &V^{\pi^\star_{\omega_{k}}}_{\omega_{k}}(\rho) - V^{\pi_{\theta_{t^\prime}^{(k-1)}}}_{\omega_{k}}(\rho) = V^{\pi^\star_{\omega_{k}}}_{\omega_{k}}(\rho) - V^{\pi^\star_{\omega_{k-1}}}_{\omega_{k}}(\rho) + V^{\pi^\star_{\omega_{k-1}}}_{\omega_{k}}(\rho) - V^{\pi_{\theta_{t^\prime}^{(k-1)}}}_{\omega_{k}}(\rho)\\
        \overset{(i)}{\leq} & \frac{\eps_0}{2} + V^{\pi^\star_{\omega_{k}}}_{\omega_{k}}(\rho) - V^{\pi_{\theta_{t^\prime}^{(k-1)}}}_{\omega_{k}}(\rho)\overset{(ii)}{<}\eps_0,
    \end{split}
\end{equation}
where inequality $(i)$ holds by equation~\eqref{eq:final_bound/2} in \cref{thm:softmax_spg_convergence_context}, inequality $(ii)$ holds because of the induction assumption in~\eqref{eq:induction_derivation_1}. 

Therefore, we have shown \eqref{eq:induction_result} holds for $t=\wt{\Omega}(S)$, $\forall k\in[K]$. Since we have $K$ contexts in total, we know that \cref{algo:PGSPG_context} can enforce a good initialization $\theta_0^{(k)}$ that directly starts from phase 2 for learning all $\pi^\star_{\omega_k}$, and for each $k\in [K]$, the iteration complexity is $\wt{\Omega}(S)$. Hence the total iteration complexity of obtaining an $\eps$-optimal policy for the final context $\omega_{K}$ is $\wt{\Omega}\paren{KS}$, with per iteration sample complexity of $\wt{\Omega}\paren{S^3}$.
\end{proof}
\section{Key Lemmas}
\subsection{Optimal Policy of Maximum Entropy RL \cite{nachum2017bridging}}
\begin{lemma}
\label{lemma:OptPolicyMaxEnt}
The optimal policy $\pi^\star$ that maximizes the $\alpha$-MaxEnt RL objective \eqref{eq:MaxEntRLV} with penalty term $\alpha$ satisfies:
\begin{align}
    \label{eq:MaxEntRLOptPolicy}
    \pi^\star(a|s) = \exp\left[\paren{Q^{\pi^\star}(s, a) - V^{\pi^\star}(s)}/ \alpha\right] = \frac{\exp \paren{Q^{\pi^\star}(s,a)/\alpha}}{\sum_a\exp\paren{Q^{\pi^\star}(s,a)/\alpha}}
\end{align}
for all $h \in \bb N$, where 
\begin{equation}
    \label{eq:MaxEntRLOptV}
    \begin{split}
        Q^{\pi^\star}(s, a) &:= r(s, a)  + \gamma \mathbb{E}_{s' \sim P(s' | s, a)} V(s') \\
        V^{\pi^\star}(s) &:= \alpha\log \left(\sum_a \exp\left(Q^{\pi^\star}(s, a) / \alpha\right)\right).
    \end{split}
\end{equation}
\end{lemma}
\begin{proof}
Similar proof appears in~\citep{nachum2017bridging}, we provide the proof for completeness. At the optimal policy $\pi_\theta = \pi^\star$, take the gradient of \eqref{eq:MaxEntRLV} w.r.t. $p\in \Delta(\mc A)$ and set it to 0, we have
\begin{equation}
    \frac{\partial}{\partial p(a)}\brac{ \sum_{a\in \mc A}p(a)\paren{Q^{\pi^\star}(s,a) - \alpha\ln p(a)}} = Q^{\pi^\star}(s,a)-\alpha\ln p(a)-\alpha=0,
\end{equation}
which implies 
\begin{equation}
    p(a) = \exp\paren{\frac{Q^{\pi^\star}(s,a)}{\alpha}-1}\propto \exp\paren{\frac{Q^{\pi^\star}(s,a)}{\alpha}}.
\end{equation}
Hence, we conclude that $\pi^\star(a|s)\propto \exp(Q^\star(s,a)/\alpha)$.
\end{proof}
\subsection{Bounding the Difference between Optimal Policies}
\begin{lemma}
\label{lemma:RewardToOptPolicy}
Suppose \cref{assump:LipReward} holds, let $\pi^\star_\omega(a|s),\pi^\star_{\omega^\prime}(a|s)$ denote the optimal policy for $\alpha$-MaxEnt RL \eqref{eq:MaxEntRLOptPolicy}, then $\forall (s,a)\in \mc S\times\mc A$, the optimal policies of $\alpha$-MaxEnt RL under context $\omega,\omega^\prime$ satisfy:
\begin{equation}
    \label{eq:OmegaToPi}
    \abs{\pi^\star_\omega(a|s)-\pi^\star_{\omega^\prime}(a|s)} \leq \frac{L_r\norm{\omega-\omega^\prime}{2}}{\alpha(1-\gamma)}.
\end{equation}
\end{lemma}
\begin{proof}
\label{proof:RewardToOptPolicy}
From \cref{lemma:ContractionSoftVI}, we know that the soft value iteration 
\begin{equation}
    \mc T Q(s,a) = r(s,a) +\gamma\alpha \bb E_{s^\prime}\brac{\log \sum_{a^\prime}\exp Q(s^\prime,a^\prime)/\alpha}
\end{equation}
is a contraction. Let $Q^t_\omega,Q^t_{\omega^\prime}$ denote the Q functions at the $t^\text{th}$ value iteration under context $\omega,\omega^\prime$ respectively, we know $Q^\infty_\omega = Q^{\pi^\star}_\omega$ and $Q^\infty_{\omega^\prime} = Q^{\pi^\star}_{{\omega^\prime}}$. Let $\eps_t = \norm{Q_\omega^t-Q_{\omega^\prime}^t}{\infty}$, then we have
\begin{equation}
    \label{eq:Qstar_diff_derivation1}
    \begin{split}
        &\eps_{t+1} = \norm{Q^{t+1}_\omega -  Q^{t+1}_{\omega^\prime}}{\infty}\\
        = &\norm{r_\omega(s,a) - r_{\omega^\prime}(s,a) +\gamma\alpha \bb E_{s^\prime}\brac{\log \sum_{a^\prime}\exp \frac{Q^t_{\omega^\prime}(s^\prime,a^\prime)}{\alpha}} - \gamma\alpha \bb E_{s^\prime}\brac{\log \sum_{a^\prime}\exp \frac{Q^t_{\omega^\prime}(s^\prime,a^\prime)}{\alpha}}}{\infty}\\
        \leq &\norm{r_\omega - r_{\omega^\prime}}{\infty} + \gamma \alpha\norm{\bb E_{s^\prime}\log\sum_{s^\prime}\exp Q_\omega^t(s^\prime,a^\prime)/\alpha-\bb E_{s^\prime}\log\sum_{s^\prime}\exp Q_{\omega^\prime}^t(s^\prime,a^\prime)/\alpha}{\infty}\\
        \leq& \norm{r_\omega-r_{\omega^\prime}}{\infty} + \gamma \norm{Q_\omega^t-Q_{\omega^\prime}^t}{\infty} = \norm{r_\omega - r_{\omega^\prime}}{\infty} + \gamma \eps_t,
    \end{split}
\end{equation}
where the last inequality holds because $f(\mb x)=\log\sum_{i=1}^n\exp(x_i)$ is a contraction. From \eqref{eq:Qstar_diff_derivation1}, we have
\begin{equation}
    \eps_{t+1} \leq \norm{r_\omega-r_{\omega^\prime}}{\infty} + \gamma \eps_t \leq (1+\gamma)\norm{r_\omega-r_{\omega^\prime}}{\infty} + \gamma^2 \eps_{t-1} \leq \dots \leq \norm{r_\omega-r_{\omega^\prime}}{\infty} \sum_{i=0}^t\gamma^i +\gamma^t\eps_1,
\end{equation}
which implies 
\begin{equation}
    \norm{Q_\omega^{\pi^\star} - Q_{\omega^\prime}^{\pi^\star}}{\infty} = \eps_\infty \leq \frac{\norm{r_\omega-r_{\omega^\prime}}{\infty}}{1-\gamma}\leq \frac{L_r\norm{\omega-\omega^\prime}{2}}{1-\gamma},
\end{equation}
where the last inequality holds by \cref{assump:LipReward}. Hence, we have
\begin{equation}
    \label{eq:Q_diff}
    \frac{1}{\alpha}\abs{Q^{\pi_\omega^\star}(s,a) - Q^{\pi_{\omega^\prime}^\star}(s,a)}\leq\frac{L_r\norm{\omega-\omega^\prime}{2}}{\alpha(1-\gamma)},\; \forall s,a\in \mc S\times \mc A
\end{equation}
which implies
\begin{equation}
    \label{eq:Q_diff_infty}
   \frac{1}{\alpha}\norm{Q^{\pi^\star_\omega}-Q^{\pi^\star_{\omega^\prime}}}{\infty}\leq\frac{L_r\norm{\omega-\omega^\prime}{2}}{\alpha(1-\gamma)}.
\end{equation}
Next, let $\pi^\star_{\omega},\pi^\star_{\omega^\prime}$ denote the maximum entropy policy RL under context $\omega, \omega^\prime$ respectively. Then for a fixed state action pair $(s,a)\in \mc S\times \mc A$, we have 
\begin{equation}
    \begin{split}
        \pi^\star_\omega(a|s) &= \brac{\mr{softmax}(Q^{\pi^\star_{\omega}}(\cdot,s)/\alpha)}_a :=  \frac{\exp\brac{Q^{\pi_\omega^\star}(s,a)/\alpha}}{\sum_{a^\prime}\exp\brac{Q^{\pi_\omega^\star}(s,a^\prime)/\alpha}},\\
        \pi^\star_{\omega^\prime}(a|s) &= \brac{\mr{softmax}(Q^{\pi^\star_{\omega^\prime}}(\cdot,s)/\alpha)}_a := \frac{\exp\brac{Q^{\pi_{\omega^\prime}^\star}(s,a)/\alpha}}{\sum_{a^\prime}\exp\brac{Q^{\pi_{\omega^\prime}^\star}(s,a^\prime)/\alpha}},
    \end{split}
\end{equation}
where $Q^{\pi^\star_{\omega}}(\cdot,s),Q^{\pi^\star_{\omega^\prime}}(\cdot,s)\in \bb R^{A}$, and we want to bound $\abs{\pi_\omega^\star(a|s) - \pi_{\omega^\prime}^\star(a|s)}
$. Next we will use \eqref{eq:Q_diff_infty} to bound $\abs{\pi_{\omega}^\star(a|s) - \pi^\star_{\omega^\prime}(a|s)}$, where the last inequality holds by \eqref{eq:Q_diff}. Let $f(\mb x)$ denote the softmax function for an input vector $\mb x=[x_1,x_2,\dots,x_{A}]^\top$ such that $x_i\geq 0$, then for a small perturbation $\mb \Delta\in \bb R^{A}$, the intermediate value theorem implies
\begin{equation}
    \abs{\brac{f(\mb x+\mb \Delta)}_i- \brac{f(\mb x)}_i} = \abs{\mb \Delta^\top \nabla_{\mb x} \brac{f(\mb z)}_i}, 
\end{equation}
for some vector $\mb z$ on the segment $[\mb x, \mb x+\mb \Delta]$. Hence 
\begin{equation}
    \label{eq:f_FirstOrderBound}
    \begin{split}
        &\abs{\brac{f(\mb x+\mb \Delta)}_i- \brac{f(\mb x)}_i}= \abs{\mb \Delta^\top\brac{\nabla_{\mb x}f(\mb z)}_i}\leq \norm{\mb \Delta}{\infty}\sum_{k=1}^A\abs{\frac{\partial [f(\mb z)]_i}{\partial z_k}}\\
        \leq&\norm{\mb \Delta}{\infty}\paren{p_i(\mb z)(1-p_i(\mb z))+\sum_{j\neq i}p_i(\mb z)p_j(\mb z)}<\norm{\mb\Delta}{\infty}\paren{p_i(\mb z)+\sum_{j\neq i}p_j(\mb z)}=\norm{\mb\Delta}{\infty},
    \end{split}
\end{equation}
where the Jacobian of the softmax function $\partial \brac{\nabla_{\mb x}f(\mb z)}_i/\partial z_j$ satisfies:
\begin{equation}
    \frac{\partial \brac{\nabla_{\mb x}f(\mb z)}_i }{\partial z_j }=
    \begin{cases}
        p_i(\mb z)(1-p_i(\mb z))& \text{ if }i=j,\\
        p_i(\mb z)p_j(\mb z)& \text{ otherwise},\\
    \end{cases}
\end{equation}
and $p_i(\mb z) = \exp(z_i)/\sum_{k=1}^A\exp(z_k)$. Now let
\begin{equation}
    \begin{split}
        \mb x &= \frac{1}{\alpha}[Q^{\pi_\omega^\star}(s,a_1),Q^{\pi_\omega^\star}(s,a_2),\dots,Q^{\pi_\omega^\star}(s,a_{A})],\\
        \mb x+\mb \Delta &= \frac{1}{\alpha}[Q^{\pi_{\omega^\prime}^\star}(s,a_1),Q^{\pi_{\omega^\prime}^\star}(s,a_2),\dots,Q^{\pi_{\omega^\prime}^\star}(s,a_A)].
    \end{split}
\end{equation}
We know that $f(\mb x) = \pi_\omega^\star(a|s)$ and $f(\mb x+\mb \Delta) = \pi_{\omega^\prime}^\star(a|s)$. Then \eqref{eq:Q_diff_infty} implies that 
\begin{equation}
    \norm{\mb \Delta}{\infty}\leq\frac{L_r\norm{\omega-\omega^\prime}{2}}{\alpha(1-\gamma)}, 
\end{equation}
substituting this bound on $\norm{\mb \Delta}{\infty}$ into \eqref{eq:f_FirstOrderBound}, we have 
\begin{equation}
    \abs{\pi^\star_\omega(a|s)-\pi^\star_{\omega^\prime}(a|s)} = \abs{f(\mb x) - f(\mb x + \mb \Delta)} \leq \norm{\mb \Delta}{\infty}\leq\frac{L_r\norm{\omega-\omega^\prime}{2}}{\alpha(1-\gamma)},
\end{equation}
which completes the proof.
\end{proof}
\subsection{Soft Sub-Optimality lemma (Lemma 25 \& 26 of \cite{mei2020global})}
\begin{lemma}
\label{lemma:Soft_subopt}
For any policy $\pi$ and any initial distribution $\rho$, the value function $V^\pi(\rho)$ of the $\alpha$-MaxEnt RL \eqref{eq:MaxEntRLOptV} satisfies:
\begin{equation}
    \label{eq:sub_opt_lemma}
    V^{\pi^\star}(\rho) - V^{\pi}(\rho) = \frac{1}{1-\gamma} \sum_s\brac{d_\rho^\pi(s)\cdot\alpha\cdot D_{\mr{KL}}\paren{\pi(\cdot|s)||\pi^\star(\cdot|s)}},
\end{equation}
where $\pi^\star$ is the optimal policy of the $\alpha$-MaxEnt RL \eqref{eq:MaxEntRLV}.
\end{lemma}
\begin{proof}
\label{proof:Soft_subopt}
Similar proof appears in Lemma 25 \& 26 of \cite{mei2020global}, we provide the proof here for completeness.
\paragraph{Soft performance difference.}
We first show a soft performance difference result for the MaxEnt value function (Lemma 25 of \cite{mei2020global}). By the definition of MaxEnt value function and $Q$-function \eqref{eq:MaxEntRLV}, \eqref{eq:MaxEntRLAQ}, $\forall \pi,\pi^\prime$, we have
\begin{equation}
    \begin{split}
        &V^{\pi^\prime}(s) - V^{\pi}(s)\\
        =&\sum_a\pi^\prime(a|s)\cdot\brac{Q^{\pi^\prime}(s,a)-\alpha\log\pi^\prime(a|s)} - \sum_a\pi(a|s)\cdot\brac{Q^{\pi}(s,a)-\alpha\log\pi(a|s)}\\
        =& \sum_a\paren{\pi^\prime(a|s)-\pi(a|s)}\cdot\brac{Q^{\pi^\prime}(a|s)-\alpha\log\pi^\prime(a|s)}\\
        &+ \sum_a\pi(a|s)\cdot\brac{Q^{\pi^\prime}(s,a)-\alpha \log \pi^\prime(a|s)-Q^\pi(s,a)+\alpha\log\pi(a|s)}\\
        =&\sum_a\paren{\pi^\prime(a|s)-\pi(a|s)}\cdot\brac{Q^{\pi^\prime}(a|s)-\alpha\log\pi^\prime(a|s)}+\alpha D_{\mr{KL}}\paren{\pi(\cdot|s)||\pi^\prime(\cdot|s)}\\
        &+\gamma \sum_a\pi(a|s)\sum_{s^\prime}P(s^\prime|s,a)\cdot\brac{V^{\pi^\prime}(s^\prime)-V^\pi(s^\prime)}\\
        =&\frac{1}{1-\gamma}\sum_{s^\prime}d_{s}^\pi(s^\prime)\Bigg[\sum_{a^\prime}(\pi^\prime(a^\prime|s^\prime)-\pi(a^\prime|s^\prime))\brac{Q^{\pi^\prime}(s^\prime,a^\prime)-\alpha \log \pi^\prime(a^\prime|s^\prime)}\\
        &+\alpha D_{\mr {KL}}\paren{\pi(\cdot|s^\prime)||\pi^\prime(\cdot|s^\prime)}\Bigg],
    \end{split}
\end{equation}
where the last equality holds because by the definition of state visitation distribution
\begin{equation}
    d_{s_0}^\pi(s) = (1-\gamma)\sum_{t=0}^\infty \gamma^t \bb P^\pi(s_t=s|s_0),
\end{equation}
taking expectation of $s$ with respect to $s\sim\rho$, yields
\begin{equation}
    \label{eq:soft_perf_diff_lemma}
    \begin{split}
        &V^{\pi^\prime}(\rho) - V^{\pi}(\rho)\\
        =&\frac{1}{1-\gamma}\sum_{s^\prime}d_{\rho}^\pi(s^\prime)\Bigg[\sum_{a^\prime}(\pi^\prime(a^\prime|s^\prime)-\pi(a^\prime|s^\prime))\cdot\brac{Q^{\pi^\prime}(s^\prime,a^\prime)-\alpha \log \pi^\prime(a^\prime|s^\prime)}\\
        &+\alpha D_{\mr {KL}}\paren{\pi(\cdot|s^\prime)||\pi^\prime(\cdot|s^\prime)}\Bigg],
    \end{split}
\end{equation}
and \eqref{eq:soft_perf_diff_lemma} is known as the soft performance difference lemma (Lemma 25 in \cite{mei2020global}). 
\paragraph{Soft sub-optimality.}
Next we will show the soft sub-optimality result. By the definition of the optimal policy of $\alpha$-MaxEnt RL \eqref{eq:MaxEntRLOptPolicy}, we have
\begin{equation}
    \alpha\log\pi^\star(a|s) = Q^{\pi^\star}(s,a) - V^{\pi^\star}(s).
\end{equation}
Substituting $\pi^\star$ into the performance difference lemma \eqref{eq:soft_perf_diff_lemma}, we have 
\begin{equation}
    \begin{split}
        &V^{\pi^\star}(s)-V^\pi(s)\\
        = &\frac{1}{1-\gamma}\sum_{s^\prime}d^\pi_{s}(s^\prime)\cdot\Bigg[\sum_{a^\prime}\paren{\pi^\star(a^\prime|s^\prime)-\pi(a^\prime|s^\prime)}\cdot\underbrace{\brac{Q^{\pi^\star}(s^\prime,a^\prime)-\alpha \log \pi^\star(a^\prime|s^\prime)}}_{ = V^{\pi^\star}(s^\prime)}\\
        &+\alpha D_{\mr {KL}}\paren{\pi(\cdot|s^\prime)||\pi^\star(\cdot|s^\prime)}\Bigg]\\
        =&\frac{1}{1-\gamma}\sum_{s^\prime}d_s^\pi(s^\prime)\cdot\Bigg[\underbrace{\sum_{a^\prime}\paren{\pi^\star(a^\prime|s^\prime)-\pi(a^\prime|s^\prime)}}_{=0}\cdot V^{\pi^\star}(s^\prime)+\alpha D_{\mr {KL}}\paren{\pi(\cdot|s^\prime)||\pi^\star(\cdot|s^\prime)}\Bigg]\\
        =&\frac{1}{1-\gamma}\sum_{s^\prime}\brac{d_s^\pi(s^\prime)\cdot\alpha D_{\mr {KL}}\paren{\pi(\cdot|s^\prime)||\pi^\star(\cdot|s^\prime)}},
    \end{split}
\end{equation}
taking expectation $s\sim \rho$ yields
\begin{equation}
    V^{\pi^\star}(\rho) - V^{\pi}(\rho) = \frac{1}{1-\gamma} \sum_s\brac{d_\rho^\pi(s)\cdot\alpha\cdot D_{\mr{KL}}\paren{\pi(\cdot|s)||\pi^\star(\cdot|s)}},
\end{equation}
which completes the proof.
\end{proof}

\section{Supporting Lemmas}

\subsection{Bellman Consistency Equation of MaxEnt RL}
\begin{lemma}[Contraction of Soft Value Iteration]
\label{lemma:ContractionSoftVI}
From \eqref{eq:MaxEntRLOptV} and \eqref{eq:MaxEntRLAQ}, the soft value iteration operator $\mc T$ defined as
\begin{equation}
    \label{eq:SoftValueIter}
    \mc T Q(s,a) := r(s,a) +\gamma\alpha \bb E_{s^\prime}\brac{\log \sum_{a^\prime}\exp \left( Q(s^\prime,a^\prime)/\alpha\right)}
\end{equation}
is a contraction.
\end{lemma}
\begin{proof}
\label{proof:ContractionSoftVI}
A similar proof appears in \cite{haarnoja2018acquiring}, we provide the proof for completeness. To see \eqref{eq:SoftValueIter} is a contraction, for each $(s,a)\in \mc S\times \mc A$, we have
\begin{equation}
    \begin{split}
        &\mc T Q_1(s,a) = r(s,a) +\gamma\alpha\log \sum_{a^\prime}\exp \paren{\frac{Q_1(s,a)}{\alpha}}\\
        \leq& r(s,a)+\gamma\alpha\log \sum_{a^\prime}\exp \paren{\frac{Q_2(s,a)+\norm{Q_1-Q_2}{\infty}}{\alpha}}\\
        \leq&r(s,a)+\gamma\alpha\log\Brac{\exp\paren{\frac{\norm{Q_1-Q_2}{\infty}}{\alpha}}\sum_{a^\prime}\exp\paren{\frac{Q_2(s,a)}{\alpha}}}\\
        = &\gamma\norm{Q_1-Q_2}{\infty}+r(s,a)+\gamma\alpha\log\sum_{a^\prime}\exp\paren{\frac{Q_2(s,a)}{\alpha}} = \gamma\norm{Q_1-Q_2}{\infty}+ \mc TQ_2(s,a),
    \end{split}
\end{equation}
which implies $\mc T Q_1(s,a) - \mc TQ_2(s,a)\leq \gamma\norm{Q_1-Q_2}{\infty}$. Similarly, we also have $\mc T Q_2(s,a) - \mc TQ_1(s,a)\leq \gamma\norm{Q_1-Q_2}{\infty}$, hence we conclude that 
\begin{equation}
    \abs{Q_1(s,a) - Q_2(s,a)}\leq \gamma\norm{\mc TQ_1-\mc TQ_2}{\infty},\;\forall (s,a)\in \mc S\times \mc A,
\end{equation}
which implies $\norm{Q_1-Q_2}{\infty} \leq \gamma \norm{\mc TQ_1 -\mc TQ_2}{\infty}$. Hence $\mc T$ is a $\gamma$-contraction and the optimal policy $\pi^\star$ of it is unique.
\end{proof}

\subsection{Constant Minimum Policy Probability}
\begin{lemma}[Lemma 16 of \cite{mei2020global}]
\label{lemma:min_policy_constant}
Using the policy gradient method (\cref{algo:PG}) with an initial distribution $\rho$ such that $\rho(s)>0, \forall S$, we have
\begin{equation}
    c:=\inf_{t\geq1}\min_{s,a}\pi_{\theta_t}(a|s)>0   
\end{equation}
is a constant that does not depend on $t$.
\end{lemma}
\begin{remark}[State Space Dependency of constant $c,C_\delta^0$]
For the exact PG case, $c$ in \cref{lemma:min_policy_constant} could also depend on $S$, similarly for the constant $C_\delta^0$ in the stochastic PG case. As pointed out by~\citet{li2021softmax} (Table 1), the constant $c$ (or $C_\delta^0$ in \cref{thm:spg-complexity} of the SPG case) may depend on the structure of the MDP. The ROLLIN technique only improves the mismatch coefficient $\norm{d_\mu^{\pi^\star}/\mu}{\infty}$, instead of the constant $c$ (or $C_\delta^0$). Still, in the exact PG case, if one replaces the constant $c$ with other $S$ dependent function $f(S)$, one still can apply a similar proof technique for \cref{thm:softmax_spg_convergence_context-Main} to show that \ours{} reduces the iteration complexity, and the final iteration complexity bound in \cref{thm:softmax_spg_convergence_context-Main} will include an additional $f(S)$. In addition, omitting the factor $C_\delta^0$, \ours{} can improve the exponential complexity dependency incurred by the stochastic optimization to a polynomial dependency.
\end{remark}

\section{Supporting Algorithms}
\begin{algorithm}[H]
  \caption{PG for $\alpha$-MaxEnt RL (Algorithm 1 in \cite{mei2020global})}\label{algo:PG}
  \begin{algorithmic}[1]
    \State \textbf{Input:} $\rho$, $\theta_0$, $\eta>0$. 
      \For{$t=0,\dots,T$}
        \State $\theta_{t+1}\leftarrow \theta_t+\eta\cdot \frac{\partial V^{\pi_{\theta_t}}(\rho)}{\partial \theta_t}$
      \EndFor
  \end{algorithmic}
\end{algorithm}
\begin{algorithm}[H]
  \caption{Two-Phase SPG for $\alpha$-MaxEnt RL (Algorithm 5.1 in \cite{ding2021beyond})}\label{algo:SPG}
  \begin{algorithmic}[1]
    \State \textbf{Input:} $\rho,\theta_0,\alpha,B_1,B_2,T_1,T,\Brac{\eta_t}_{t=0}^T$ 
    \For{$t=0,1,\dots,T$}
      \If{$t\leq T_1$}
      \State $B=B_1$ \Comment{Phase 1}
      \Else
      \State $B=B_2$ \Comment{Phase 2}
      \EndIf
      \State Run random horizon SPG with $\rho,\alpha,\theta_t,B,t,\eta_t$ \Comment{\cref{algo:random-horizon-spg}} 
    \EndFor
  \end{algorithmic}
\end{algorithm}

\begin{algorithm}[H]
  \caption{Random-horizon SPG for $\alpha$-MaxEnt RL Update (Algorithm 3.2 in \cite{ding2021beyond})}\label{algo:random-horizon-spg}
  \begin{algorithmic}[1]
    \State {\bf Input:} $\rho,\alpha,\theta_0,B,t,\eta_t$ 
      \For{$i=1,2,...,B$}
        \State $s_{H_t}^i,a_{H_t}^i\leftarrow \mr{SamSA}(\rho,\theta_t,\gamma)$ \Comment{\cref{algo:SampleSA}}
        \State $\hat{Q}^{\pi_{\theta_t},i}\leftarrow\mr{EstEntQ}(s_{H_t}^i,a_{H_t}^i,\theta_t,\gamma,\alpha)$ \Comment{\cref{algo:EstimateEntQ}}
      \EndFor
      \State $\theta_{t+1}\leftarrow \theta_t+\frac{\eta_t}{(1-\gamma)B}\sum_{i=1}^B\brac{\nabla_\theta\log \pi_{\theta_t}(a_{H_t}^i|s_{H_t}^i)\paren{\hat{Q}^{\pi_{\theta_t},i}-\alpha\log \pi_{\theta_t}}(a_{H_t}^i|s_{H_t}^i)}$
  \end{algorithmic}
\end{algorithm}
\begin{remark}
Lemma 3.4 in \cite{ding2021beyond} implies that the estimator
\begin{equation}
    \frac{1}{(1-\gamma)}\brac{\nabla_\theta\log \pi_{\theta_t}(a_{H_t}^i|s_{H_t}^i)\paren{\hat{Q}^{\pi_{\theta_t},i}-\alpha\log \pi_{\theta_t}}(a_{H_t}^i|s_{H_t}^i)}
\end{equation}
in line 6 of \cref{algo:SampleSA} is an unbiased estimator of the gradient $\nabla_\theta V^{\pi_\theta}(\rho)$.
\end{remark}
\begin{algorithm}[H]
  \caption{SamSA: Sample $s,a$ for SPG (Algorithm 8.1 in \cite{ding2021beyond})}\label{algo:SampleSA}
  \begin{algorithmic}[1]
    \State {\bf Input:} $\rho,\theta,\gamma$
    \State Draw $H\sim\mr{Geom}(1-\gamma)$ \Comment{$\mr{Geom}(1-\gamma)$ geometric distribution with parameter $1-\gamma$}
    \State Draw $s_0\sim\rho,a_0\sim \pi_\theta(\cdot|s_0)$
    \For{$h=1,2,\dots,H-1$}
      \State Draw $s_{h+1}\sim \bb P(\cdot |s_h,a_h),a_{h+1}\sim\pi_{\theta_t}(\cdot|s_{h+1})$
    \EndFor
  \State {\bf Output: } $s_H,a_H$
  \end{algorithmic}
\end{algorithm}

\begin{algorithm}[H]
  \caption{EstEntQ: Unbiased Estimation of MaxEnt Q (Algorithm 8.2 in \cite{ding2021beyond})}\label{algo:EstimateEntQ}
  \begin{algorithmic}[1]
    \State {\bf Input:} $s,a,\theta,\gamma,\alpha$
    \State Initialize $s_0\leftarrow s,a_0\leftarrow a,\hat{Q}\leftarrow r(s_0,a_0)$
    \State Draw $H\sim\mr{Geom}(1-\gamma)$
    \For{$h=0,1,\dots,H-1$}
      \State $s_{h+1}\sim \bb P(\cdot|s_h,a_h),a_{h+1}\sim\pi_\theta(\cdot|s_{h+1})$
      \State $\hat{Q}\leftarrow\hat{Q}+\gamma^{h+1}/2\brac{r(s_{h+1},a_{h+1})-\alpha \log \pi_\theta(a_{h+1}|s_{h+1})}$
    \EndFor
    \State {\bf Output:} $\hat{Q}$
  \end{algorithmic}
\end{algorithm}

\newpage
\section{Experimental Details}
We use the SAC implementation from  \url{https://github.com/ikostrikov/jaxrl}~\citep{jaxrl} for all our experiments in the paper. 
\subsection{Goal Reaching with an Oracle Curriculum}
\label{appendix:oracle-details}
For our \texttt{antmaze-umaze} experiments with oracle curriculum, we use a sparse reward function where the reward is $0$ when the distance $D$ between the ant and the goal is greater than $0.5$ and $r = \exp(-5D)$ when the distance is smaller than or equal to $0.5$. The performance threshold is set to be $R=200$.  Exceeding such threshold means that the ant stays on top of the desired location for at least $200$ out of $500$ steps, where $500$ is the maximum episode length of the \texttt{antmaze-umaze} environment. We use the average return of the last $10$ episodes and compare it to the performance threshold $R$. For both of the SAC agents, we use the same set of hyperparameters shown in \cref{table:sac-param}. See \cref{algo:detailed-impl}, for a more detailed pseudocode.

\begin{algorithm}[H]
  \caption{Practical Implementation of \ours{}}\label{algo:detailed-impl}
  \begin{algorithmic}[1]
    \State \textbf{Input:} $\{\omega_k\}_{k=0}^K$: input curriculum, $\rho$: initial state distribution, $R$: near-optimal threshold, $\beta$: roll-in ratio, discount factor $\gamma$. 
    \State Initialize $\mathcal{D} \leftarrow \emptyset,\mathcal{D}_{\mathrm{exp}} \leftarrow \emptyset,k \leftarrow 0$, and two off-policy RL agents $\pi_\mathrm{main}$ and $\pi_\mathrm{exp}$.
    \For {each environment step}
        \If{episode terminating or beginning of training}
            \If {average return of the last 10 episodes under  context $\omega_k$ is greater than $R$ }
                \State $k \leftarrow k + 1$, $\mathcal{D}_{\mathrm{exp}} \leftarrow \emptyset$
                \State Re-initialize the exploration agent $\pi_{\mathrm{exp}}$
            \EndIf
            \State Start a new episode under context $\omega_k$ with $s_0 \sim \rho$, $t \leftarrow 0$
            \If {$k > 0$ and with probability of $\beta$}
                \State enable \textbf{Rollin} for the current episode.
            \Else
                \State disable \textbf{Rollin} for the current episode.
            \EndIf
        \EndIf
        \If{\textbf{Rollin} is enabled for the current episode}
            \If{\textbf{Rollin} is stopped for the current episode}
                \State $a_t \sim \pi_{\mathrm{exp}}(a_t | s_t, \omega_k)$
            \Else
                \State $a_t \sim \pi_{\mathrm{main}}(a_t | s_t, \omega_{k-1})$
                \State with probability of $1 - \gamma$, stop \textbf{Rollin} for the current episode
            \EndIf
        \Else
            \State $a_t \sim \pi_{\mathrm{main}}(a_t | s_t, \omega_{k})$
        \EndIf
        \State take action $a_t$ in the environment and receives $s_{t+1}$ and $r_t = r_{\omega_k}(s_t, a_t)$
        \State add $(s_t, a_t, s_{t+1}, r_t)$ in replay buffer $\mc{D}$
        \If{\textbf{Rollin} is disabled for the current episode}
            \State update $\pi_{\mathrm{main}}$ using $\mc{D}$.
        \EndIf
        \If{$\pi_{\mathrm{exp}}$ was used to produce $a_t$}
            \State add $(s_t, a_t, s_{t+1}, r_t)$ in replay buffer $\mc{D}_{\mathrm{exp}}$
            \State update $\pi_{\mathrm{exp}}$ using $\mc{D}_{\mathrm{exp}}$.
        \EndIf
        \State $t \leftarrow t + 1$
    \EndFor
    \State \textbf{Output:} $\pi_{\mathrm{main}}$
  \end{algorithmic}
\end{algorithm}

\begin{figure}[h]
    \centering
    \includegraphics[width=0.35\linewidth]{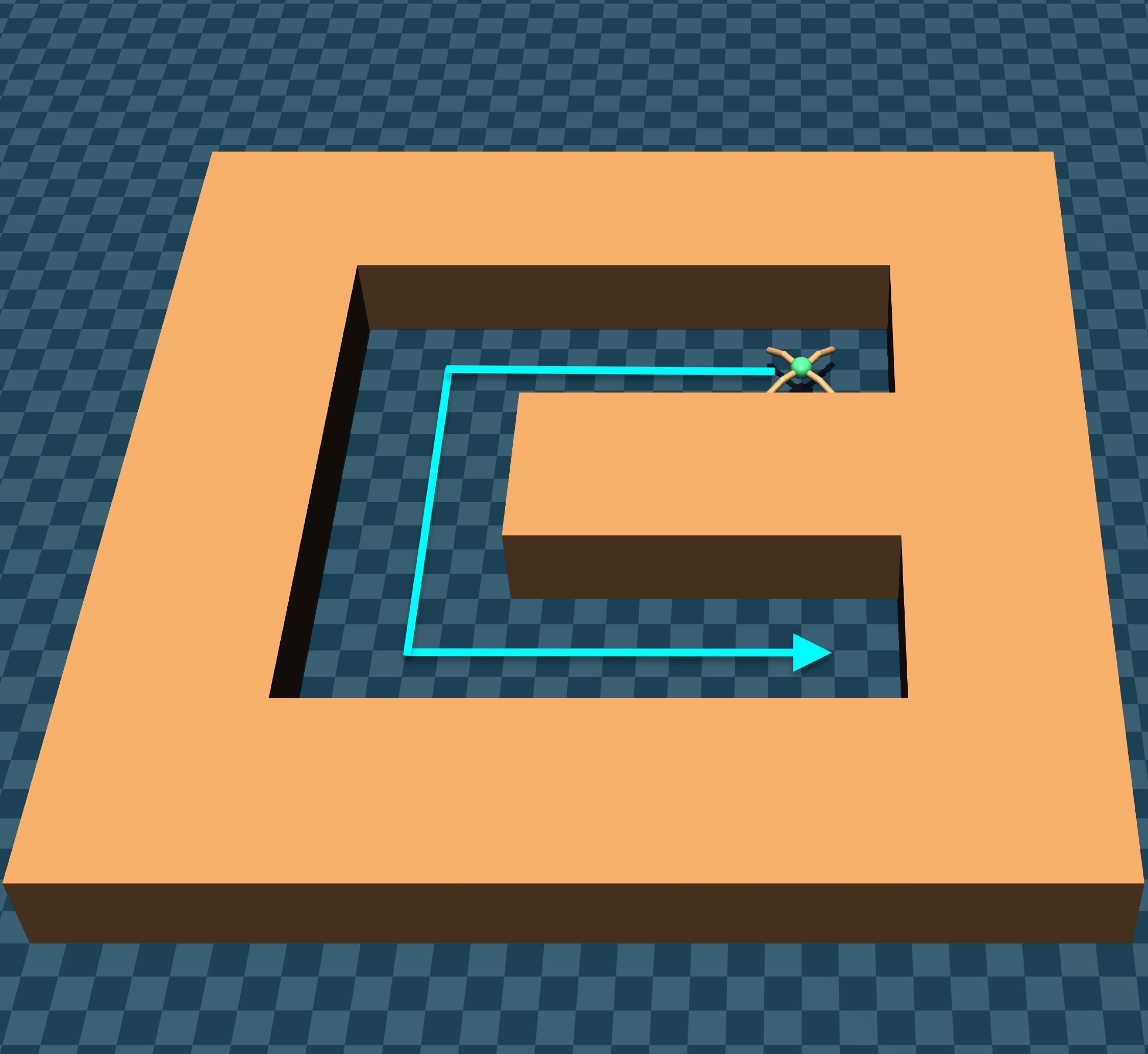}
    \captionof{figure}{\textbf{Oracle curriculum of desired goals on} \texttt{antmaze-umaze}. The ant starts from the right top corner and the farthest goal is located at the bottom right corner.}
    \label{fig:antmaze-path-cshape} 
\end{figure}

\begin{center}
\begin{table}[H]
\centering
\small
\def\arraystretch{1.35}
\begin{tabular}{|l|l|c|} 
\hline
\textbf{Initial Temperature} &  & $1.0$ \\
\textbf{Target Update Rate} & update rate of target networks & $0.005$ \\
\textbf{Learning Rate} & learning rate for the Adam optimizer  & $0.0003$ \\
\textbf{Discount Factor} & & $0.99$ \\
\textbf{Batch Size} & & $256$ \\
\textbf{Warmup Period} & number of steps of initial random exploration (random actions) & $10000$ \\
\textbf{Network Size} &  & $(256, 256)$ \\
\hline
\end{tabular}
\caption{Hyperparameters used for the SAC algorithm~\citep{haarnoja2018soft}}
\label{table:sac-param}
\end{table}
\end{center}

\subsection{Non-Goal Reaching}

\label{appendix:meta-learning-details}

For the non goal reaching tasks in \texttt{walker2d}, \texttt{hopper}, \texttt{humanoid}, and \texttt{ant} experiments, the desired $x$-velocity range $[\lambda\kappa,\lambda(\kappa+0.1))$, the near-optimal threshold $R(\kappa)$, and the $\texttt{healthy\_reward}$ all depend on the environments. The maximum episode length 1000. Details are provided in \cref{table:appendix-meta-learning-params-details}.

\begin{table}[ht]
\centering
\begin{minipage}[b]{\linewidth}
\centering
\small
\begin{tabular}{c|ccccc}
\toprule
\multicolumn{1}{c}{}&
\multicolumn{2}{c}{}&
\multicolumn{3}{c}{\texttt{healthy\_reward}}\\
\cmidrule(lr){4-6} \multicolumn{1}{c}{\texttt{Env.}} & $\lambda$& $R(\kappa)$ & original & \texttt{high} & \texttt{low} \\
\midrule \texttt{walker}  & $5$ & $500+4500\kappa$ & $1.0$ & $1.5$ & $0.5$\\ 
\midrule \texttt{hopper} & $3$ & $500+4500\kappa$ & $1.0$ & $1.5$ & $0.5$\\
\midrule
\texttt{humanoid} & $1$  & $2500+2500\kappa$ & $5.0$ & $7.5$ & $2.5$\\
\midrule
\texttt{ant} & $6$ & $500+4500\kappa$ & {$1.0$} & $1.5$ & $0.25$ \\
\bottomrule
\end{tabular}
\caption{\textbf{Learning progress $\kappa$, average $x$-velocity, and average return at the 0.75 and 1.0 million environment steps in \texttt{walker}, \texttt{hopper}, \texttt{humanoid}, and \texttt{ant}.} The average $x$-velocity and return are estimated using the last 50k time steps. We pick $\beta = 0.1$ for all experiments using \ours{}, the results of using other $\beta$s can be found in \cref{table:meta-learning-progress}, \cref{table:meta-learning-speed}, and \cref{table:meta-learning-return} in \cref{appendix:learning-curves-meta-learning}. The standard error is computed over 8 random seeds.} 
\label{table:appendix-meta-learning-params-details}
\end{minipage}
\hfill
\end{table}

\section{Numerical Experiments: The Four Room Navigation}
\label{appendix:numerical}
\subsection{MDP Setup}
The grid world consists of $12 \times 12$ grid cells where each cell corresponds to a state in the MDP. The agent can take four different actions to move itself in four directions to a different cell or take a fifth action to receive reward (positive if close to the goal, 0 otherwise). Each context in the context space represents a distinct goal state in the grid-world. The agent (when taking the fifth action) receives higher reward the closer it is to the goal state and receives 0 reward when it is too far (4 steps away for easy, and 5 steps away for hard). We also include 100 additional dummy actions in the action space (taking these actions do not result in reward nor state changes) to make the exploration problem challenging. See Figure \ref{fig:numerical-mdp} for a visualization of the environment and the two reward functions we use. More concretely, let $D(s, g)$ be the number of action it takes to go from state $s$ to state $g$ (the current goal) if the walls did not exist (the Manhattan distance), the reward received when taking the fifth action at state $s$ is
\begin{align*}
    r_{\mathrm{four\_room}}(s) = \begin{cases}
        \gamma_{\mathrm{reward}}^{D(s, g)}, & D(s, g) \leq D_{\mathrm{threshold}} \\
        0, & D(s, g) > D_{\mathrm{threshold}} \\
    \end{cases}
\end{align*}
For the easy reward function, $\gamma_{\mathrm{reward}} = 0.9$, $D_{\mathrm{threshold}} = 5$. For the hard reward function, $\gamma_{\mathrm{reward}} = 0.5$, $D_{\mathrm{threshold}} = 4$.

\subsection{Pre-defined curriculum.} Our curriculum contains $16$ contexts in sequence, $\{\omega_k\}_{k=0}^{16}$, which form a continuous path from the start location of the agent $(0, 0)$ to the goal location of the agent at $(8, 8)$. We use a fixed success rate threshold (an episode is considered to be successful if the agent reaches the goal state and perform the fifth action at that goal state) to determine convergence of the stochastic PG algorithm. We switch to the next context/curriculum step whenever the success rate exceeds 50\%. We use $\kappa \in [0, 1]$ to denote a normalized curriculum progress which is computed as the current curriculum step index divided by the total number of curriculum steps.

\subsection{Stochastic PG description} We follow Algorithm \ref{algo:PGSPG_context} closely for our implementation. In particular, we adopt the softmax parameterization of $\pi$ that is parameterized by $\theta \in \mathbb{R}^{S \times A}$ as $\pi_\theta(a=j | s = i) = \frac{\exp(\theta_{ij})}{\sum_{j'} \exp(\theta_{ij'})}$ with $i \in [S]$ and $j \in [A]$ (in this MDP, $S = 144$ as there are $12 \times 12 = 144$ cells in the grid world and $A = 105$ due to the dummy actions). To sample from $d_{\mu_{k-1}}^{\pi^\star_{\omega_{k-1}}}$ (Line 6), we rollout the policy from the previous context $\pi_{\theta_{k-1}}$ in the MDP for $h$ steps (where $h$ being sampled from $\mathrm{Geom}(1 - \gamma)$, the geometric distribution with a rate of $1 - \gamma$) and take the resulting state as a sample from $d_{\mu_{k-1}}^{\pi^\star_{\omega_{k-1}}}$. We implement the stochastic PG using Adam optimizer~\citep{kingma2017adam} on $\theta$ with a constant learning rate of 0.001. Every gradient step is computed over 2000 trajectories with the trajectory length capped at 50 for each. The empirical policy gradient for $\pi_\theta$ is computed as over $B=2000$ trajectories ($\{(s^b_0, a^b_0, s^b_1, \cdots, s^b_T \}_{b=1}^{B}$) and $T=50$ time steps collected by rolling out the current policy $\pi_\theta$: $\frac{1}{BT}\sum_{b} \sum_{t}\nabla_\theta \log \pi_\theta(a^b_t | s^b_t) R^b_t$ with $R^b_t = -\sum_{t'=t}^{T} \gamma^{t' - t} r^b_{\mathrm{ent}, t'}, r^b_{\mathrm{ent}, t} = r(s^b_t, a^b_t) - \alpha \log \pi(a^b_t | s^b_t)$ where $R^b_t$ is the Monte-Carlo estimate of the discounted, entropy-regularized cumulative reward.

\subsection{Results}
\label{appendix:numerical-detailed}
\begin{table*}[h]
\centering
\begin{minipage}[b]{\columnwidth}
\centering
\tiny
\begin{tabular}{cc|ccccccc}
\toprule
\multicolumn{1}{c}{Setting} & \multicolumn{1}{c}{Entropy Coefficient} & $\beta = 0.0$ (Baseline)& $\beta = 0.1$& $\beta = 0.2$ & $\beta = 0.3$ & $\beta = 0.5$ & $\beta = 0.75$ & $\beta = 0.9$\\ \midrule
Hard & $\alpha = 0.01$ & $0.500\pm 0.000$ & 
$0.506\pm0.001$ & $0.512\pm0.001$ & {\color{purple}$0.525\pm 0.000$} & {\color{purple}$0.562 \pm 0.000$} & {\color{purple}$0.562 \pm 0.000$} & {\color{purple}$0.562 \pm 0.000$} \\
& $\alpha=0.001$& $0.856\pm0.006$ & $0.981\pm0.003$ & $0.981\pm0.001$& {\color{purple} $1.000\pm0.000$} & {\color{purple} $1.000\pm0.000$} & {\color{purple} $1.000\pm0.000$} & {\color{purple} $1.000\pm0.000$} \\ \midrule
Easy & $\alpha=0.01$ & $0.944 \pm 0.003$ & $0.994\pm0.001$ & $0.994\pm0.001$ & {\color{purple}$1.000\pm0.000$} & {\color{purple}$1.000\pm0.000$} & {\color{purple}$1.000\pm0.000$} & {\color{purple}$1.000\pm0.000$}\\
& $\alpha = 0.001$ &  {\color{purple}$1.000\pm0.000$}& {\color{purple}$1.000\pm0.000$} &{\color{purple}$1.000\pm0.000$}& {\color{purple}$1.000\pm0.000$} & {\color{purple}$1.000\pm0.000$} & {\color{purple}$1.000\pm0.000$} & {\color{purple}$1.000\pm0.000$} \\
\bottomrule
\end{tabular}
\caption{\footnotesize \textbf{Curriculum progress $\kappa$ on the four-room navigation with stochastic PG at step 50,000.} We tested with two different entropy coefficients and seven different $\beta$'s. The standard error is computed over 10 random seeds.}

\label{table:numerical-progress}
\end{minipage}
\hfill
\end{table*}

\begin{table*}[h]
\centering
\begin{minipage}[b]{\columnwidth}
\centering
\tiny
\begin{tabular}{cc|ccccccc}
\toprule
\multicolumn{1}{c}{Setting} & \multicolumn{1}{c}{Entropy Coefficient} & $\beta = 0.0$ (Baseline)& $\beta = 0.1$& $\beta = 0.2$ & $\beta = 0.3$ & $\beta = 0.5$ & $\beta = 0.75$ & $\beta = 0.9$\\ \midrule
Hard & $\alpha = 0.01$ & $0.000\pm 0.000$ & 
$0.000\pm0.000$ & $0.000\pm0.000$ & $0.000\pm 0.000$ & 
$0.000\pm0.000$ & $0.000\pm0.000$ & $0.000\pm0.000$ \\
& $\alpha=0.001$& $0.424\pm0.023$ & $0.939\pm0.0014$ & $0.710\pm0.021$& $1.010\pm0.005$ & $1.062\pm0.000$ & {\color{purple} $1.067\pm0.000$} & $1.060\pm0.000$ \\ \midrule
Easy & $\alpha=0.01$ & $4.093 \pm 0.224$ & $5.136\pm0.230$ & $4.156\pm0.228$ & $1.040\pm0.140$ & $3.913\pm0.218$ & {\color{purple}$7.374\pm0.216$} & $4.227\pm0.232$\\
& $\alpha = 0.001$ &  $10.536\pm0.002$& $10.566 \pm 0.003$ & $10.602\pm0.003$ & $10.593\pm0.002$  & $10.611\pm0.002$  & {\color{purple}$10.620\pm0.002$}  & $10.575\pm0.002$  \\
\bottomrule
\end{tabular}
\caption{\footnotesize \textbf{Final return $V^\pi$ on the four-room navigation with stochastic PG at step 50,000.} We tested with two different entropy coefficients and seven different $\beta$'s. The standard error is computed over 10 random seeds.}

\label{table:numerical-return}
\end{minipage}
\hfill
\end{table*}
\begin{figure}[H]
    \centering
    \includegraphics[width=0.87\textwidth]{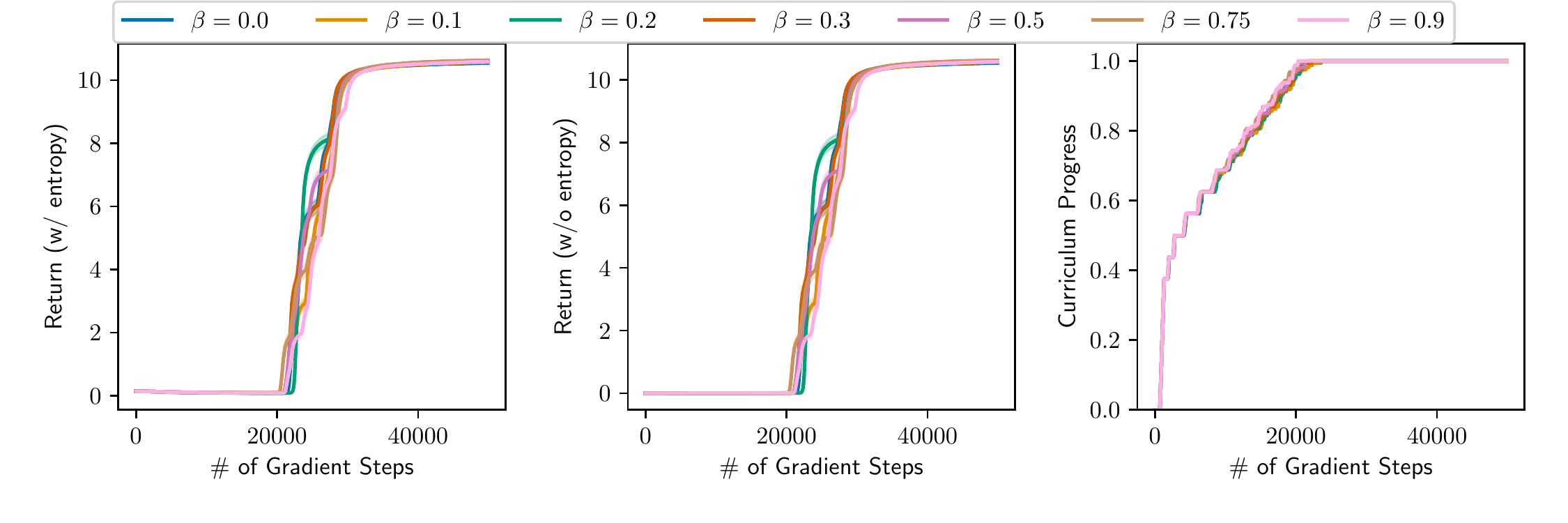}
    \includegraphics[width=0.87\textwidth]{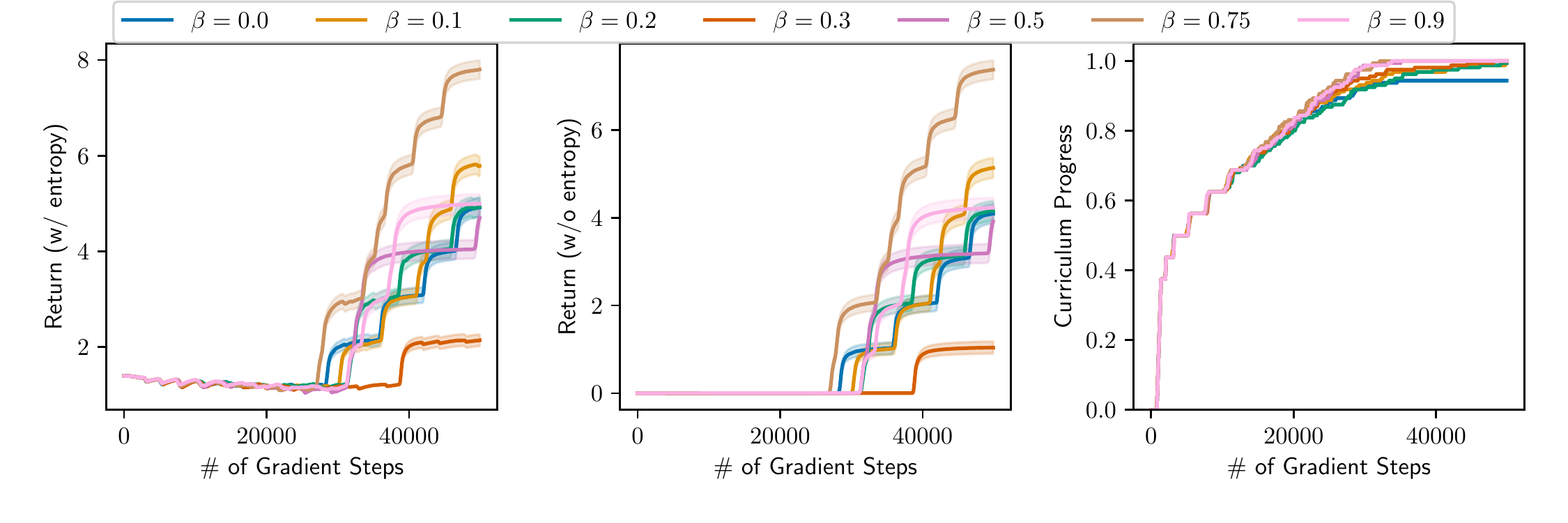}
    \caption{\textbf{Learning curves for the numerical experiments on the easy curriculum.}}
    \label{fig:numerical-easy}
\end{figure}

\begin{figure}[H]
    \centering
    \includegraphics[width=0.87\textwidth]{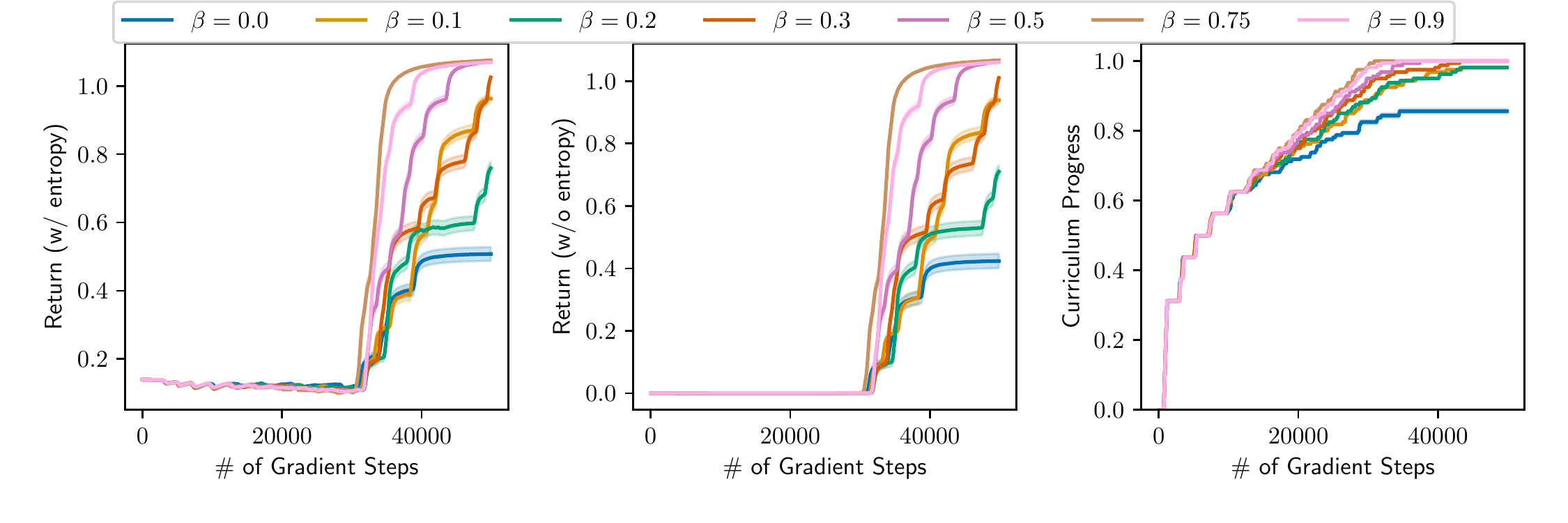}
    \includegraphics[width=0.87\textwidth]{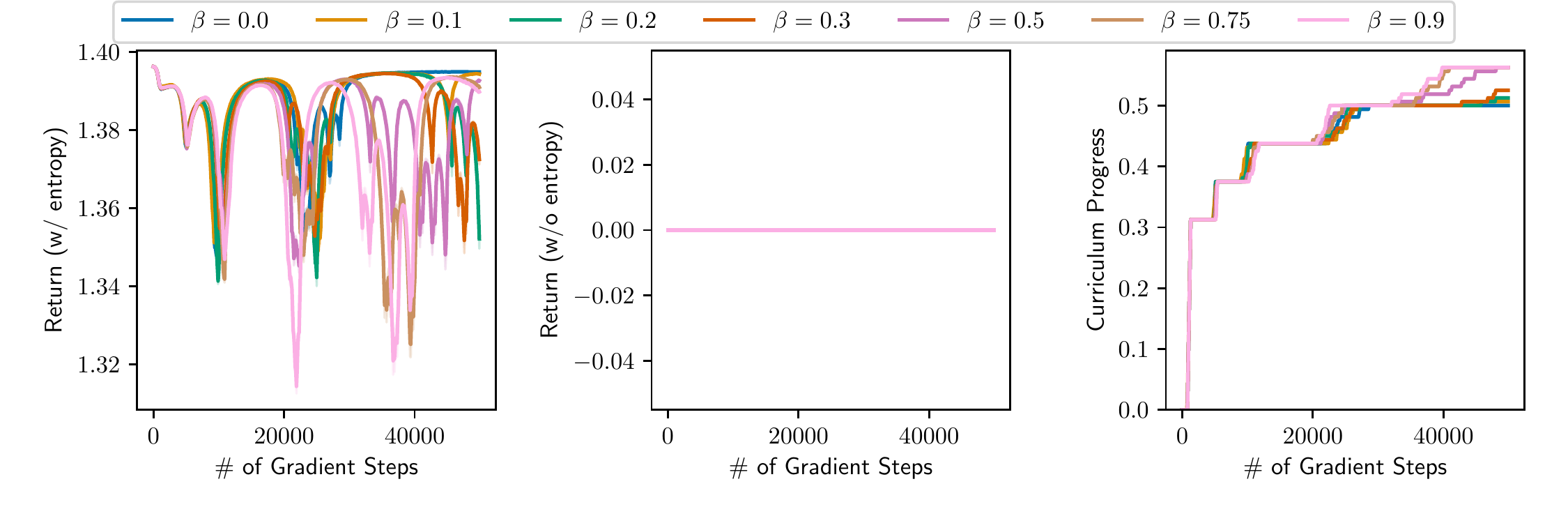}
    \caption{\textbf{Learning curves for the numerical experiments on the hard curriculum.}}
    \label{fig:numerical-hard}
\end{figure}

\section{Additional Learning Curves and Tables}
\subsection{Goal Reaching}
\label{appendix:learning-curves-goal-reaching-oracle}
\begin{figure}[H]
    \centering
    \begin{minipage}{0.95\linewidth}
    \centering\captionsetup[subfigure]{justification=centering}
    \includegraphics[width=\linewidth]{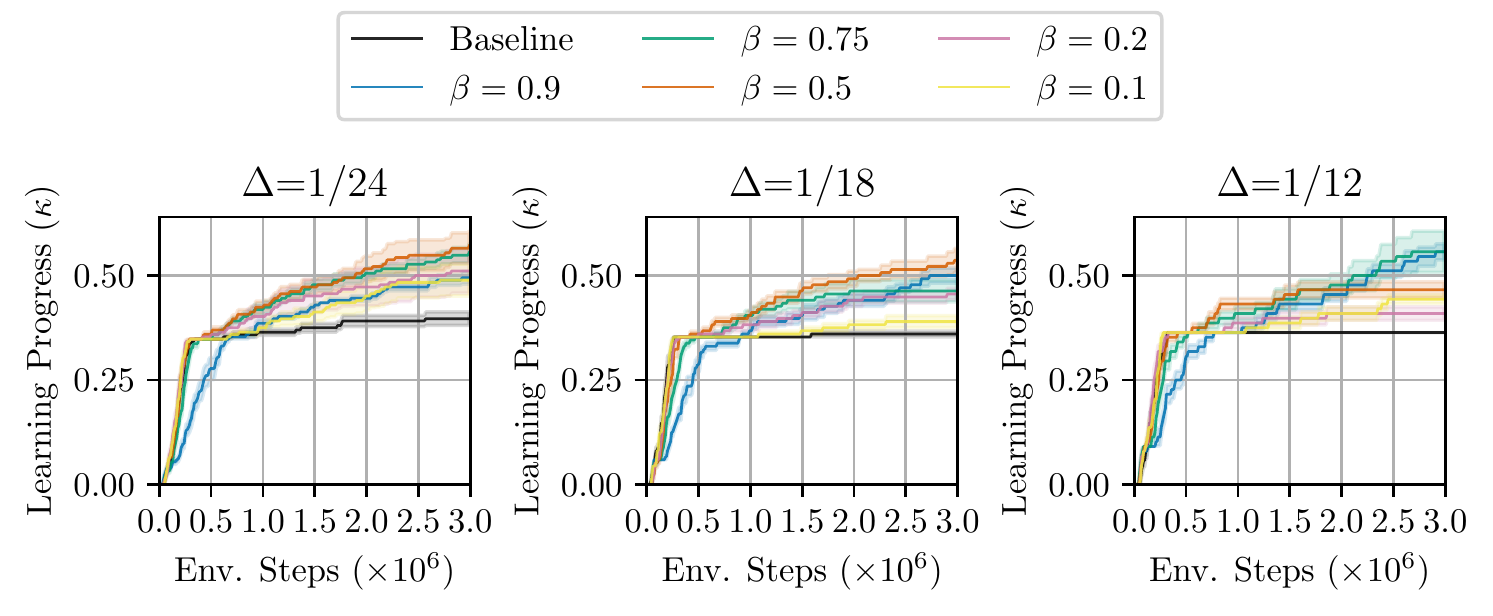}
        \label{fig:drl-vannila-rollin-geo-false} 
        \subcaption{Vanilla \ours{} without geometric sampling}
    \includegraphics[trim=0 0 0 0,clip,width=\textwidth]{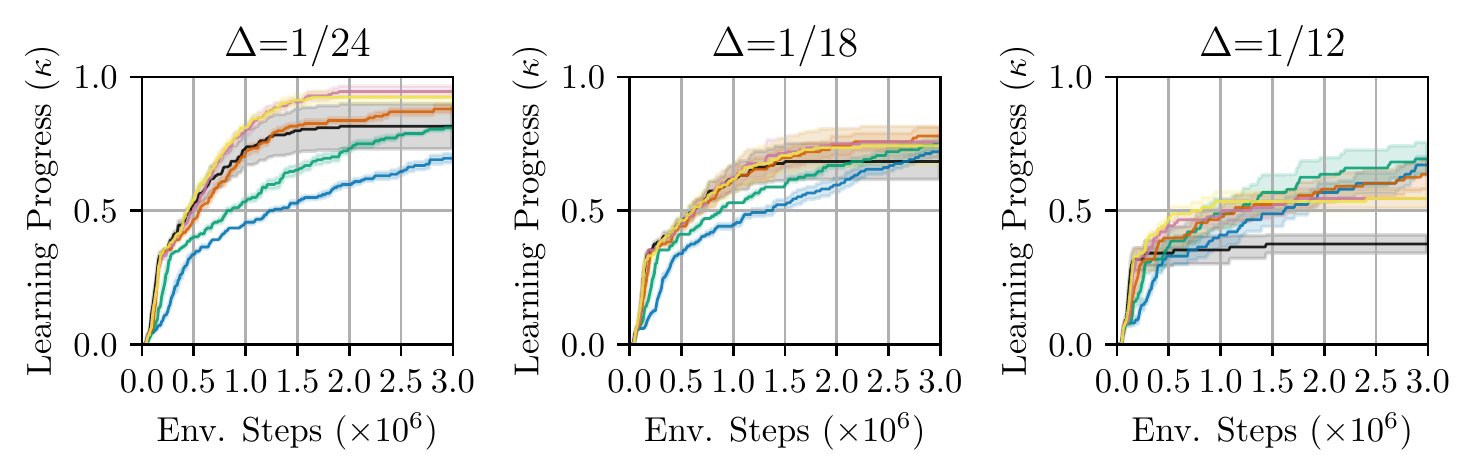}
        \label{fig:drl-vannila-rollin-geo-true} 
        \subcaption{Vanilla \ours{} with geometric sampling}
        \end{minipage}
    \caption{\textbf{Vanilla Goal reaching.} Accelerating learning on \texttt{antmaze-umaze} with \ours{} on an oracle curriculum in \cref{fig:antmaze-path-cshape}. The confidence interval represents the standard error computed over 8 random seeds.}
    \label{fig:drl-vannila-rollin}
\end{figure}
\begin{figure}[H]
    \centering
    \begin{minipage}{0.95\linewidth}
    \centering\captionsetup[subfigure]{justification=centering}
    \includegraphics[width=\linewidth]{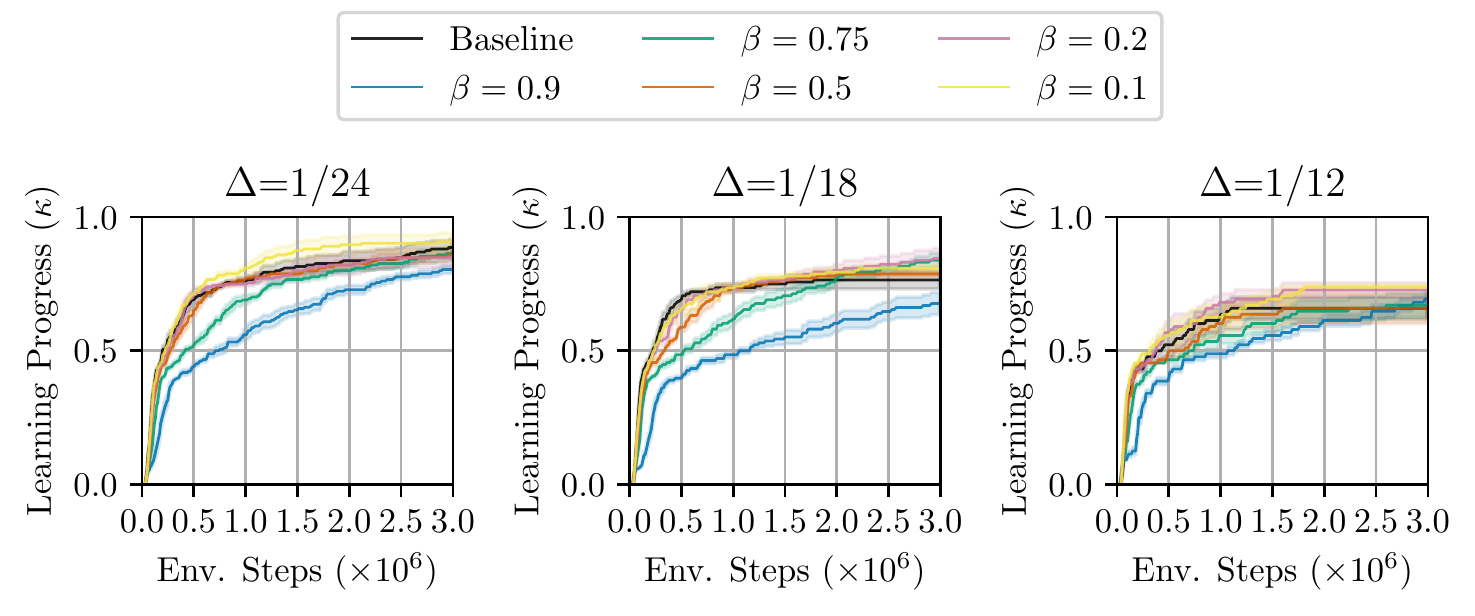}
        \label{fig:drl-rel-rollin-geo-false}
        \subcaption{\ours{} + relabeling without geometric sampling}
    \includegraphics[trim=0 0 0 0,clip,width=\textwidth]{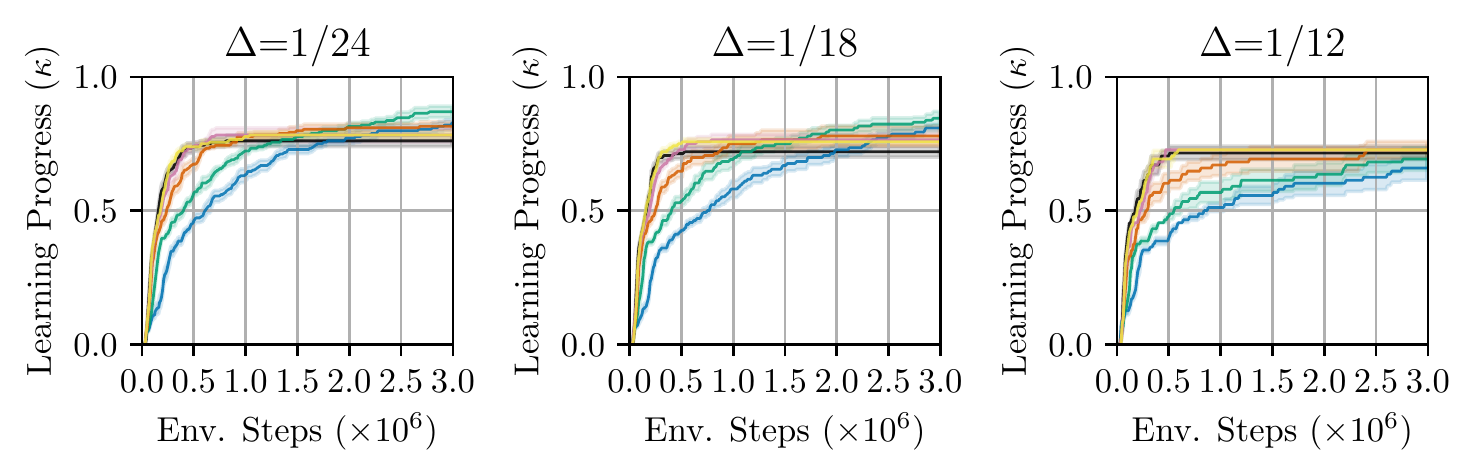}
        \label{fig:drl-rel-rollin-geo-true}
        \subcaption{\ours{} + relabeling with geometric sampling}
        \end{minipage}
    \caption{\textbf{Goal relabeling.} Accelerating learning on \texttt{antmaze-umaze} with \ours{} on an oracle curriculum in \cref{fig:antmaze-path-cshape}. The confidence interval represents the standard error computed over 8 random seeds.}
    \label{fig:drl-relabeling-rollin}
\end{figure}
\begin{figure}[H]
    \centering
    \begin{minipage}{0.95\linewidth}
    \centering\captionsetup[subfigure]{justification=centering}
    \includegraphics[width=\linewidth]{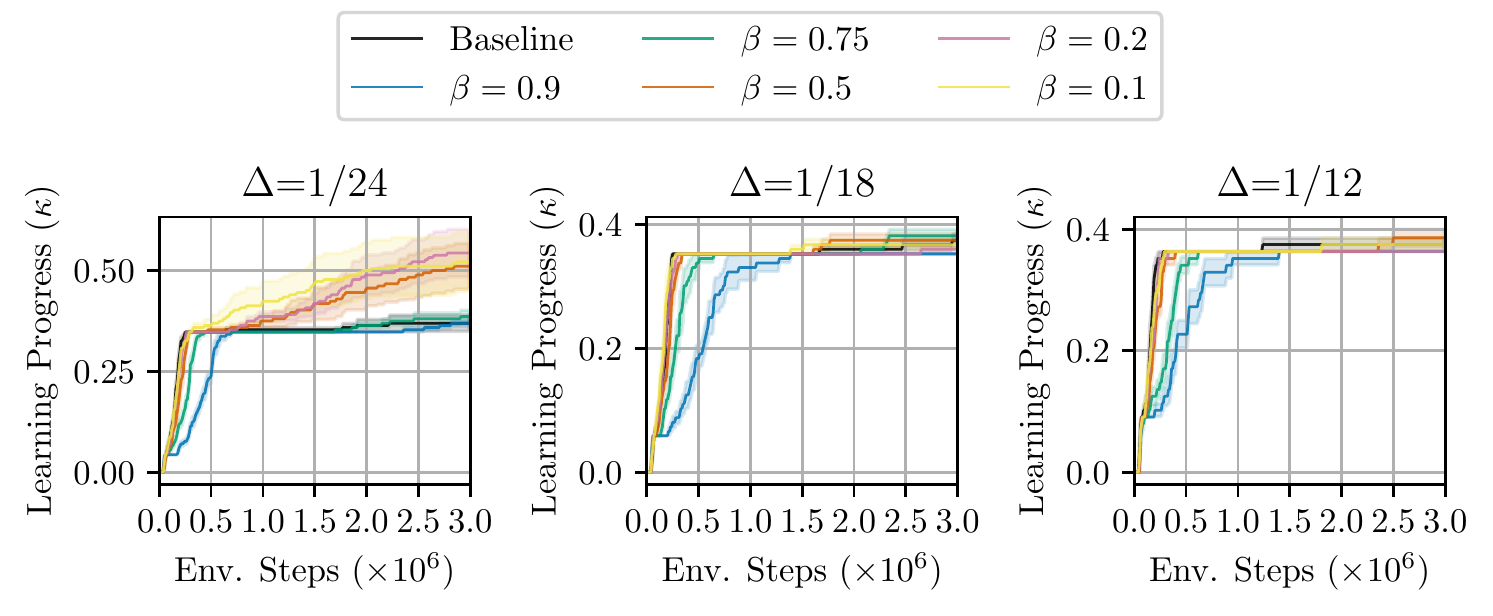}
        \label{fig:drl-go-expl-0.1-rollin-geo-false}
        \subcaption{\ours{} + Go-Explore without geometric sampling}
    \includegraphics[trim=0 0 0 0,clip,width=\textwidth]{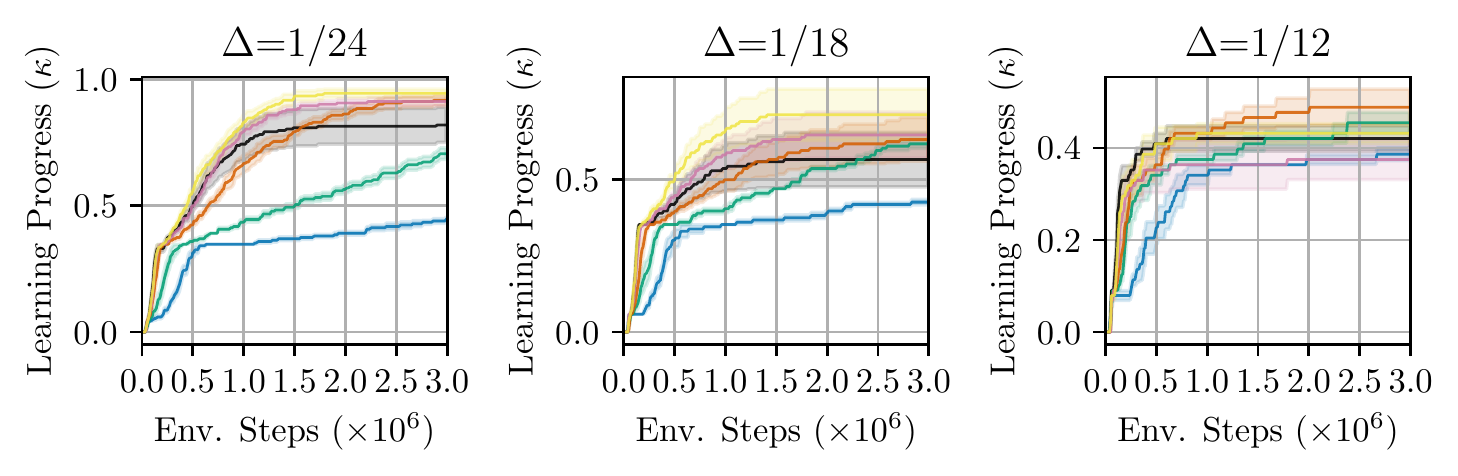}
        \label{fig:drl-go-expl-0.1-rollin-geo-true}
        \subcaption{\ours{} + Go-Explore with geometric sampling}
        \end{minipage}
    \caption{\textbf{Go-Explore (exploration noise = 0.1).} Accelerating learning on \texttt{antmaze-umaze} with \ours{} on an oracle curriculum in \cref{fig:antmaze-path-cshape}. The confidence interval represents the standard error computed over 8 random seeds.}
    \label{fig:drl-go-explore-0.1-rollin}
\end{figure}

\begin{figure}[H]
    \centering
    \begin{minipage}{0.95\linewidth}
    \centering\captionsetup[subfigure]{justification=centering}
    \includegraphics[width=\linewidth]{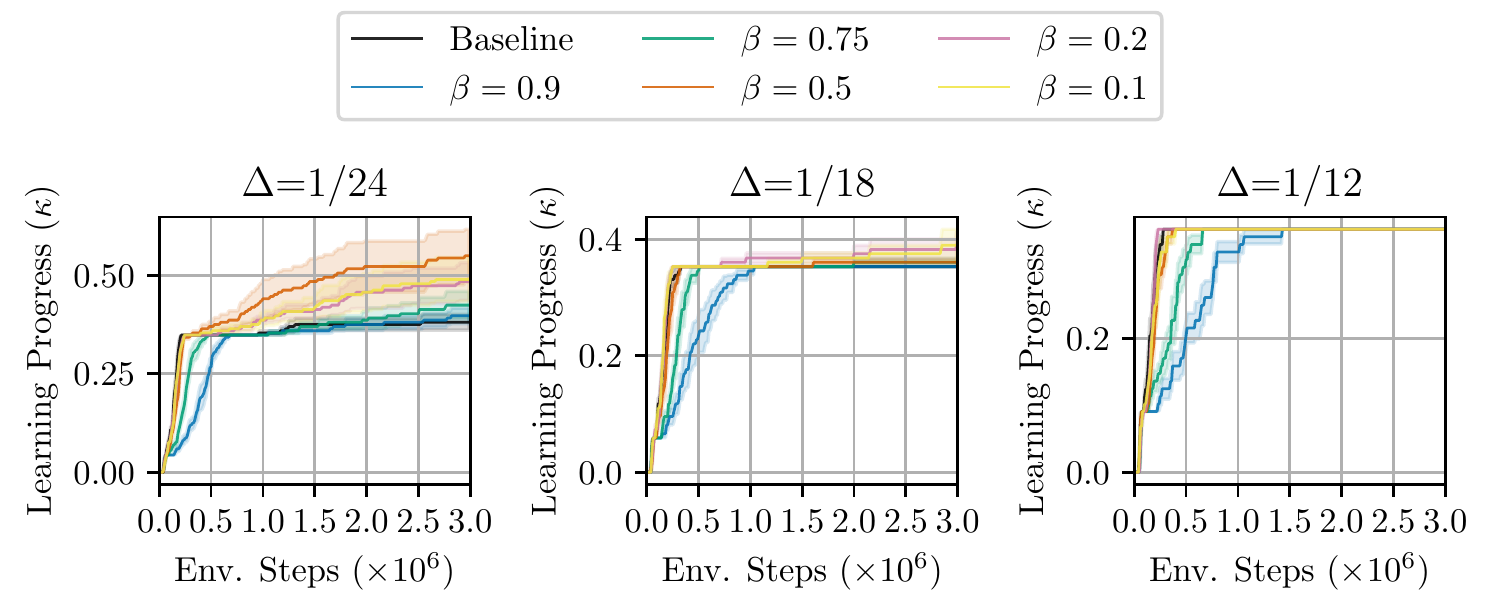}
        \label{fig:drl-go-expl-0.25-rollin-geo-false}
        \subcaption{\ours{} + Go-Explore without geometric sampling}
    \includegraphics[trim=0 0 0 0,clip,width=\textwidth]{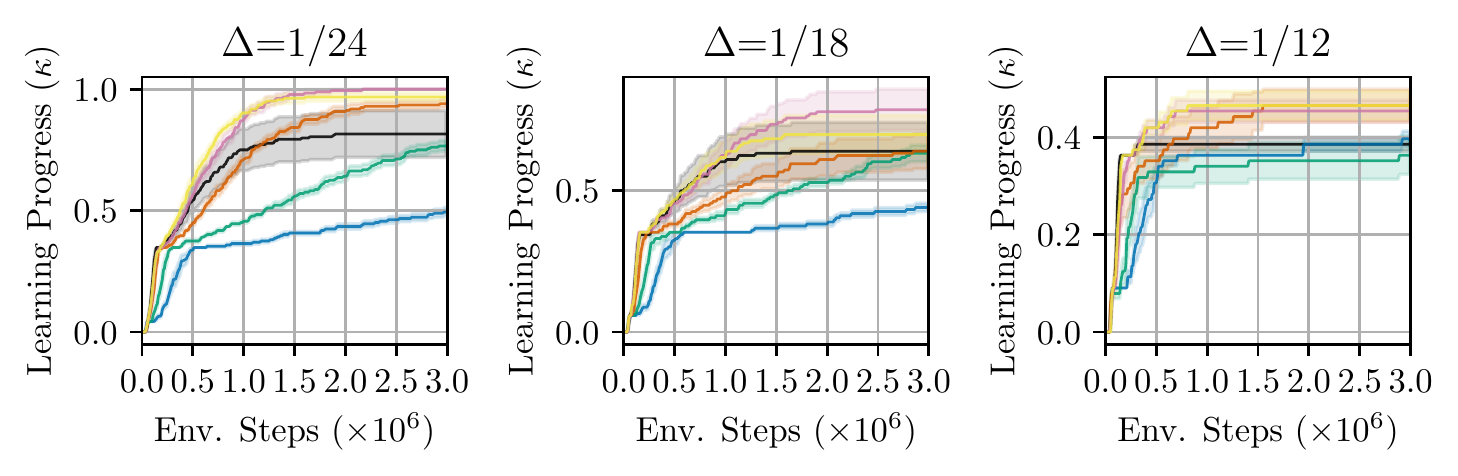}
        \label{fig:drl-go-expl-0.25-rollin-geo-true}
        \subcaption{\ours{} + Go-Explore with geometric sampling}
        \end{minipage}
    \caption{\textbf{Go-Explore (exploration noise = 0.25).} Accelerating learning on \texttt{antmaze-umaze} with \ours{} on an oracle curriculum in \cref{fig:antmaze-path-cshape}. The confidence interval represents the standard error computed over 8 random seeds.}
    \label{fig:drl-go-explore-0.25-rollin}
\end{figure}
\begin{figure}[H]
    \centering
    \begin{minipage}{0.95\linewidth}
    \centering\captionsetup[subfigure]{justification=centering}
    \includegraphics[width=\linewidth]{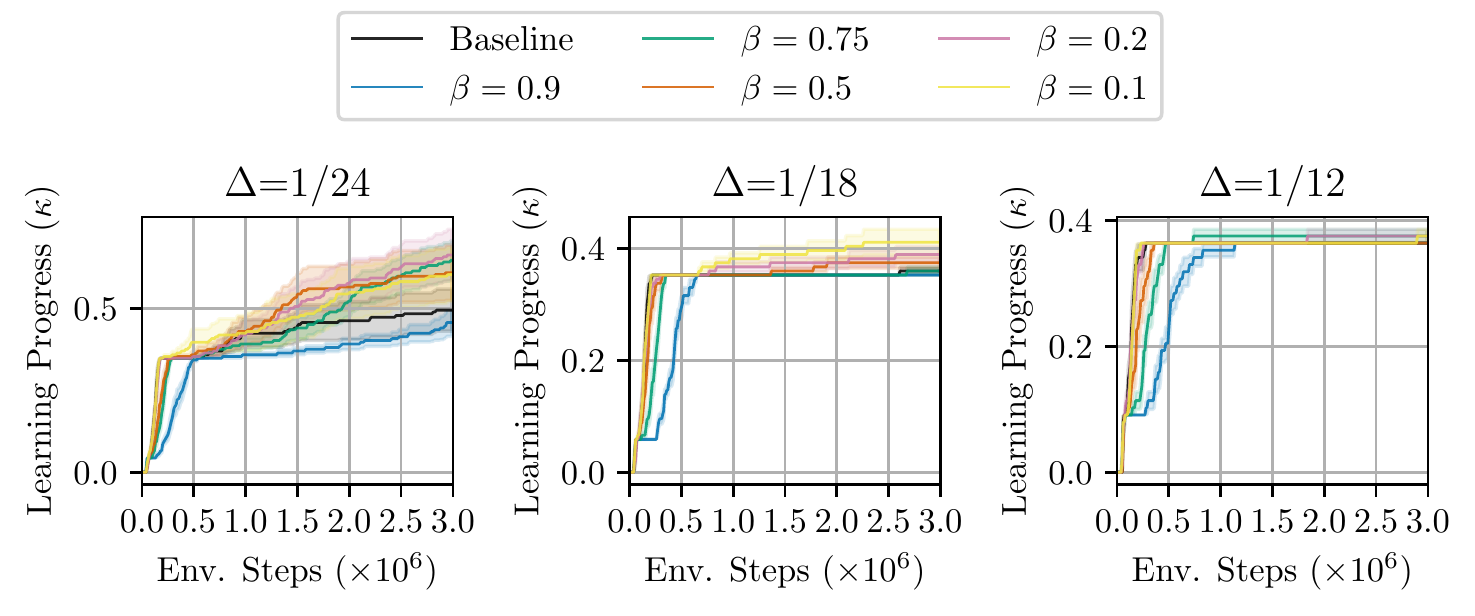}
        \label{fig:drl-go-expl-0.5-rollin-geo-false}
        \subcaption{\ours{} + Go-Explore without geometric sampling}
    \includegraphics[trim=0 0 0 0,clip,width=\textwidth]{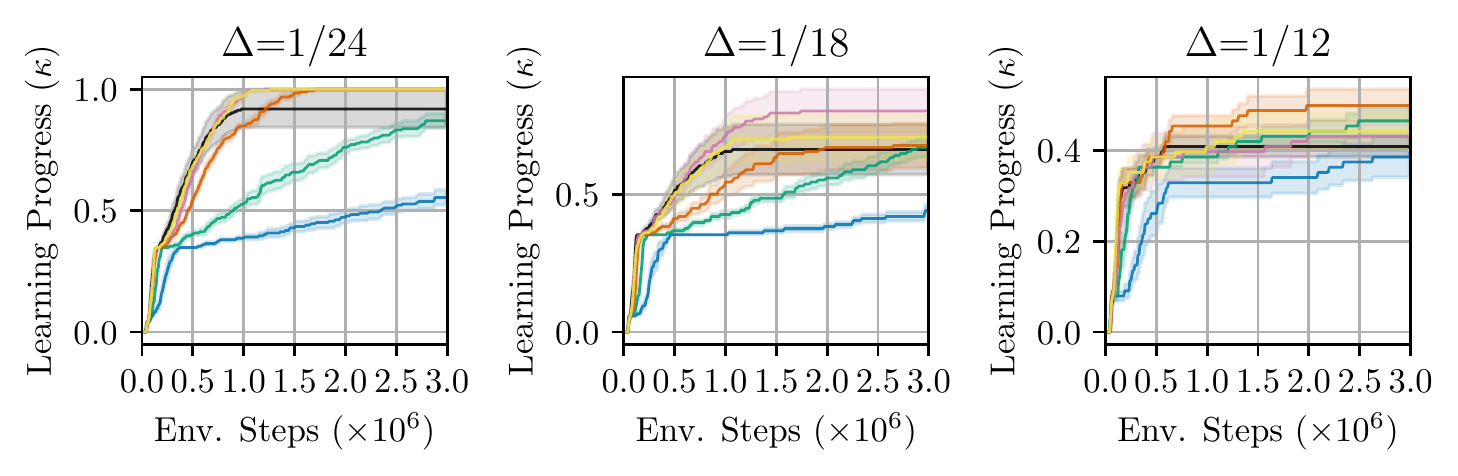}
        \label{fig:drl-go-expl-0.5-rollin-geo-true}
        \subcaption{\ours{} + Go-Explore with geometric sampling}
        \end{minipage}
    \caption{\textbf{Go-Explore (exploration noise = 0.5).} Accelerating learning on \texttt{antmaze-umaze} with \ours{} on an oracle curriculum in \cref{fig:antmaze-path-cshape}. The confidence interval represents the standard error computed over 8 random seeds.}
    \label{fig:drl-go-explore-0.5-rollin}
\end{figure}

\newpage
\begin{table}[H]
\centering
\begin{minipage}[b]{\linewidth}
\centering
\small
\begin{tabular}{cccccccc}
\toprule
\texttt{Geo}&$\Delta$& Baseline &$\beta=0.1$&$\beta=0.2$ & $\beta=0.5$ & $\beta=0.75$ & $\beta=0.9$ \\ \midrule
\xmark&$1/24$ & $0.40 \pm 0.02$ &{\color{purple}$0.49\pm0.04$} & {\color{purple} $0.51\pm0.05$} & {\color{purple}$\mb{0.57\pm0.04}$} &{\color{purple}$0.55\pm0.02$} & {\color{purple}$0.49\pm0.02$}\\
\xmark&$1/18$ & $0.36\pm0.01$ &{\color{purple}$0.39\pm0.01$}& {\color{purple}$0.46\pm0.01$} &{\color{purple}$\mb{0.54\pm0.03}$} & {\color{purple}$0.46\pm0.03$} & {\color{purple}$0.50\pm0.02$} \\ 
\xmark&$1/12$ & $0.36\pm0.00$ &{\color{purple}$0.44\pm0.01$}&{\color{purple}$0.41\pm0.02$} & {\color{purple}$0.47\pm0.02$} & {\color{purple}$\mb{0.56\pm0.05}$} & {\color{purple}$\mb{0.56\pm0.02}$}\\
\midrule
\cmark&$1/24$ & $0.82 \pm 0.08$ &{\color{purple}$0.92\pm0.02$} & {\color{purple}$\mb{0.95\pm0.02}$} & {\color{purple}$0.88\pm0.01$} &$0.81\pm0.01$ & $0.70\pm0.02$\\
\cmark&$1/18$ & $0.68\pm0.07$ &{\color{purple}$0.74\pm0.07$}& {\color{purple}$0.76\pm0.06$} &{\color{purple}$\mb{0.78\pm0.03}$} & {\color{purple}$0.75\pm0.02$} & {\color{purple}$0.72\pm0.02$} \\ 
\cmark&$1/12$ & $0.38\pm0.03$ &{\color{purple}$0.55\pm0.04$}&{\color{purple}$0.55\pm0.04$} & {\color{purple}$0.64\pm0.06$} & {\color{purple}$\mb{0.69\pm0.06}$} & {\color{purple}$0.67\pm0.03$}\\\bottomrule
\end{tabular}
\caption{\textbf{Vanilla Goal reaching.} Learning progress $\kappa$ at 3 million environment steps with varying $\beta$ and curriculum step size $\Delta$ of vanilla goal reaching task. \texttt{Geo} indicates the usage of geometric sampling. Baseline corresponds to $\beta = 0$, where no \ours{} is used. The standard error is computed over 8 random seeds. We highlight the values that are larger than the baseline ($\beta=0$) in {\color{purple} purple}, and the largest value in {\color{purple} \bf bold font}.}
\label{table:vanilla-antumaze}
\end{minipage}
\hfill
\end{table}

\begin{table}[H]
\centering
\begin{minipage}[b]{\linewidth}
\centering
\small
\begin{tabular}{cccccccc}
\toprule
\texttt{Geo}&$\Delta$&$\beta=0$ &$\beta=0.1$&$\beta=0.2$ & $\beta=0.5$ & $\beta=0.75$ & $\beta=0.9$ \\ \midrule
\xmark&$1/24$ & $0.89\pm 0.03$ &{\color{purple}$\mb{0.91\pm0.03}$} & $0.85\pm0.04$ & $0.86\pm0.02$ &$0.86\pm0.02$ & $0.80\pm0.02$\\
\xmark&$1/18$ & $0.76\pm0.03$ &{\color{purple}$0.81\pm0.01$}& {\color{purple}$\mb{0.85\pm0.04}$} &{\color{purple}$0.79\pm0.01$} & {\color{purple}$0.84\pm0.03$} & $0.68\pm0.04$ \\ 
\xmark&$1/12$ & $0.66\pm0.04$ &{\color{purple}$\mb{0.74\pm0.01}$}&{\color{purple}$0.73\pm0.03$} & $0.66\pm0.06$ & {\color{purple}$0.67\pm0.05$} & {\color{purple}$0.69\pm0.03$}\\
\midrule
\cmark&$1/24$ & $0.76 \pm 0.02$ &{\color{purple}$0.78\pm0.01$} & {\color{purple}$0.78\pm0.03$} & {\color{purple}$0.82\pm0.02$} &{\color{purple}$\mb{0.87\pm0.02}$} &{\color{purple} $0.83\pm0.02$}\\
\cmark&$1/18$ & $0.72\pm0.02$ &{\color{purple}$0.76\pm0.01$}& {\color{purple}$0.76\pm0.02$} &{\color{purple}$0.78\pm0.04$} & {\color{purple}$\mb{0.85\pm0.03}$} & {\color{purple}$0.81\pm0.01$} \\ 
\cmark&$1/12$ & $0.72\pm0.03$ &{\color{purple}$\mb{0.73\pm0.00}$}&{\color{purple}$\mb{0.73\pm0.00}$} & {\color{purple}$\mb{0.73\pm0.03}$} & $0.69\pm0.04$ & $0.66\pm0.04$\\\bottomrule
\end{tabular}
\caption{\textbf{Goal relabeling.} All other settings are the same as \cref{table:vanilla-antumaze}.}
\label{table:relabeling-antumaze}
\end{minipage}
\hfill
\end{table}

\begin{table}[H]
\centering
\begin{minipage}[b]{\linewidth}
\centering
\small
\begin{tabular}{cccccccc}
\toprule
\multicolumn{1}{c}{\texttt{EN}   \texttt{Geo}}&$\Delta$&$\beta=0$ &$\beta=0.1$&$\beta=0.2$ & $\beta=0.5$ & $\beta=0.75$ & $\beta=0.9$ \\ \midrule
0.1\quad\xmark&$1/24$ & $0.37\pm 0.02$ & {\color{purple}$0.52\pm 0.07$} & {\color{purple}$\mb{0.54\pm 0.06}$}& {\color{purple}$0.51\pm 0.06$} & {\color{purple}$0.39\pm 0.02$} & $0.37\pm 0.01$ \\
0.1\quad\xmark&$1/18$ & $0.38\pm0.01$ & $0.37\pm0.01$ & $0.36\pm0.01$ & $0.38\pm0.01$ & $0.38\pm0.01$ & $0.35\pm0.00$ \\ 
0.1\quad\xmark&$1/12$ & $0.38\pm0.01$ & $0.38\pm0.01$ & $0.36\pm0.00$ & {\color{purple}$\mb{0.39\pm0.01}$} & $0.36\pm0.00$ & $0.36\pm0.00$ \\ 
\midrule
0.1\quad\cmark&$1/24$ & $0.82\pm0.07$ &{\color{purple} $\mb{0.95\pm0.02}$} & {\color{purple} ${0.91\pm0.02}$} & {\color{purple} ${0.92\pm0.02}$} & {${0.71\pm0.02}$} & { ${0.45\pm0.01}$}\\
0.1\quad\cmark&$1/18$ & $0.57\pm0.09$ &{\color{purple} $\mb{0.71\pm0.08}$} & {\color{purple} ${0.65\pm0.07}$} & {\color{purple} ${0.63\pm0.07}$} & {\color{purple} ${0.62\pm0.02}$} & { ${0.43\pm0.01}$}\\ 
0.1\quad\cmark&$1/12$ & $0.42\pm0.03$ &{\color{purple} ${0.43\pm0.02}$} & {${0.38\pm0.04}$} & {\color{purple} $\mb{0.49\pm0.04}$} & {\color{purple} ${0.45\pm0.02}$} & {${0.39\pm0.01}$}\\
\midrule 0.25\quad\xmark& $1/24$ & {${0.38\pm0.02}$} & {\color{purple} ${0.49\pm0.06}$} & {\color{purple} ${0.48\pm0.05}$} & {\color{purple} $\mb{0.55\pm0.07}$} & {\color{purple} ${0.43\pm0.04}$} & {\color{purple} ${0.40\pm0.02}$} \\
0.25\quad\xmark&$1/18$ & {${0.35\pm0.00}$} & {\color{purple} $\mb{0.39\pm0.03}$} & {\color{purple} $\mb{0.39\pm0.02}$} & {\color{purple} ${0.36\pm0.01}$} & {\color{purple} ${0.36\pm0.01}$} & {${0.35\pm0.00}$}\\ 
0.25\quad\xmark&$1/12$ & {${0.36\pm0.00}$} & {${0.36\pm0.00}$} & {${0.36\pm0.00}$} & {${0.36\pm0.00}$} & {${0.36\pm0.00}$} & {${0.36\pm0.00}$} \\
\midrule
0.25\quad\cmark&$1/24$ & {${0.82\pm0.10}$} & {\color{purple} ${0.97\pm0.02}$} & {\color{purple} $\mb{1.00\pm0.00}$} & {\color{purple} ${0.94\pm0.02}$} & {${0.77\pm0.02}$} & {${0.49\pm0.02}$} \\
0.25\quad\cmark&$1/18$ & {${0.64\pm0.10}$} & {\color{purple} ${0.70\pm0.07}$} & {\color{purple} $\mb{0.79\pm0.07}$} & {\color{purple} ${0.64\pm0.06}$} & {${0.63\pm0.03}$} & {${0.44\pm0.01}$} \\ 
0.25\quad\cmark&$1/12$ & {${0.39\pm0.01}$} & {\color{purple} $\mb{0.47\pm0.03}$} & {\color{purple} ${0.45\pm0.02}$} & {\color{purple} $\mb{0.47\pm0.03}$} & {${0.36\pm0.04}$} & {\color{purple}${0.40\pm0.02}$} \\
\midrule
0.5\quad\xmark&$1/24$ & $0.49\pm0.06$ &  {\color{purple} ${0.60\pm0.08}$} & {\color{purple} $\mb{0.66\pm0.08}$} & {\color{purple} ${0.61\pm0.08}$} & {\color{purple} ${0.65\pm0.06}$} & {${0.46\pm0.04}$} \\
0.5\quad\xmark&$1/18$ & $0.36\pm0.01$ &  {\color{purple} $\mb{0.41\pm0.02}$} & {\color{purple} ${0.39\pm0.01}$} & {\color{purple} ${0.38\pm0.01}$} & {\color{purple} ${0.37\pm0.01}$} & {${0.35\pm0.00}$} \\ 
0.5\quad\xmark&$1/12$ & $0.36\pm0.00$ &  {\color{purple} $\mb{0.38\pm0.01}$} & {\color{purple} $\mb{0.38\pm0.01}$} & {${0.36\pm0.00}$} & {\color{purple} $\mb{0.38\pm0.01}$} & {${0.36\pm0.00}$} \\
\midrule
0.5\quad\cmark&$1/24$ & $0.92\pm0.08$ &  {\color{purple} $\mb{1.00\pm0.00}$} & {\color{purple} $\mb{1.00\pm0.00}$} & {\color{purple} $\mb{1.00\pm0.00}$} & {${0.87\pm0.03}$} & {${0.55\pm0.03}$} \\
0.5\quad\cmark&$1/18$ & $0.66\pm0.09$ &  {\color{purple} ${0.71\pm0.08}$} & {\color{purple} $\mb{0.80\pm0.08}$} & {\color{purple} ${0.68\pm0.08}$} & {\color{purple} ${0.67\pm0.04}$} & {${0.44\pm0.02}$} \\ 
0.5\quad\cmark&$1/12$ & $0.41\pm0.02$ &  {\color{purple} ${0.44\pm0.04}$} & {\color{purple} ${0.43\pm0.03}$} & {\color{purple} $\mb{0.50\pm0.04}$} & {\color{purple} ${0.47\pm0.03}$} & {${0.39\pm0.04}$} \\
\bottomrule
\end{tabular}
\caption{\textbf{Go-Explore with different exploration noise.} \texttt{EN} represents the multiplier for the Gaussian exploration noise. All other settings are the same as \cref{table:vanilla-antumaze}.}
\label{table:go-explore-antumaze}
\end{minipage}
\hfill 
\end{table}

\subsection{Non Goal Reaching Tasks}
\label{appendix:learning-curves-meta-learning}
\begin{figure}[H]
    \centering
    \begin{minipage}{0.85\linewidth}
    \centering\captionsetup[subfigure]{justification=centering}
    \includegraphics[width=\linewidth]{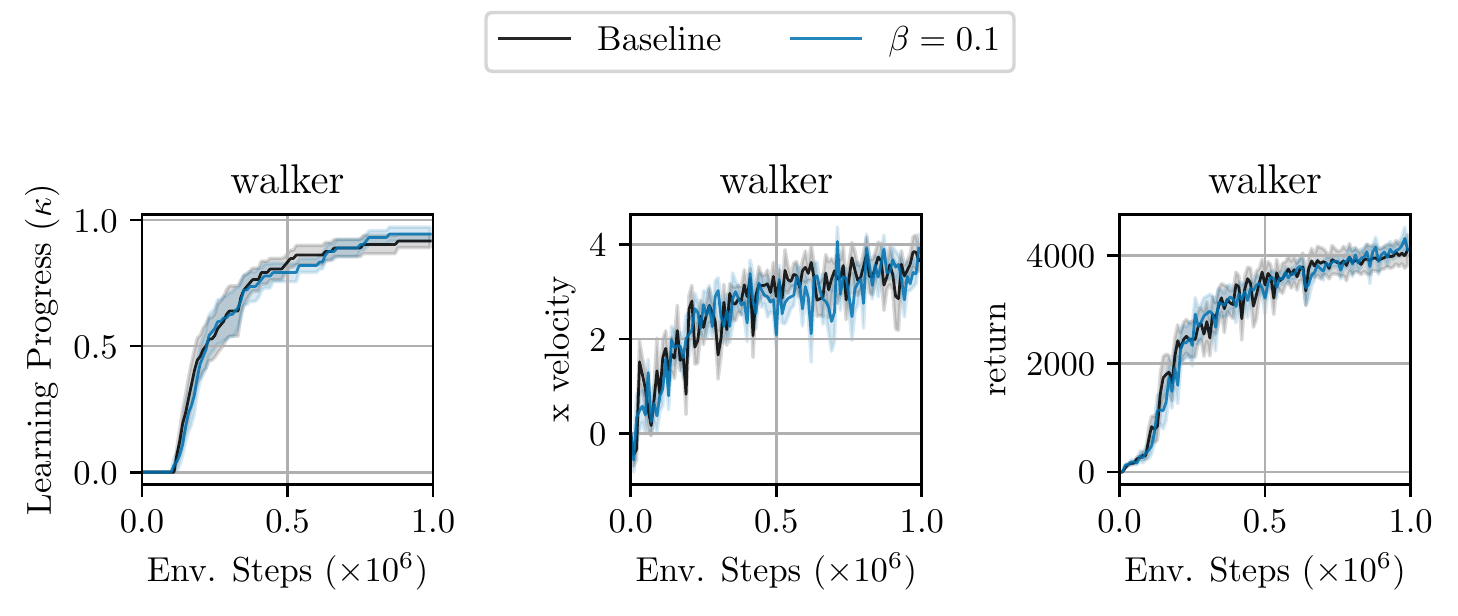}
    \includegraphics[trim=0 0 0 0,clip,width=\textwidth]{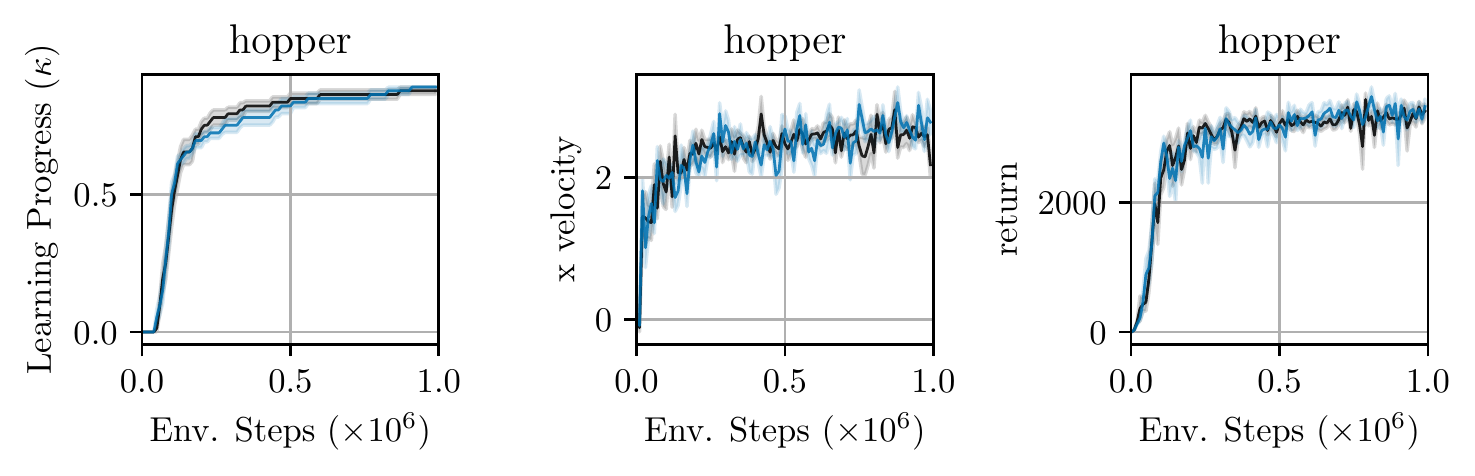}
    \includegraphics[trim=0 0 0 0,clip,width=\textwidth]{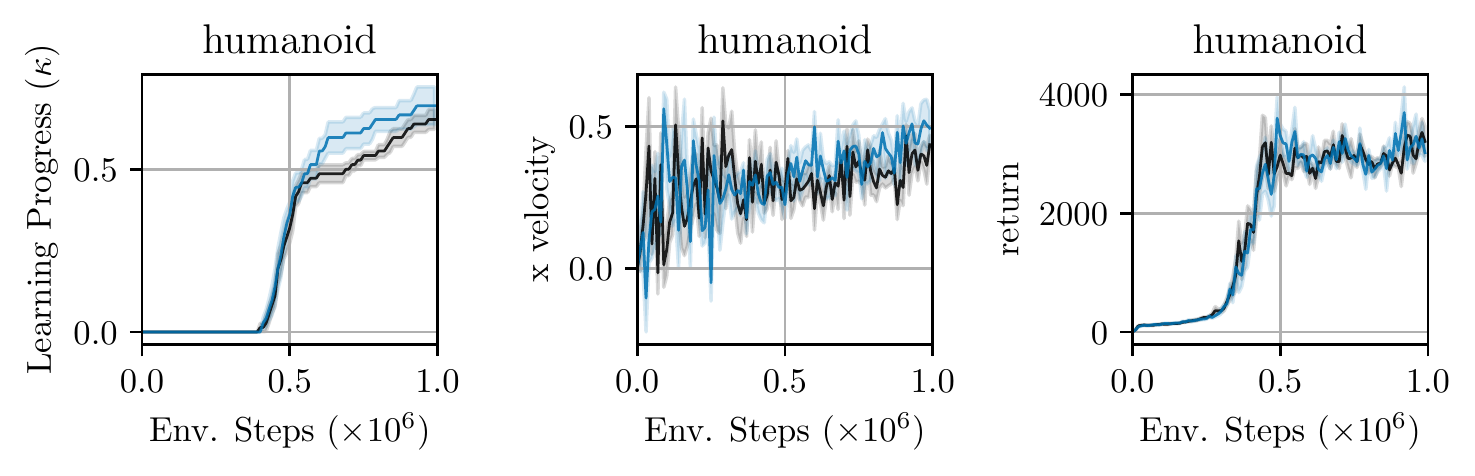}
    \includegraphics[trim=0 0 0 0,clip,width=\textwidth]{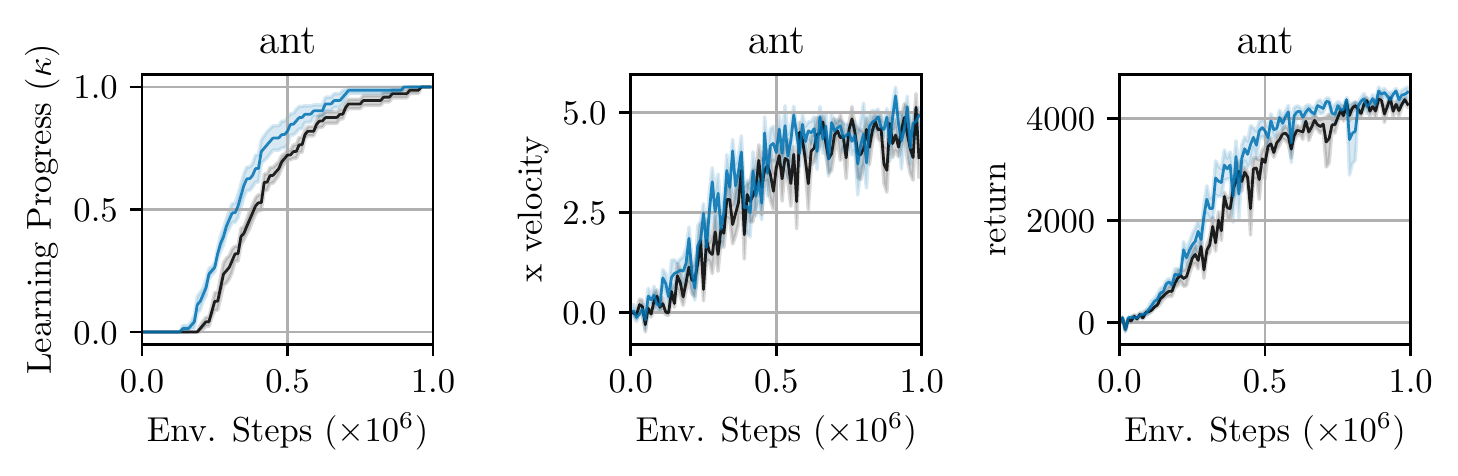}
    \end{minipage}
    \caption{\textbf{Accelerating learning on several non goal-reaching tasks.} The confidence interval represents the standard error computed over 8 random seeds, for $\beta=0.1$.}
    \label{fig:drl-meta-rollin}
\end{figure}

\begin{table}[H]
\centering
\begin{minipage}[b]{\linewidth}
\centering
\small
\begin{tabular}{ccccccc}
\toprule
\texttt{Env.}&Step&$\beta=0$ &$\beta=0.1$&$\beta=0.2$ & $\beta=0.5$ & $\beta=0.75$ \\ \midrule
\texttt{walker}& 0.5m& $0.83\pm 0.03$ &{$0.79\pm0.04$} & $0.75\pm0.04$ & $0.78\pm0.05$ &$0.76\pm0.05$\\
& 1m& $0.92\pm0.03$ &{\color{purple}$\mb{0.94\pm0.03}$}& {$0.90\pm0.01$} &{$0.92\pm0.04$} & {$0.92\pm0.03$} \\
\midrule\texttt{hopper}& 0.5m& $0.85\pm 0.02$ &{${0.82\pm0.03}$} & {$0.83\pm0.02$} & $0.78\pm0.02$ &$0.75\pm0.02$\\
& 1m& $0.88\pm0.01$ &{\color{purple}$\mb{0.89\pm0.00}$}& {\color{purple}$\mb{0.89\pm0.03}$} &{$0.82\pm0.02$} & {$0.81\pm0.02$} \\
\midrule
\texttt{humanoid}& 0.5m& $0.32 \pm 0.05$ &{\color{purple}$\mb{0.36\pm0.04}$} & {$0.21\pm0.07$} & {\color{purple}${0.33\pm0.04}$} &{${0.14\pm0.06}$}\\
& 1m & $0.67\pm0.03$ &{\color{purple}$0.69\pm0.06$}& {$0.62\pm0.02$} &{\color{purple}$\mb{0.76\pm0.03}$} & {\color{purple}${0.71\pm0.06}$} \\ \midrule
\texttt{ant}& 0.5m& $0.72 \pm 0.02$ &{\color{purple}$\mb{0.82\pm0.06}$} & {$0.68\pm0.08$} & {$0.64\pm0.05$} &{$0.47\pm0.05$}\\
& 1m & $1.00\pm0.00$ &{$1.00\pm0.00$}& {$0.83\pm0.08$} &{$0.86\pm0.06$} & {${0.71\pm0.07}$} \\ 
\bottomrule
\end{tabular}
\caption{\textbf{Learning progress $\kappa$ at 0.5 and 1.0 million environment steps with varying $\beta$ of non goal reaching tasks.} Baseline corresponds to $\beta = 0$, where no \ours{} is used. The standard error is computed over 8 random seeds. We highlight the values that are larger than the baseline ($\beta=0$) in {\color{purple} purple}, and the largest value in {\color{purple} \bf bold font}.}
\label{table:meta-learning-progress}
\end{minipage}
\hfill 
\end{table}

\begin{table}[H]
\centering
\begin{minipage}[b]{\linewidth}
\centering
\small
\begin{tabular}{ccccccc}
\toprule
\texttt{Env.}&Step&$\beta=0$ &$\beta=0.1$&$\beta=0.2$ & $\beta=0.5$ & $\beta=0.75$ \\ \midrule
\texttt{walker}& 0.5m& $3.09\pm 0.31$ &{${2.83\pm0.31}$} & {$2.41\pm0.33$} & $2.77\pm0.31$ &{ $2.88\pm0.32$}\\
& 1m& $3.69\pm0.27$ &{${3.62\pm0.26}$}& {$3.09\pm0.28$} &{$3.48\pm0.27$} & {$3.14\pm0.34$} \\
\midrule\texttt{hopper}& 0.5m& $2.42\pm 0.18$ &{${2.26\pm0.22}$} & {\color{purple}$\mb{2.45\pm0.14}$} & $2.34\pm0.16$ &{${2.34\pm0.16}$}\\
& 1m& $2.58\pm0.16$ &{\color{purple}$\mb{2.65\pm0.15}$}& {\color{purple}$\mb{2.65\pm0.17}$} &{$2.39\pm0.18$} & {$2.52\pm0.19$} \\
\midrule
\texttt{humanoid}& 0.5m& $0.26 \pm 0.05$ &{\color{purple}$0.32\pm0.07$} & {\color{purple}$0.27\pm0.05$} & {\color{purple}${0.34\pm0.05}$} &{\color{purple}$\mb{0.38\pm0.07}$}\\
& 1m & $0.39\pm0.05$ &{\color{purple}$0.46\pm0.09$}& {\color{purple}$0.41\pm0.05$} &{\color{purple}${0.41\pm0.06}$} & {\color{purple}$\mb{0.49\pm0.10}$} \\ \midrule
\texttt{ant}& 0.5m& $3.38 \pm 0.43$ &{\color{purple}$\mb{3.85\pm0.41}$} & {\color{purple} $3.43\pm0.53$} & {$3.15\pm0.45$} &{$2.38\pm0.46$}\\
& 1m & $4.29\pm0.51$ &{\color{purple}$\mb{4.66\pm0.30}$}& {$3.93\pm0.45$} &{$3.99\pm0.48$} & {${3.50\pm0.49}$} \\ 
\bottomrule
\end{tabular}
\caption{\textbf{Average $x$-direction velocity of the last 50k time steps, at 0.5 and 1.0 million environment steps with varying $\beta$ of non goal reaching tasks.} Baseline corresponds to $\beta = 0$, where no \ours{} is used. The standard error is computed over 8 random seeds. We highlight the values that are larger than the baseline ($\beta=0$) in {\color{purple} purple}, and the largest value in {\color{purple} \bf bold font}.}
\label{table:meta-learning-speed}
\end{minipage}
\hfill 
\end{table}

\begin{table}[H]
\centering
\begin{minipage}[b]{\linewidth}
\centering
\small
\begin{tabular}{ccccccc}
\toprule
\texttt{Env.}&Step&$\beta=0$ &$\beta=0.1$&$\beta=0.2$ & $\beta=0.5$ & $\beta=0.75$ \\ \midrule
\texttt{walker}& 0.5m& $3450.1\pm 307.4$ &{${3350.4\pm184.6}$} & $2897.4\pm276.5$ & $3255.9\pm203.8$ &$3185.8\pm341.5$\\
& 1m& $4032.3\pm224.3$ &{\color{purple}$\mb{4128.8\pm159.6}$}& {$3685.5\pm135.6$} &{$4028.8\pm164.2$} & {$3895.4\pm265.4$} \\
\midrule\texttt{hopper}& 0.5m& $3192.5\pm 80.4$ &{$ {3148.6\pm160.7}$} & {\color{purple}$\mb{3241.5\pm130.8}$} & $3116.5\pm141.8$ &$3059.6\pm153.8$\\
& 1m& $3386.2\pm124.7$ &{\color{purple}$\mb{3421.9\pm109.8}$}& {${3262.3\pm98.1}$} &{$3170.7\pm180.6$} & {\color{purple}$3394.5\pm126.5$} \\
\midrule
\texttt{humanoid}& 0.5m& $2910.1 \pm 262.9$ &{\color{purple}$2939.7\pm392.0$} & {$2598.9\pm309.8$} & {\color{purple}$\mb{3137.3\pm305.6}$} &{${2259.6\pm245.4}$}\\
& 1m & $3017.2\pm169.0$ &{\color{purple}$3173.6\pm238.3$}& {$2935.8\pm181.1$} &{${2905.5\pm125.9}$} & {\color{purple}$\mb{3290.7\pm275.9}$} \\ \midrule
\texttt{ant}& 0.5m& $2976.2 \pm 252.4$ &{\color{purple}$\mb{3593.1\pm237.8}$} & {\color{purple}$3071.8 \pm 340.0$} & {$2818.3 \pm 265.2$} &{$2188.3\pm 256.2$}\\
& 1m & $4248.5\pm88.6$ &{\color{purple}$\mb{4473.0\pm102.2}$}& {$3683.1\pm345.0$} &{$3708.7\pm290.5$} & {${3250.1\pm316.2}$} \\ 
\bottomrule
\end{tabular}
\caption{\textbf{Average return of the last 50k time steps, at the 0.5 and 1.0 million environment steps with varying $\beta$ of non goal reaching tasks.} Baseline corresponds to $\beta = 0$, where no \ours{} is used. The standard error is computed over 8 random seeds. We highlight the values that are larger than the baseline ($\beta=0$) in {\color{purple} purple}, and the largest value in {\color{purple} \bf bold font}.}
\label{table:meta-learning-return}
\end{minipage}
\hfill 
\end{table}


\end{document}